\documentclass[11pt]{article} %
\usepackage{times}
\usepackage{url}            %
\usepackage{booktabs}       %
\usepackage{amsfonts}       %
\usepackage{nicefrac}       %
\usepackage{microtype}      %
\usepackage{mkolar_definitions}
\usepackage{xspace}
\usepackage{multirow}
\usepackage{amsmath}
\usepackage{algorithm}
\usepackage{algorithmic}
\usepackage{color}
\usepackage{color, colortbl}
\usepackage{enumitem}
\usepackage{comment}
\usepackage{bm}
 \usepackage[left=2cm,top=2cm,right=2cm]{geometry}

\usepackage[scaled=.9]{helvet}

\usepackage{url}
\usepackage{authblk}
\usepackage{csquotes}

\usepackage[dvipsnames]{xcolor}
\usepackage[colorlinks=true, linkcolor=blue, citecolor=blue]{hyperref}

\usepackage{natbib}
\bibpunct{(}{)}{;}{a}{,}{,}

\newcount\Comments  %
\definecolor{darkgreen}{rgb}{0,0.5,0}
\definecolor{darkred}{rgb}{0.7,0,0}
\definecolor{teal}{rgb}{0.3,0.8,0.8}
\definecolor{orange}{rgb}{1.0,0.5,0.0}
\definecolor{purple}{rgb}{0.8,0.0,0.8}
\newcommand{\kibitz}[2]{\ifnum\Comments=1{\textcolor{#1}{\textsf{\footnotesize #2}}}\fi}

\definecolor{Gray}{gray}{0.9}

\newcommand{\MLE}{\mathrm{MLE}}
\newcommand{\Bayes}{\text{Bayes}}
\newcommand{\pa}{\mathrm{pa}}
\newcommand{\TV}{\mathrm{TV}}

\newcommand{\newedit}{\color{black}}
\newcommand{\ouredit}{\color{black}}

\usepackage{amsmath,amsfonts,bm}

\def\eqref#1{equation~\ref{#1}}

\def\1{\bm{1}}

\def\rd{{\textnormal{d}}}

\DeclareMathAlphabet{\mathsfit}{\encodingdefault}{\sfdefault}{m}{sl}
\SetMathAlphabet{\mathsfit}{bold}{\encodingdefault}{\sfdefault}{bx}{n}

\newcommand{\E}{\mathbb{E}}

\newcommand{\KL}{D_{\mathrm{KL}}}

\DeclareMathOperator*{\argmax}{arg\,max}
\DeclareMathOperator*{\argmin}{arg\,min}

\DeclareMathOperator{\Tr}{Tr}

\usepackage{mkolar_definitions}

\begin{document}
\title{Pessimistic Model-based Offline Reinforcement Learning under Partial Coverage }
\author[1]{Masatoshi Uehara\thanks{mu223@cornell.edu}} 
\author[1]{Wen Sun \thanks{ws455@cornell.edu}}
\affil[1]{Department of Computer Science, Cornell University}
\date{}

\maketitle

\begin{abstract}
We study model-based offline Reinforcement Learning with general function approximation. We present an algorithm named \emph{Constrained Pessimistic Policy Optimization (CPPO)} which leverages a general function class and uses a constraint to encode pessimism. Under the assumption that the ground truth model belongs to our function class, CPPO can learn with the offline data only providing partial coverage, i.e., it can learn a policy that competes against any policy that is covered by the offline data, in polynomial sample complexity with respect to the statistical complexity of the function class. We then demonstrate that this algorithmic framework can be applied to many specialized Markov Decision Processes and the additional structural assumptions can further refine the concept of partial coverage. One notable example is \emph{low-rank MDP with representation learning} where the partial coverage is defined using the concept of relative condition number measured by the underlying unknown ground truth feature representation. {\newedit Finally, we introduce and study the Bayesian setting in offline RL. The key benefit of Bayesian offline RL is that algorithmically, we do not need to explicitly construct pessimism or reward penalty which could be hard beyond models with linear structures. We present a \emph{posterior sampling based incremental policy optimization} algorithm (PS-PO) which proceeds by iteratively sampling a model  from the posterior distribution and performing one step incremental policy optimization inside the sampled model. Theoretically, in expectation with respect to the prior distribution, PS-PO can learn a near optimal policy under partial coverage with polynomial sample complexity.  This work is a long version of the conference paper in \url{https://openreview.net/pdf?id=tyrJsbKAe6}.} 
\end{abstract}

\section{Introduction}
\label{sec:intro}

Offline Reinforcement Learning (RL) is one of the important areas of RL where the learner is presented with a static dataset consisting of transition-related information (state, action, reward, and next state) collected by some behavior policy, and needs to learn purely from the offline data without any future online interaction with the environment. 
Offline RL is used in a number of applications where online random experimentation is costly or dangerous such as health care \citep{2019PM}, digital marketing \citep{ChenMinmin2019TOCf} and robotics \citep{levine2020offline}. 

The performance guarantees of offline RL often rely on two quantities: the coverage of the offline data and the property of the function approximation used in the algorithms. For instance, for the classic Fitted-Q-iteration (FQI) algorithm \citep{ernst2005tree,munos2008finite}, it requires (a) full coverage in the offline data, i.e., $\max_{(s,a)}d^{\pi}(s,a)/\rho(s,a)<\infty$ for \emph{any} stochastic policies $\pi$ including history-dependent non-Markovian policies, where $d^{\pi}(s,a)$ is a state-action occupancy distribution of a policy $\pi$ and $\rho(s,a)$ is an offline distribution,
(b) realizability in a Q function class, i.e., the optimal Q function belongs to the function class, and (c) Bellman completeness, i.e., applying the Bellman operator on any function in the function class results in a new function that also belongs to the function class (see the first row in Table~\ref{tab:comparison}). Among these three assumptions, the full coverage and the Bellman completeness are particularly strong. The full coverage means that the behavior policy needs to be exploratory enough, although figuring out an exploratory policy itself is an extremely hard problem for large-scale MDPs. The Bellman completeness assumption does not have a monotonic property, i.e., even starting with a function class that originally permits Bellman completeness, slightly increasing the capacity of the function class could result in a new class that does not have Bellman completeness anymore.  
Thus, we aim to relax the assumptions on the offline data and the function class. Particularly, we are interested in the following question: 
\begin{displayquote}

{Given a realizable function class and an offline distribution that only provides partial coverage, can we learn a policy that is able to compete with any policy that is covered by the offline distribution?} 

\end{displayquote}
\emph{We study this question from a model-based learning perspective and provide an affirmative answer to the question}. More specifically, different from FQI, we start with a realizable model class, i.e., the ground truth transition falls into the model class. We further abandon the strong full coverage assumption, and instead, assume partial coverage 
 which means the offline data distribution only covers a state-action distribution of \emph{some  high-quality comparator policy $\pi^*$} ($\pi^*$ is not necessarily the optimal policy, and $\pi^*$ could be non-Markovian), i.e., $\max_{s,a} d^{\pi^*}(s,a) / \rho(s,a) <\infty$,
We design an algorithm --- Constrained Pessimistic Policy Optimization (CPPO), which can learn a policy that is as good as any comparator policy $\pi^*$ that is covered by the offline data. The fact that CPPO can learn to compete against history-dependent policies is meaningful in offline RL when the offline data does not cover the optimal policy.

\begin{table*}[!t]\label{tab:comparison}

\centering\resizebox{\columnwidth}{!}{
\begin{tabular}{ |l|l|l|l| } 
\hline
Methods & {\small Type} & Coverage  & Additional Structures \\
\hline 
FQI \citep{munos2008finite} & $\mathbf{F}$ & Full: $\max_{s,a}\frac{d^{\pi}(s,a)}{\rho(s,a)} < \infty,\forall \pi$   & Bellman complete\\ 
\hline
Minimax Way \citep{UeharaMasatoshi2019MWaQ} & $\mathbf{F}$ & Full: $\max_{s,a}\frac{d^{\pi}(s,a)}{\rho(s,a)} < \infty,\forall \pi$  & Realizability in density ratio  \\ 
\hline
\citet{DuanYaqi2020MOEw} & $\mathbf{F}$ & Full: $\mathbb{E}_{s,a\sim \rho}\phi(s,a)\phi(s,a)^{\top}$ is PSD  & Linear Bellman complete \\  
\hline 
\citet{XieTengyang2020BVAw} & $\mathbf{F}$ & Full: $\max_{s,a,s'}\frac{P^*(s' | s,a)}{ \rho(s') } < \infty$  & None \\
\hline  
\citet{Liu2020} &  $\mathbf{F}$ & Partial$^\dagger$ : $\max_{s,a}\frac{d^{\pi^{*}}(s,a)}{\rho(s,a)} < \infty$ & Bellman / Policy class complete    \\
\hline 
\citet{RashidinejadParia2021BORL} & $\mathbf{F}$ & Partial:  $\max_{s,a}\frac{d^{\pi^{*}}(s,a)}{\rho(s,a)} < \infty$  & Tabular MDP \\  
\hline 
\begin{tabular}{c}
   \citet{JinYing2020IPPE}  \vspace{-3pt}    \\
   \citet{zhang2021corruption}
\end{tabular}    & $\mathbf{F}$ & Partial$^{\dagger\dagger}$:  $\max_{x }\frac{x^{\top} \mathbb{E}_{s,a\sim d^{\pi^*}}\phi(s,a)(\phi(s,a))^{\top} x   }{ x^{\top} \mathbb{E}_{s,a\sim \rho}\phi(s,a)(\phi(s,a))^{\top} x   }< \infty$  & Linear MDP \citep{Jin2020}  \\  
\hline 
\citet{XieTengyang2021BPfO} & $\mathbf{F}$ & Partial: $\max_{f}\frac{ \|f - \Tcal f\|^2_{d^{\pi^*}}  }{\|f - \Tcal f\|^2_{\mu}} < \infty$ & Bellman complete \\ \hline 
\citet{zanette2021provable} & $\mathbf{F}$ & Partial :  $\max_{x }\frac{x^{\top} \mathbb{E}_{s,a\sim d^{\pi^*}}\phi(s,a)(\phi(s,a))^{\top} x   }{ x^{\top} \mathbb{E}_{s,a\sim \rho}\phi(s,a)(\phi(s,a))^{\top} x   }< \infty$  & Linear Bellman complete \\  
\hline 
\hline
Batch \citep{ross2012agnostic} & $\mathbf{B}$ & Full: $\max_{s,a }\frac{ d^{\pi}(s,a) }{\rho(s,a)} < \infty,\forall \pi$ & None \\
\hline
Milo \citep{ChangJonathanD2021MCSi} & $\mathbf{B}$ & Partial: $\max_{x }\frac{x^{\top} \mathbb{E}_{s,a\sim d^{\pi^*}}\phi(s,a)(\phi(s,a))^{\top} x   }{ x^{\top} \mathbb{E}_{s,a\sim \rho}\phi(s,a)(\phi(s,a))^{\top} x   }< \infty$  & KNR / GP  \\
\hline
   \rowcolor{Gray}  &      & Partial$^{\dagger\dagger\dagger}$: $\max_{s,a }\frac{d^{\pi^*}(s,a)}{\rho(s,a)} < \infty$   & None \\
    \rowcolor{Gray}     &  & Partial: $\max_{x }\frac{x^{\top} \mathbb{E}_{s,a\sim d^{\pi^*}}\phi(s,a)(\phi(s,a))^{\top} x   }{ x^{\top} \mathbb{E}_{s,a\sim \rho}\phi(s,a)(\phi(s,a))^{\top} x   }< \infty$  & {\small Linear MDP /KNR / GP} \\
 \rowcolor{Gray}     &  & Partial: $\max_{P\in \Mcal}\max_{x }\frac{x^{\top} \mathbb{E}_{s,a\sim d^{\pi^*}}\psi_P(s,a)(\psi_P(s,a))^{\top} x   }{ x^{\top} \mathbb{E}_{s,a\sim \rho}\psi_P(s,a)(\psi_P(s,a))^{\top} x   }< \infty$  & {\small \textbf{Linear Mixture MDPs} ($\psi_P$ depends on $P$)} \\
 \rowcolor{Gray}     &  & Partial: $\max_{j \in [1,\cdots,d]}\max_{ s_j \in \Scal_j, a \in \Acal}\frac{d^{\pi^{\star}}_{P^{\star}}(s_j,a) }{\rho(s_j,a) }$  & {\small \textbf{Factored MDPs}} \\
 \rowcolor{Gray} \multirow{-4}{*}{\textbf{CPPO (Ours)} } & \multirow{-4}{*}{$\mathbf{B}$} & Partial: $\max_{x }\frac{x^{\top} \mathbb{E}_{s,a\sim d^{\pi^*}}\phi^*(s,a)(\phi^*(s,a))^{\top} x   }{ x^{\top} \mathbb{E}_{s,a\sim \rho}\phi^*(s,a)(\phi^*(s,a))^{\top} x   }< \infty$  & \textbf{Low-rank MDP} (unknown $\phi^*$)\\
\hline
\end{tabular}
}%
\caption[Caption for LOF]{Comparison among existing works regarding their type, coverage, and additional structural assumptions on the function class or MDPs. Type $\mathbf{F}$ means model-free and type $\mathbf{B}$ means model-based. {\ouredit  Linear mixture MDPs, factored MDPs, and low-rank MDPs are models that our algorithm gives the partial coverage result for the first time.
} Partial coverage means \protect\footnotemark that the offline distribution $\rho$ covers a state-action distribution of a comparator policy $\pi^*$. %
$\dagger$ means it assumes an accurate density estimator for $\rho(s,a)$. $\dagger\dagger$ means although the analysis in \citet{Jin2020} is done under the full coverage for linear MDPs, based on the argument \citep{zhang2021corruption}, we can show the algorithm has the PAC guarantee under partial coverage in terms of the relative condition number for linear MDPs. $\dagger\dagger\dagger$ means that we can refine it to a more adaptive quantity using the model class (i.e., Definition~\ref{def:partial_con}). All the methods in the table require realizability in the function class.   }
\end{table*}

While one could assume density ratio based concentrability coefficient ($\max_{s,a} d^{\pi^*}(s,a) / \rho(s,a)$) to be under control for small size MDPs, in large-scale MDPs (e.g. continuous state space), the density ratio could quickly become an extremely large quantity which makes the performance guarantee vacuous. When applying CPPO to MDPs with additional structural assumptions, we can seamlessly refine the density ratio based concentrability coefficient to more natural and tighter quantities. Notably, we consider the offline representation learning setting where the underlying MDPs permit a low-rank structure (unlikely linear MDPs \citep{Jin2020,YangLinF2019RLiF}, we do not assume the ground truth state-action feature representation $\phi^\star$ is known, and instead we need to learn $\phi^\star$) and we show that we can refine the density ratio to a relative condition number that is defined using the \emph{unknown} true state-action feature representation $\phi^\star$. Intuitively this means that as long as there exists a high-quality comparator policy that only visits the subspace (defined using the true representation $\phi$) that is covered by the offline data, CPPO can compete against such a policy, \emph{even without knowing the true $\phi^\star$}. Such bounded relative condition number assumption is much weaker than the bounded density ratio assumption.%
While the concept of relative condition number was originally introduced in the online RL setting (e.g.,  \cite{pmlr-v125-agarwal20a,Agarwal2020pcpg} with a known linear feature $\phi$), and later was introduced in offline RL (\cite{zhang2021corruption,ChangJonathanD2021MCSi}), these prior works \emph{all} rely on the fact that the feature representation $\phi$ is known to the learner a priori (see Table~\ref{tab:comparison} for the comparison). Another interesting example is factored MDPs \citep{kearns1999efficient} where we show CPPO refines the density ratios to be density ratio associated with individual factors, which leverages the factored structure and is provably tighter.  {\ouredit We also give examples on parametric linear MDPs \citep{YangLinF2019RLiF}, nonparametric linear MDPs \citep{Jin2020}, linear mixture MDPs \citep{ayoub2020model,modi2020sample}, kernelized nonlinear regulators (KNRs) and MDPs with Gaussian processes (GPs) \citep{Kakade2020,CuriSebastian2020EMRL}, where we again show that CPPO enjoys problem specific quantities for measuring the coverage. }

\paragraph{Our contributions.} Our contributions are three-folds, which we summarize below:
\begin{enumerate}

\item We show that in the model-based setting, realizability and partial coverage is enough to learn a high-quality comparator policy (\pref{thm:version} {\newedit and \pref{thm:version2}}). Notably, (1) this result holds for \emph{any MDPs with realizable model classes}, (2) we can compete against even history-dependent policies. This is in sharp contrast to the state-of-art provable model-free offline RL results: see \pref{tab:comparison} for detailed comparisons to prior works. 

\item Under additional structural assumptions (e.g.,  KNRs, linear MDPs, linear mixture MDPs, low-rank MDPs, factored MDPs), we show that we can seamlessly refine the density ratio based concentrability coefficients to problem specific quantities. 
{\newedit This flexibility to adapt to problem specific coverage measuring quantities is in sharp contrast to standard offline RL algorithms. 
Especially, two notable settings are low-rank MDPs (with unknown features) (\pref{thm:low_rank}) and factored MDPs (\pref{thm:factoed}): (a) for offline representation learning in low-rank MDPs, the density ratio concentrability coefficient is refined to be a relative condition number under the true (but unknown) representation (\pref{thm:low_rank}); (b) for factored MDPs, the concentrability coefficient is refined using the density ratios associated to individual factors (\pref{thm:factoed}).} %

{\newedit \item For computational purpose, we develop incremental policy optimization and posterior sampling-based offline RL algorithms under Bayesian setting (\pref{alg:pspo} and \pref{thm:bayesian_pspo} in \pref{sec:alg}). While moving to the Bayesian setting, we sacrifice from a worst-case guarantee to a guarantee on the Bayesian suboptimality gap, we gain benefits in terms of no need to design pessimism inside the algorithms.}
\vspace{-2pt}
\end{enumerate}

While we focus on the model-based setting and have demonstrated advantages of our approach over model-free ones (i.e., no more Bellman completeness assumption on function classes, being able to compete against a larger pool of policies, and the ability to seamlessly adapt to problem-dependent structures), it is worth noting that realizability in the model-based setting is usually considered stronger than the one in the model-free setting. 
On the empirical side, model-based offline RL algorithms are the state-of-art (e.g., \cite{Yu2020,Kidambi2020,MatsushimaTatsuya2020DRLv,cang2021behavioral,ChangJonathanD2021MCSi}).
Our theoretical results provide a sharp contrast between model-based and model-free approaches in offline RL. For details, refer to \pref{sec:conversion}. 

{\newedit  
The rest of the article is organized as follows. In \pref{sec:related}, we discuss the related work. In \pref{sec:prelim}, we introduce our setting and notation. In \pref{sec:version}, we introduce two types of main algorithms, which we term Constrained Pessmistic Policy Optimization (CPPO). In \pref{sec:examples}, we instantiate our results in several models such as tabular MDPs, linear mixture MDPs, (parametric) linear MDPs, low-rank MDPs and factored MDPs. In \pref{sec:knrs}, we continue this instantiation in KNRs. In \pref{sec:linear_mdps}, we modify CPPO to capture (nonparametric) linear MDPs. In \pref{sec:alg}, we introduce the posterior sampling-based offline RL algorithm under the Bayesian setting. In \pref{sec:conclusion}, we discuss our summary and future works. 
}

\section{Related work}\label{sec:related}

We discuss two families of related works: offline RL and representation learning in RL.
 
\paragraph{Offline RL.} Insufficient coverage of the dataset due to the lack of online exploration is known as the main challenge in offline RL \citep{Wang2020}. To deal with this problem, a number of methods have been recently proposed from both model-free \citep{WuYifan2019BROR,TouatiAhmed2020SPOv,kumar2020conservative,Liu2020,RezaeifarShideh2021ORLa,pmlr-v97-fujimoto19a,fakoor2021continuous,ghasemipour2021emaq,buckman2020importance} and model-based perspectives \citep{Yu2020,Kidambi2020,MatsushimaTatsuya2020DRLv,YinMing2021NORL}. More or less, their methods rely on the idea of pessimism and its variants in the sense that the learned policy can avoid uncertain regions not covered by offline data. As a theoretical side, \citet{munos2008finite,DuanYaqi2020MOEw,DuanYaqi2021RBaR,FanJianqing2019ATAo} proved FQI has a PAC (probably approximately correct) guarantee under realizability, the global coverage, and Bellman completeness. Other offline model-free RL methods such as minimax offline RL methods also require realizability and the global coverage \citep{ChenJinglin2019ICiB,antos2008learning,UeharaMasatoshi2021FSAo,DuanYaqi2021RBaR,zhang2019gendice,nachum2019algaedice}. Recently, by being inspired by aforementioned the pessimism idea, \citet{Jin2020,RajaramanNived2020TtFL} showed that FQI with an additional pessimistic bonus (penalty) term can weaken the assumption from the global coverage to partial coverage. Comparing to their works, our analysis focuses on a model-based method.  The offline model-based method is known to have a PAC guarantee under the realizability and the global coverage \citep{ross2012agnostic,ChenJinglin2019ICiB}. As the most closely related work, \citet{ChangJonathanD2021MCSi} proved a model-based method with an additional penalty term can weaken the assumption from the global coverage to the partial coverage for structured MDPs such as KNRs and Gaussian Processes models \citep{deisenroth2011pilco}. In this work, we consider \emph{arbitrary} MDPs with a realizable model class and aim for PAC bounds under a partial coverage condition.

\paragraph{Representation learning.} We discuss literature related to representation learning in RL. Representation learning for low-rank MDPs (ground truth feature representation is unknown) in online learning is studied from a model-based perspective \citep{Agarwal2020_flambe} and model-free perspective \citep{ModiAditya2021MRLa}. In the online setting, \citet{zhang2021provably,papini2021leveraging} also study representation learning under different model assumptions. Comparing with these works, since our setting is offline, the algorithm and analysis are totally different. 

In the offline setting, \cite{ni2021learning} study dimensionality reduction in a given kernel space, and  \cite{hao2021sparse} study feature selection in sparse linear MDPs. Their focus is different as they do not study PAC guarantees under partial coverage.  \cite{ni2021learning} assumes the transition operator can be properly embedded into predefined Reproducing Kernel Hilbert Spaces and learns low-dimensional state-action representations via kernelized embedding and low-rank tensor decomposition. However, they did not study the errors for policy optimization after using these learned features. Regarding offline distribution coverage, \cite{ni2021learning} assumes that the feature covariance matrix (feature associated with the pre-defined kernel) of the offline distribution  is full rank.  \cite{hao2021sparse} studies an OPE problem  on sparse linear Bellman complete MDPs in the offline learning setting where they assume all covariance matrices (covariance matrices that correspond to all possible subsets of features) under the offline distribution are full rank as well. 
We study policy optimization in low-rank MDPs (with unknown feature representation), and we do not assume full coverage, i.e., we do not assume the feature covariance matrix is full rank, and indeed our result is distribution-dependent since it scales with respect to the rank of the covariance matrix that is defined using the ground truth feature representation.

\section{Preliminaries}

\label{sec:prelim}
We consider a Markov Decision process (MDP) $\Mcal = \{\Scal,\Acal, P,\gamma,r,d_0\}$ where $P:\Scal\times \Acal \to \Delta(\Scal)$ is the transition, $r:\Scal\times \Acal\to [0,1]$ is the reward function, $\gamma\in [0,1)$ is the discount factor, and $d_0\in \Delta(\Scal)$ is the initial state distribution. With slight abuse of notation, we denote the Radon-nikodym derivative of $P$ with respect to a baseline measure $\iota$ by $P$ as well, i.e., $P$ is a probability mass function in the discrete case ($\iota$ is the counting measure) and a probability density function in the continuous setting ($\iota$ is the Lebesgue measure). A policy $\pi$ maps from state (or history) to distribution over actions. Given a policy $\pi$ and a transition distribution $P$, $V^{\pi}_{P}$ denotes the expected cumulative reward of $\pi$ under $P,d_0$ and $r$. Similarly, $Q^{\pi}_P:\Scal\times \Acal \to \RR, A^{\pi}_P:\Scal\times \Acal \to \RR$ are a Q-function and advantage-function under $P$ and $\pi$. Given a transition $P$, we denote $\pi(P)$ as the optimal policy associated with model $P$ under reward $r$.
We also denote $d^{\pi}_P\in \Delta(\Scal\times \Acal)$ as the average state-action distribution of $\pi$ under the transition model $P$, i.e, $d^{\pi}_P=(1-\gamma)\sum_{t=0}^{\infty}\gamma^t d^{\pi}_{P,t}$, where $d^{\pi}_{P,t}\in \Delta(\Scal\times \Acal)$ is the distribution of $(s_{t},a_{t})$ under $\pi$ and $P$ at a time-step $t$. We denote the true transition distribution as $P^\star$, which we do not know in advance. For simplicity, we suppose $r$ is known. The extension to the unknown reward is straightforward. 

In the offline RL setting, we have an offline distribution $\rho\in\Delta(\Scal\times\Acal)$, and an offline dataset $\Dcal = \{s^{(i)},a^{(i)}, r^{(i)}, s'^{(i)}\}_{i=1}^{n}$ which is sampled in the following way: $s,a\sim \rho, r = r(s,a), s'\sim P^\star(\cdot | s,a)$. We hope to obtain $\pi(P^{\star})=\argmax_{\pi}V^{\pi}_{P^\star}$ from this offline dataset without any further interaction with the environment. We often denote $\EE_{\Dcal}[f(s,a,s')]=1/n\sum_{(s,a,s')\in \Dcal}f(s,a,s')$. Our goal is to construct an offline RL algorithm \text{Alg}, which maps from $\Dcal$ to $\pi$ %
so that the suboptimality gap $V^{\pi^{*}}_{P^\star} -  V^{\text{Alg}(\Dcal)}_{P^\star}$ for any comparator policy $\pi^{*}\in \Pi$ is minimized, where $\Pi$ in this work can be an unrestricted policy class (e.g., including non-Markovian policies). Hereafter, $c,c_1,c_2,\cdots$ are always universal constants.

\paragraph{Partial coverage.} Throughout this work, we do not assume $\rho$ has global coverage. The global coverage in this work means that the density ratio based concentrability coefficient $d^{\pi}_{P^\star}(s,a)/\rho(s,a)$ is upper-bounded by some constant $C\in \RR^{+}$ for all polices $\pi\in \Pi$
, or the feature covariance matrix corresponding to the offline distribution $\EE_{s,a\sim \rho} \phi(s,a)\phi(s,a)^{\top}$ ($\phi\in\Scal\times\Acal \to \RR$ is a feature representation) is full rank and has a non-zero minimum eigenvalue, which are commonly used assumptions in offline RL \citep{munos2005error,antos2008learning,ChenJinglin2019ICiB,DuanYaqi2020MOEw}. Under the full coverage, they show the output policy can compete with the globally optimal policy $\pi(P^{\star})$.  However, this assumption may not be true in practice as computing an exploratory policy itself is a challenging task for large-scale RL problems. 
Instead, we are interested in the partial coverage setting such as $d^{\pi^{*}}_{P^\star}(s,a)/\rho(s,a)\leq C$, which means the state-action occupancy measure under some comparator policy $\pi^*$ is covered by the offline dataset. We want to design an algorithm that can compete against \emph{any} policy $\pi^*$ that is covered by the offline data. 
This assumption is much weaker than the global coverage. 
\looseness=-1

\section{Pessimistic Model-based Offline RL }\label{sec:version}

We first introduce a general model-based algorithm that has a PAC guarantee of the suboptimality gap under partial coverage defined with a newly introduced concentrability coefficient. The algorithm takes a realizable model class as input and outputs a policy that is as good as any comparator policy that is covered by the offline data in the sense of the bounded concentrability coefficient.

\subsection{With Total Variation Constraints}

{\ouredit Our algorithm, \emph{Constrained Pessimistic Policy Optimization with total variation constraints (CPPO-TV)} (\pref{alg:main_version}),} takes a realizable hypothesis class $\Mcal$ (with $P^\star\in \Mcal$) consisting of $|\Mcal|$ candidate models as input,  computes the maximum likelihood estimator (MLE) $\widehat{P}_{\mathrm{MLE}}$ using the given offline data $\Dcal = \{s,a,s'\}$. It then forms a min-max objective subject to a constraint. The min-max objective introduces pessimism via searching for the least favorable model $P$ (in terms of its policy's value $V^{\pi}_P$) that is feasible with respect to the constraint. 
We can also express the constrained optimization procedure using a version space $\Mcal_{\Dcal}$ and a policy optimization procedure defined below: 
\begin{align}\label{eq:version}
\max_{\pi\in \Pi} \min_{P\in\Mcal_{\Dcal}} V^{\pi}_{P},\,\, \text{ where }\Mcal_{\Dcal}=\braces{P\mid P\in \Mcal, \mathbb{E}_{\Dcal}\left[ \TV (\widehat{P}_{\MLE}(\cdot | s,a), P(\cdot | s,a))^2\right] \leq \xi}, 
\end{align}
where $\mathrm{TV}(P_1,P_2)$ is a total variation (TV) distance between two distributions $P_1$ and $P_2$. The version space $\Mcal_{\Dcal}$ contains models that are not far away from $\widehat{P}_{\mathrm{MLE}}$ in terms of the average TV distance under $\Dcal$. The version space is constructed such that with high probability $P^\star \in \Mcal_{\Dcal}$.  %

Below we state the algorithm's performance guarantee. Assuming for now that $P^\star \in \Mcal_{\Dcal}$ holds with high probability, then, $\hat V^{\pi}\coloneqq \min_{P\in \Mcal_{\Dcal}}V^{\pi}_P$ is a pessimistic policy evaluation estimator, which satisfies $\hat V^{\pi} \leq V^{\pi}_{P^\star}$ for all $\pi\in\Pi$. Using the idea of pessimism, we have the following observation: 
\begin{align*}
          V^{\pi^{*}}_{P^{\star}}-V^{\hat \pi}_{P^{\star}}=V^{\pi^{*}}_{P^{\star}}-\hat V^{\pi^{*}}+\hat V^{\pi^{*}}- V^{\hat \pi}_{P^{\star}}\leq V^{\pi^{*}}_{P^{\star}}-\hat V^{\pi^{*}}+\hat V^{\hat \pi}- V^{\hat \pi}_{P^{\star}}\leq V^{\pi^{*}}_{P^{\star}}-\hat V^{\pi^{*}}, 
\end{align*} where the first inequality uses $\hat \pi=\argmax_{\pi\in \Pi}\hat V^{\pi}$ and the second inequality uses $\hat V^{\pi} \leq V^{\pi}_{P^\star}$ for all $\pi\in\Pi$. Thus,  the final error only incurs the policy evaluation error for the comparator policy $\pi^{*}$, which leads to the error only depending on the concentrability coefficient for the comparator policy. %

\begin{algorithm}[!t]
\caption{Constrained Pessimistic  Policy Optimization with Total Variation constraints (CPPO-TV) }\label{alg:main_version}
\begin{algorithmic}[1]
    \STATE {\bf Require}: Models $\Mcal$, dataset $\Dcal$, parameter $\xi$, policy class $\Pi$ (note $\Pi$ could be unrestricted)
    \STATE Obtain the estimator $\hat P_{\MLE}$ by MLE: $\widehat P_{\MLE}=\argmax_{P\in \Mcal}\EE_{\Dcal}[\ln P(s'\mid s,a)]$. 
    \STATE Constrained policy optimization:
    	\begin{equation*}
		\hat\pi = \argmax_{\pi\in\Pi} \min_{P\in\Mcal} V^{\pi}_P, \text{ s.t., } \mathbb{E}_{\Dcal}\bracks{ \TV (\widehat{P}_{\MLE}(\cdot | s,a), P(\cdot | s,a))^2 }\leq \xi.	\end{equation*}
    \STATE \textbf{Return} $\hat\pi$%
\end{algorithmic}
\end{algorithm}

We define the following new concentrability coefficient that uses the model class $\Mcal$ :
\begin{definition}[Model-based Concentrability Coefficient]\label{def:partial_con}
For a comparator policy $\pi^*$, we define the concentrability coefficient $C^\dagger_{\pi^*}$ as follows:
\begin{align*}
  C^{\dagger}_{\pi^{*}}=\sup_{P'\in \Mcal}\frac{\EE_{(s,a)\sim d^{\pi^{*}}_{P^{\star}}}[ \TV ({P}'(\cdot | s,a), P^{\star}(\cdot | s,a))^2] }{\EE_{(s,a)\sim \rho}[\TV ({P}'(\cdot | s,a), P^{\star}(\cdot | s,a))^2]}. 
\end{align*} 
\end{definition}
The following theorem shows  CPPO learns a policy that competes against $\pi^*$ when $C^\dagger_{\pi^*}<\infty$. 

\begin{theorem}[PAC Bound for CPPO-TV with general function class] \label{thm:version}
Assume $P^{\star}\in \Mcal$. We set $\xi=c_1\frac{ \ln(c_2|\Mcal| / \delta) }{n} $.  Then, with probability $1-\delta$, for any comparator policy $\pi^*\in \Pi$ ($\Pi$ can be the unrestricted policy class containing non-Markovian policies),  
\begin{align} \label{eq:result}
      V^{\pi^{*}}_{P^{\star}}-V^{\hat \pi}_{P^{\star}} \leq  c_3 (1-\gamma)^{-2}  \sqrt{\frac{C^{\dagger}_{\pi^{*}} \ln(c_2|\Mcal| / \delta) }{n}}. 
\end{align}
\end{theorem} 

To the best of our knowledge, this is the \emph{first} algorithm that achieves a PAC guarantee for \emph{any MDPs} under the partial coverage assumption $  C^{\dagger}_{\pi^{*}}<\infty$ with only a realizable hypothesis class. We emphasize that the inequality in the above \emph{uniformly} holds for \emph{all} policies with probability $1-\delta$ including \emph{history-dependent non-Markovian policies} %
. Note that the ability to compete against non-Markovian policies in offline RL is meaningful  when the offline data does not cover the optimal policy $\pi^\star$ (i.e., there could be a high-quality history-dependent policy that is covered by the offline data against which we want to compete). %
In model-free approaches, this type of result generally cannot be obtained. Indeed, the model-free approach from \cite{XieTengyang2021BPfO} requires $\Pi$ to be a restricted Markovian policy class, since their bound contains $\text{poly}(\ln(|\Pi|))$ dependence. 
For the detailed discussion, refer to Section~\ref{sec:conversion}.  \looseness=-1

The quantity $C^{\dagger}_{\pi^{*}}$ adaptively captures the discrepancy between the offline data and the state-action occupancy measure under a comparator policy $\pi^{*}$ depending on the model class $\Mcal$. For example, $C^{\dagger}_{\pi^{*}}$ can be reduced to a relative condition number in KNRs.  %
Besides, it is  always upper bounded by the density ratio based concentrability coefficient:
\begin{align*} 
    C_{\pi^{*}, \infty}\coloneqq \sup_{(s,a)}\frac{d^{\pi^{*}}_{P^{\star}}(s,a)}{\rho(s,a)}.  
\end{align*}

Prior works that achieve PAC guarantees with only realizable model classes rely on much stronger global coverage $\sup_{\pi}C_{\pi,\infty}<\infty$ \citep{ChenJinglin2019ICiB}. Even when the comparator policy is the optimal policy $\pi(P^{\star})$, the partial coverage condition $C_{\pi(P^{\star}),\infty}<\infty$ is weaker. Existing pessimistic model-based algorithms and their theoretical results \citep{ChangJonathanD2021MCSi} often assume that a \emph{point-wise} model uncertainty measure is given as a by-product of model fitting, which limits the applicability to special linear models such as KNRs/GPs. {\ouredit CPPO-TV} can work for any MDPs with the realizable function class having a valid statistical complexity such that the MLE properly works. \looseness=-1

{\newedit 
\begin{remark}[Variations of Concentrability Coefficients  ]
With a slight modification of the proof, we can obtain the bound \pref{eq:result} where $C^{\dagger}_{\pi^{\star}}$ is replaced with 
\begin{align*}
          C_{\pi^{*}, 2}\coloneqq \EE_{(s,a)\sim \rho }\bracks{ \prns{\frac{d^{\pi^{*}}_{P^{\star}}(s,a)}{\rho(s,a)}}^2 }^{1/2}.
\end{align*}

\end{remark}
}

{\newedit 
\subsection{With Likelihood-ratio Based Constraints }

In \pref{alg:main_version}, the constraint is given using the total variation distance. Here, we propose a similar contained pessimistic policy optimization algorithm in \pref{alg:main_version2}. The only difference compared to \pref{alg:main_version} is that the constraint is given based on the log likelihood-ratio. Since this new constraint is generally easier to calculate than the total variation distance, \pref{alg:main_version} might be preferable compared to \pref{alg:main_version}. 

\begin{algorithm}[!t]
{\newedit 
\caption{Constrained Pessimistic  Policy Optimization with Likelihood-Ratio based constraints  (CPPO-LR) }\label{alg:main_version2}
\begin{algorithmic}[1]
    \STATE {\bf Require}: Models $\Mcal$, dataset $\Dcal$, parameter $\bar \zeta$, policy class $\Pi$ (note $\Pi$ could be unrestricted)
    \STATE Constrained policy optimization:
    \vspace{-5pt}
    	\begin{align*}
		\hat\pi &= \argmax_{\pi\in\Pi} \min_{P\in \bar \Mcal_{\Dcal}} V^{\pi}_P,\,\,\text{s.t.} \\ 
 \bar \Mcal_{\Dcal} &= \{P \in \Mcal: \EE_{\Dcal}[\ln P(s'\mid s,a)]\geq   \max_{P\in \Mcal}\EE_{\Dcal}[\ln P(s'\mid s,a)]- \bar \zeta \}. 
	\end{align*}
    \STATE \textbf{Return} $\hat\pi$%
\end{algorithmic}
}
\end{algorithm}

CPPO-LR has the following same statistical guarantee as the one obtained in \pref{thm:version} for CPPO-TV. 

\begin{theorem}[PAC Bound for CPPO-LR with general function class] \label{thm:version2}
Assume $P^{\star}\in \Mcal$. We set $\bar \zeta=c_1\frac{ \ln(c_2|\Mcal| / \delta) }{n} $.  Then, with probability $1-\delta$, for any comparator policy $\pi^*\in \Pi$ ($\Pi$ can be the unrestricted policy class containing non-Markovian policies),  
\begin{align*} 
      V^{\pi^{*}}_{P^{\star}}-V^{\hat \pi}_{P^{\star}} \leq  c_3 (1-\gamma)^{-2}  \sqrt{\frac{C^{\dagger}_{\pi^{*}} \ln(c_2|\Mcal| / \delta) }{n}}. 
\end{align*}
\end{theorem}

Next, we consider the case where the function class is infinite. To quantify statistical complexities for infinite function classes, we define bracketing numbers as follows \citep{geer2000empirical}. 

\begin{definition}[Bracketing numbers]
Consider a function class $\Fcal$ that maps $\Xcal$ to $\RR$. Given two functions $l(\cdot)$ and $u(\cdot)$, the bracket $[l,u]$ is the set of all functions $f\in \Fcal$ with $l(x)\leq f(x)\leq u(x)$ for all $x \in \Xcal$. An $\epsilon$-bracket is a bracket $[l,u]$ with $\|l-u\|\leq \epsilon$. The bracketing number of $\Fcal$ w.r.t. the metric $\|\cdot\|$ denoted by $N_{[]}(\epsilon,\Fcal,\|\cdot\|$) is the minimum number of $\epsilon$-brackets need to cover $\Fcal$. 
\end{definition}

Using bracketing numbers, we can obtain the guarantee when the function class is infinite. Note $\iota(\Scal)$ is $|\Scal|$ in the discrete state space, and the volume of $\Scal$ in the continuous state space. Recall $\iota(\cdot)$ is a baseline measure.

\begin{theorem}[PAC Bound for CPPO-LR with general function class] \label{thm:version3}
Assume $P^{\star}\in \Mcal$. We set $\bar \zeta=c_1\frac{ \ln(c_2 N_{[]}(\epsilon,\Mcal,\|\cdot\|_{\infty})/ \delta) }{n} $ where $\epsilon =1/(n\iota(\Scal) )$. Then, with probability $1-\delta$, for any comparator policy $\pi^*\in \Pi$ ($\Pi$ can be the unrestricted policy class containing non-Markovian policies),  
\begin{align*} 
      V^{\pi^{*}}_{P^{\star}}-V^{\hat \pi}_{P^{\star}} \leq  c_3 (1-\gamma)^{-2}  \sqrt{\frac{C^{\dagger}_{\pi^{*}} \ln(c_2 N_{[]}(\epsilon,\Mcal,\|\cdot\|_{\infty}) / \delta) }{n}}. 
\end{align*}
\end{theorem}

\begin{remark}[Comparison between CPPO-TV and CPPO-LR]
\pref{thm:version} consider the case where the hypothesis class $\Mcal$ is finite in CPPO-TV. When the hypothesis class is infinite, we can still obtain the PAC guarantee of CPPO-TV by utilizing the generalized result for any realizable model class with valid statistical complexity. However, in this result, we still need certain non-trivial calculations for each model while \pref{thm:version3} just requires the simple calculation of log bracketing numbers of models. For example, this benefit is later seen when we consider linear mixture MDPs.  Refer to \pref{rem:lower_bound}. 
\end{remark}
}

\subsection{Comparison to the model-free approach from \texorpdfstring{\cite{XieTengyang2021BPfO,zanette2021provable}}{aa} }\label{sec:conversion}
\cite{XieTengyang2021BPfO} study the model-free  setting where the function class $\Qcal$ models Q functions assumed to be Bellman complete for any Markovian policy in $\Pi$. 
While directly comparing model-based approaches to model-free approaches is hard as they use different inductive biases in function classes, we can leverage the approach from \citet[Corollary 6]{ChenJinglin2019ICiB} to convert a model class $\Mcal$ to a pair of $\Qcal$ and $\Pi$ class. Specifically, we can convert a model class $\Mcal$ to a pair of $\Qcal$ class and $\Pi$ class such that $\Qcal$ will be realizable and also Bellman complete with respect to all $\pi\in\Pi$. After such conversion from the model-based setting to the model-free setting, running the algorithm from \cite{XieTengyang2021BPfO} using $\Qcal$ and $\Pi$ achieves 
$V^{\pi^{*}}_{P^{\star}}-V^{\hat \pi}_{P^{\star}}=\sqrt{C^{\diamond}\ln(|\Mcal||\Pi|/n}),\forall \pi^{*} \in \Pi$, where $C^{\diamond}$ is some concentrability coefficient. For the detailed derivation, we refer readers to Appendix \pref{ape:comparison}. Since the suboptimality gap from such conversion incurs  $\ln |\Pi|$, a policy class $\Pi$ cannot be too large. Especially, unlike our results, it cannot take the unrestricted policy class as $\Pi$. This restriction \emph{cannot} be fixed even if we use natural policy gradient (NPG) algorithms unless models have special structures  \citep{XieTengyang2021BPfO,zanette2021provable}. The details are given in Section~\ref{ape:comparison}.

In summary, our theorem (\pref{thm:version} and \pref{thm:version2}) indicates two advantages of model-based approaches: (1) realizability in function class is enough to ensure a PAC guarantee under a partial coverage condition, (2) it can compete against a larger pool of candidate policies  including history-dependent non-Markovian policies, which is a meaningful property when the offline data does not cover the globally optimal policy. %

\section{Examples with Refined Concentrability Coefficients}\label{sec:examples}

In the previous section, our results apply to any MDP as long as its true transition belongs to a function class $\Mcal$.  In this section, we consider several concrete MDPs with additional structural conditions. We show that by leveraging the additional structural conditions, we can refine the model-based concentrability coefficient to more natural quantities. The examples that we discuss here are: (1) linear mixture MDPs which generalize linear MDPs from \cite{YangLinF2019RLiF} and tabular MDPs, (3) low-rank MDPs, and (4) factored MDPs.

\subsection{Tabular MDPs}

Tabular MDPs are MDPs where the state and action spaces are finite. Although the corresponding hypothesis class for tabular MDPs is infinite, we can still run MLE, that is, estimating $P^{\star}$ by the empirical distribution. Then, \pref{alg:main_version} and \pref{alg:main_version2} has the following guarantee. %

\begin{corollary}[PAC bound for tabular MDP]\label{cor:tabular}{\ouredit We set $\xi= c_1\frac{|\Scal|^2|\Acal| \ln (n|\Scal|\Acal|c_2/\delta)}{n}$ and denote an output of CPPO-TV (\pref{alg:main_version}) by $\hat \pi$.} %
Then with probability $1-\delta$, for all $\pi^{*}\in \Pi$ ($\Pi$ is the unrestricted policy class),  
\begin{align*}  
V^{\pi^{*}}_{P^{\star}}-V^{\hat \pi}_{P^{\star}}  \leq    c_3(1-\gamma)^{-2}\braces{\sqrt{\frac{ C_{\pi^{*}, \infty}|\Scal|^2|\Acal|\ln(n|\Scal||\Acal|c_4/\delta)}{n} }}. 
\end{align*}
{\ouredit The same statement holds when $\hat \pi$ is an output of CPPO-LR (\pref{alg:main_version2}) by setting $\bar \zeta =  c_1\frac{|\Scal|^2|\Acal| \ln (n|\Scal|\Acal|c_2/\delta)}{n}$. }
\end{corollary}
Here, for tabular MDPs with $\Mcal = \{ P: P(\cdot|s,a) \in \Delta(\Scal), \forall s,a\}$, the model-based concentrability coefficient in \pref{def:partial_con} is equal to the density ratio based concentrability coefficient $C_{\pi^*,\infty}$ which is the right quantity for small-size tabular MDPs.

\subsection{Linear Mixture MDPs}

We define linear mixture MDPs \citep{ayoub2020model,modi2020sample}.  

\begin{definition}[Linear mixture MDPs]
Given a feature vector $ \psi:(\Scal,\Acal,\Scal)\to \RR^d$, a linear mixture MDP is an MDP where the ground truth transition is $ P^\star(s'|s,a) := \theta^{\star \top} \psi(s,a,s') ,\theta^{\star} \in \RR^d. $
\end{definition}
By setting, $\psi(s,a,s') = \mu(s') \bigotimes \phi(s,a)$ ($\otimes$ denotes the Kronecker product), linear mixture MDPs include the following parametric linear MDPs \citep{YangLinF2019RLiF}: 
\begin{definition}[Parametric linear MDPs]\label{def:linear_mdps}
{\ouredit Parametric linear MDP admits a decomposition: 
$$P^{\star}(s' | s,a):= \sum_{i=1}^{d_1}\sum_{j=1}^{d_2}M^{\star}_{ij}\mu_i(s')\phi_j(s,a)$$
with  $\mu:\Scal \to \RR^{d_1}$ and $\phi:\Scal\times \Acal \to \RR^{d_2}$. Here, $\mu$ and $\phi$ are known features, and $M^{\star} \in \RR^{d_1\times d_2}$ is unknown. }
\end{definition} 

We use CPPO to learn on linear mixture MDPs. The corresponding $\Mcal$  is 
\begin{align*}
   \Mcal_{\text{Mix}}=\braces{\theta^{\top}\psi(s,a,s') \mid \theta \in \Theta\subset \RR^d,\int \theta^{\top}\psi(s,a,s')\mathrm{d}\iota(s')=1 \quad \forall(s,a)}.  
\end{align*}

Given a function $V:\Scal\to \mathbb{R}$, define the state-action feature indexed by $V$ as $$\psi_{V}(s,a) :=\int \psi(s,a,s')V(s')\rd\iota(s'),$$ we have the following PAC guarantee.  
{\newedit \begin{corollary}[PAC bound for linear mixture MDPs]\label{cor:linear_mixture}
Suppose %
$\Theta=\{\theta:\|\theta\|_2\leq R\}$, $\|\psi_{V}(s,a)\|_2\leq 1$ for any $ V\in \{\Scal\to [0,1]\}$, and $P^{\star}\in  \Mcal_{\text{Mix}}$. Let $\hat \pi$ be the output of CPPO-LR (\pref{alg:main_version2}) when we set $\bar \zeta=c_1 d\ln(c_2 nR\iota(\Scal)/\delta)/n$. Then, with probability $1-\delta$, for any  $\pi^{*}$ in  $\Pi$ (again $\Pi$ can be the unrestricted policy class), CPPO outputs a policy $\hat\pi$ such that: 
\begin{align}\label{eq:linear_mixture_error}
      V^{\pi^{*}}_{P^{\star}}-V^{\hat \pi}_{P^{\star}} \leq   c_3 (1-\gamma)^{-2}  \sqrt{ \min(dC^\dagger_{\pi^{*}},d^2\bar C_{\pi^{*},\mathrm{mix}}) \frac{\ln(c_4 nR \iota(\Scal)/\delta)}{n}   },
\end{align} where the concentrability coefficient $\bar C_{\pi^*, \mathrm{mix}}$ is defined as:
\begin{align*}
    \bar C_{\pi^{*},\mathrm{mix}} :=\sup_{P\in \Zcal_{P^\star} }\sup_{x \in \mathbb{R}^d}\left(\frac{x^{\top}\Sigma_{\pi^{*},\psi_{V^{\pi^{*}}_P}} x}{x^{\top}\Sigma_{\rho,\psi_{V^{\pi^{*}}_P}}x}\right)
\end{align*}
with the localized class $\Zcal_{P^\star} :=\{P \in \Mcal_{\text{mix}}: \E_{(s,a)\sim \rho}[\TV(P(\cdot \mid s,a),P^{\star}(\cdot \mid s,a))^2]\leq \bar \zeta\}$, $\Sigma_{\rho, \psi_{V^{\pi^*}_{P}}}=\EE_{(s,a)\sim \rho}[\psi_{V^{\pi^{*}}_P}(s,a)\psi_{V^{\pi^{*}}_P}(s,a)^{\top}]$, and $\Sigma_{\pi^*, \psi_{V^{\pi^*}_{P}}} = \EE_{s,a\sim d^{\pi^*}_{P^\star}}[\psi_{V^{\pi^{*}}_P}(s,a)\psi_{V^{\pi^{*}}_P}(s,a)^{\top}]$.

When specializing to parametric linear MDPs, the above bound still holds with $  \bar C_{\pi^{*},\mathrm{mix}}$ being replaced by the relative condition number $\bar C_{\pi^*,\phi}$:
\begin{align*}
\bar C_{\pi^*,\phi} := \sup_{x\in\mathbb{R}^d} \frac{x^T \Sigma_{\pi^*} x}{ x^{\top} \Sigma_\rho, x}, \text{ where } \Sigma_{\rho}=\EE_{(s,a)\sim \rho}[\phi(s,a)\phi(s,a)^{\top}],\,\Sigma_{\pi^{*}}=\EE_{(s,a)\sim d^{\pi^{*}}_{P^{\star}} }[\phi(s,a)\phi(s,a)^{\top}].  %
\end{align*}
\end{corollary}
}
This is the first PAC-guarantee result in the offline setting under partial coverage 
$\bar C_{\pi^{*},\mathrm{mix}}<\infty$ for linear mixture MDPs. The quantity $\bar C_{\pi^{*},\mathrm{mix}}$ is a newly-introduced concentrability coefficient for linear mixture MDPs. 
This coefficient is measured on the integrated feature vectors $\phi_{V}(s,a)$ for $V:S\to [0,1]$. %
Note the class of $V$ is localized, i.e., we consider state-value functions $V^{\pi^*}_{P}(s)$ for all $P$ centered around $P^{\star}$ under data distribution $\rho$ (i.e., $P\in\Zcal_{P^\star}$). Such localization property ensures that $\bar C_{\pi^*,\mathrm{mix}} \leq C^\dagger_{\pi^\star}$
(see \pref{lem:mixture} in \pref{sec:auxi}). %

Note that these relative condition number based quantifiers are always tighter than the density ratio based concentrability coefficients (i.e., $\max\{\bar  C_{\pi^{*}}, \bar C_{\pi^{*},\mathrm{mix}}  \}\leq  C_{\pi^{*},\infty}$). %
For the special case where $\phi(s,a)$ is a one-hot encoding vector, then they are reduced to the density ratio based concentrability coefficient. In a non-tabular setting, even if when the density ratio is infinite, the relative condition number can still be finite. Intuitively, the bounded relative condition number implies that the offline data covers the subspace that the comparator policy $\pi^*$ visits.

We finally remark %
the norm assumption $\|\psi_{V}(s,a)\|_2<1$ is commonly assumed in the online setting \citep{zhou2021nearly}.  

{\newedit \begin{remark}[Guarantee of CPPO-TV]\label{rem:lower_bound}
Corollary \ref{cor:linear_mixture} is for CPPO-LR. Under $\inf_{s,a,s'}P^{\star}(s'\mid s,a)\geq c_3>0$, we can ensure the similar guarantee for CPPO-TV. However, apparently, it is not obvious how to relax this assumption when we use CPPO-TV. 
\end{remark}
}

\subsection{Low-rank MDPs with Representation Learning}

\label{sec:feature_learning}
We consider the representation learning in offline RL. Following FLAMBE \citep{Agarwal2020_flambe}, we study low-rank MDPs but in the offline setting. Note that low-rank MDPs here are a more generalized model of the aforementioned parametric linear MDPs \citep{YangLinF2019RLiF} since the true feature representation $\phi^{\star}$ in a low-rank MDP is unknown. %

\begin{definition}[Low rank MDPs]
The ground-truth model $P^{\star}$ admits a low rank decomposition with a dimension $d$ if there exists two embedding functions $\mu^{*}:\Scal \to \RR^d,\phi^{*}:\Scal \times \Acal \to \RR^d$ s.t. $P^{\star}(s'\mid s,a)=\mu^{*}(s')^{\top}\phi^{*}(s,a)$.  Neither $\mu^*$ nor $\phi^*$ is known to the learner. 
\end{definition}

One interesting special case of a low-rank MDP is the following latent variable model (see \cite{Agarwal2020_flambe} for more details).
\begin{definition}[Latent variable models] \label{def:latent_vab}
There exists a latent space $\Zcal$ along with functions $\mu^{*}:\Zcal \to \Delta(\Scal)$ and $\phi^{*}:\Scal\times \Acal \to \Delta(\Zcal)$ s.t. $P^{\star}(\cdot \mid s,a)=\sum_{z\in \Zcal}\mu^{*}(\cdot \mid z)\phi^{*}(z\mid s,a)$. 
\end{definition}

To tackle representation learning under partial coverage on low-rank MDPs,  we setup function classes as follows: given two function classes $\Psi \subset \Scal\to \RR^{d},\Phi\subset \Scal\times \Acal \to \RR^{d}$ (both are realizable in the sense that $\mu^*\in \Psi$ and $\phi^*\in \Phi$), we consider a hypothesis class $
    \{\mu(s')^{\top}\phi(s,a);\mu\in \Psi,\phi \in \Phi  \}.$
Then, CPPO (\pref{alg:main_version}) and \pref{thm:version} still work under this setting. Note that this function class setup is exactly the same as the one from FLAMBE.

Here we show that by leveraging the low-rankness, we can refine the concentrability coefficient to a relative condition number defined by the unknown true representation $\phi^*$. We emphasize that this does not depend on the other features. Particularly, given a comparator policy $\pi^*$, we define $\bar C_{\pi^*,\phi^{\star}}$: 
\begin{align*}
\bar C_{\pi^*,\phi^{\star}} = \sup_{x\in \RR^d} \frac{ x^{\top}  \Sigma_{\pi^{*}}   x  }{ x^{\top} \Sigma_{\rho} x},\quad \Sigma_{\pi^*} := \mathbb{E}_{s,a\sim d^{\pi^{*}}_{P^{\star}}} \phi^*(s,a) \phi^*(s,a)^{\top},\quad \Sigma_{\rho} := \mathbb{E}_{s,a\sim \rho} \phi^*(s,a) \phi^*(s,a)^{\top}. 
\end{align*}

We can show CPPO learns a policy that can compete against $\pi^*$ as long as $\bar C_{\pi^*,\phi^{\star}} <\infty$.

\begin{theorem}[PAC bound for low-rank MDP]\label{thm:low_rank}
 We set $\xi=c_1 \frac{ \ln(|\Phi||\Psi|c_2 / \delta) }{n} $.  
Suppose (a): $\|\phi(s,a)\|_2\leq 1, \forall (s,a)\in \Scal\times\Acal$ for any $\phi \in \Phi$, $\int \mu(s')^{\top}\phi(s,a)\rd \iota(s')=1$ and $\int \|\mu(s)\|_2 \mathrm{d}\iota(s)\leq \sqrt{d}$ for any $ \mu\in \Psi,\phi\in \Phi$, (b) $\rho(s,a)=d^{\pi_b}_{P^{\star}}(s,a)$, (c) $P^{\star}(s'|s,a)=\mu^{*}(s')^{\top}\phi^{*}(s,a)$ for some $\mu^{*}\in \Psi,\phi^{*}\in \Phi$. With probability at least $1-\delta$, for all $\pi^*\in \Pi$ (again $\Pi$ can be an unrestricted policy class), {\newedit CPPO-TV (\pref{alg:main_version}) and CPPO-LR (\pref{alg:main_version2}) find $\hat\pi$} such that:
\begin{align}\label{eq:low_rank} 
     V^{\pi^{*}}_{P^{\star}}-V^{\hat \pi}_{P^{\star}}\leq    c_3 \sqrt{\bar C_{\pi^{*},\phi^{\star}}\omega_{\pi^{*}}\rank(\Sigma_{\rho}) \frac{\ln(|\Psi||\Phi|c_4/\delta)}{(1-\gamma)^4 n}},\,\omega_{\pi^{*}}=\prns{\max_{(s,a)}\frac{\pi^{*}(a\mid s)}{\pi_b(a\mid s)}}
\end{align}
\end{theorem}

To the best of our knowledge, this is the first established PAC result under the partial coverage condition $\bar C_{\pi^*,\phi^{\star}} <\infty,\omega_{\pi^{*}}<\infty$ for low-rank MDPs in the offline setting. %
We also emphasize that our bound in \pref{thm:low_rank} is distribution dependent, i.e., it depends on $\text{rank}(\Sigma_{\rho})$ rather than the exact rank $d$. Note that $\text{rank}( \Sigma_{\rho}) \leq d$, and $\text{rank}( \Sigma_{\rho})$ could be much smaller than $d$ when the offline distribution only concentrates on a low-dimensional subspace (defined using $\phi^*$). Note that the assumption that $\omega_{\pi^*} <\infty$ does not imply the state-action density ratio $C_{\pi^*,\infty}$ is small. Indeed, $\omega_{\pi^*} <\infty$ is much weaker than $C_{\pi^*,\infty}<\infty$.  \looseness=-1%

\subsection{Factored MDPs}\label{sec:factored}

The last example we include is the factored MDP \citep{kearns1999efficient} defined as follows: 

\begin{definition}[Factored MDPs] Let $d \in \mathbb{N}^+$ and $\Ocal$ being a small finite set. The state space $\Scal = \Ocal^{d}$, and for each state $s$, we denote $s[i] \in \Ocal$ as the $i$-th variable of the state $s$.  For each $i\in [1,\cdots,d]$, the parents of $i$, $\pa_i \subset [1,\cdots,d]$, is the subset of state variables that directly influences $i$, i.e., the transition is defined as follows:
\begin{align*}
\forall s,a,s': P^\star(s' | s,a) = \prod_{i=1}^d  P^\star_i( s'[i]  | s[\pa_i], a). 
\end{align*} 
We will denote $\Scal_i = \Ocal^{|\pa_i|}$, and given $s\in\Scal$, we will have $s[\pa_i] \in \Scal_i$
\end{definition} 

Due to the factorization, the transition operator $P^\star$ can be described with $L \coloneqq \sum_{i=1}^d |\Acal| |\Ocal|^{1 + |\pa_i|}$ many parameters. In contrast, the non-factored transition will need $O(|\Ocal|^d)$ parameters. When $|\pa_i| \ll d\, \forall i$, it is expected that we can learn this model with lower sample complexity by leveraging the factorization which has been demonstrated in the online setting \citep{kearns1999efficient}.  We remark a factored MDP is an example where model-based approaches are necessary as neither the optimal policy nor the Q functions  are factored \citep{koller2000policy}. 

\paragraph{Algorithm.}

{\newedit Next, we consider the algorithm. While \pref{alg:main_version} (CPPO-TV) and \pref{alg:main_version2} (CPPO-LR) can ensure partial coverage results in terms of $C_{\pi^{\star},\infty}$, we modify these algorithms to obtain more refined results so that we can take the factored structure into account. }

We consider the modification of CPPO-TV. First, we perform MLE for model learning: %
each factor $P^\star_i$ is independently learned via MLE: 
\begin{align*}  
\forall i\in [d], \widehat{P}_{\MLE,i} = \argmax_{P}\E_{\Dcal}[\ln P(  s'[i] | s[\pa_i], a  )], \quad \widehat{P} = \prod_{i} \widehat{P}_{\MLE,i}.
\end{align*}
Next, the constrained policy optimization procedure is defined as
\begin{align}\label{eq:modified}
\hat\pi = \argmax_{\pi} \min_{P := \prod_{i} P_i} V^{\pi}_P, \; \text{s.t.},  \mathbb{E}_{\Dcal}[\TV(P_i(\cdot \mid s,a) , \widehat{P}_{\MLE,i}(\cdot \mid s,a))^2] \leq \xi_i\,(\forall i\in [1,\cdots,d]).
\end{align}
{\ouredit Compared to the original CPPO-TV, we modify the constraint so that the constraint is factored as well.}  
Note that in the above objective, there is no restriction on the policy, i.e., the $\argmax$ operator searches over all possible policies including non-Markovian ones. 

{\newedit Next, we consider the modification of CPPO-LR. The algorithm is given as follows: 
\begin{align}\label{eq:modified2}
\hat\pi = \argmax_{\pi} \min_{P := \prod_{i} P_i} V^{\pi}_P, \; \text{s.t.},  \mathbb{E}_{\Dcal}[\log(\widehat{P}_{\MLE,i}(\cdot \mid s,a)/{P}_{i}(\cdot \mid s,a))] \leq \bar \zeta_i\,(\forall i\in [1,\cdots,d]).
\end{align}
 Compared to the original CPPO-LR, we modify the constraint so that the constraint is factored as well.}

\paragraph{Analysis.} To analyze the performance of the above modified CPPO, we  introduce a specialized concentration coefficient for factored MDPs that utilizes the factored structure. We focus on density ratio based concentrability coefficients since in a factored MDP with the function class $\Mcal := \{P = \prod_i P_i:  P_i \in \Scal_i\times\Acal \to \Delta(\Ocal)\}$, the concentrability coefficient associated with $\Mcal$ in \pref{def:partial_con} will be reduced to the density ratio. For any $\pi^*$, we define the concentrability coefficients for the factored MDP as follows: \looseness=-1
\begin{align*}
\ddot{C}_{\pi^{*}, \infty}\coloneqq \max_{j\in [1,\cdots,d]}\max_{s_j \in \Scal_j,a\in \Acal} \frac{d^{\pi^{*}}_{P^{\star}}(s_j,a) }{\rho(s_j,a)}, 
\end{align*}
where for $s_j\in\Scal_j$, we denote $\nu(s_j, a) := \sum_{s\in\Scal : s[\pa_j] = s_j} \nu(s,a)$ for any distribution $\nu\in\Delta(\Scal\times\Acal)$. 
Comparing to $C_{\pi^*, \infty}$ defined on the original state space $\Scal$, here $\ddot{C}_{\pi^{*}, \infty}$ is defined over each state space $\Scal_j$ associated with each factor $j$. Note that when $|\text{pa}_j| = \Theta(1)$, $|\Scal_j|$ is exponentially smaller than $|\Scal|$. One can verify that  $\ddot{C}_{\pi^{*}, \infty}\leq {C}_{\pi^{*}, \infty}$ where ${C}_{\pi^{*}, \infty}$ ignores the factored structure and treat $\Scal$ as a whole single space. This formally demonstrates the benefit of the factored structure in terms of the coverage condition in offline RL.
{\ouredit 
\begin{lemma}\textbf{(Comparison of density-ratio based concentrability coefficients  between factorized MDPs and non-factored MDPs)} \label{lem:comparison}
We have 
     $\ddot{C}_{\pi^{*}, \infty}\leq {C}_{\pi^{*}, \infty}$
\end{lemma}
}

With the new definition of the concentrability coefficients, now we are ready to state the PAC bound of CPPO for factored MDPs. Recall $L \coloneqq \sum_{i=1}^d L_i,L_i=|\Acal| |\Ocal|^{1 + |\pa_i|}$.

\begin{theorem}[PAC bound for factored MDP]\label{thm:factoed}
We set $\xi_i= c_1 \frac{L_i\ln(L_ic_2d/\delta)}{n}$ and $\bar \zeta_i= c_1 \frac{L_i\ln(L_ic_2d/\delta)}{n}$ . %
Then with probability $1-\delta$, {\newedit modified CPPO-TV \pref{eq:modified} and modified CPPO-LR \pref{eq:modified2}} find a policy $\hat\pi$ such that for all comparator policy $\pi^{*}\in\Pi$ ($\Pi$ can be unrestricted),%
\begin{align*}  
V^{\pi^{*}}_{P^{\star}}-V^{\hat \pi}_{P^{\star}}  \leq    c_3(1-\gamma)^{-2}\sqrt{\frac{ d\ddot{C}_{\pi^{*}, \infty} L\cdot \ln (nL c_4 d/\delta)}{n} }.%
\end{align*} %
\end{theorem}
Note that our sub-optimality gap scales polynomially with respect to $L$, i.e., the complexity of the factored MDP, rather than $|\Scal|$ which can be $\Omega(\exp(d))$. {\ouredit Importantly, the bound does not scale with ${C}_{\pi^{*}}$, which will be obtained using the original CPPO-TV and CPPO-LR. Instead, it scales with $\ddot{C}_{\pi^{*}}$, which is expected to be much smaller than ${C}_{\pi^{*}}$ from \pref{lem:comparison}.  }

{\newedit \begin{remark}[Improved Concentrability Coefficients] For interpretability, in the above theorem, we use density ratio based concentrability coefficient. We remark that indeed $\ddot{C}_{\pi^{*}, \infty}$ can be replaced with an $L_2$-based concentrability coefficient $$\ddot{C}_{\pi^{*}, 2}  = \max_{j \in [1,\cdots,d]}\EE_{(s_j,a)\sim \rho }\bracks{ \prns{\frac{d^{\pi^{*}}_{P^{\star}}(s_j,a)}{\rho(s_j,a)}}^2 }^{1/2}. $$
In this $L_2$-form, we can still leverage the factorized structure of factored MDPs using the following lemma. 

\begin{lemma}\textbf{(Comparison of density-ratio based concentrability coefficients  between factorized MDPs and non-factored MDPs)}\label{lem:comparison2}
     $$\ddot{C}_{\pi^{*}, 2} \leq  {C}_{\pi^{*}, 2}.$$
\end{lemma}
\end{remark}
}

\section{Constrained Pessimistic Model-Based Policy Optimization for KNRs}\label{sec:knrs}

{\newedit We consider the example of KNRs \citep{Kakade2020,CuriSebastian2020EMRL} in this section. 
More specifically, we tailor CPPO-TV and CPPO-LR to obtain tight guarantees on KNRs. Although the partial coverage results in KNRs have already been obtained in \citet{ChangJonathanD2021MCSi} with bonus-based pessimistic policy optimization, we aim to demonstrate the wide applicability of our constrained pessimistic model-based RL framework. }

\subsection{Finite Dimensional Kernelized Nonlinear Regulators}

A kernelized Nonlinear Regulator (KNR) \citep{Kakade2020} is a model where the ground truth transition $P^{\star}(s'|s,a)$ is defined as $s' = W^\star\phi(s,a) + \epsilon$, $\epsilon\sim \Ncal(0,\zeta^2 \Ib)$, with $\phi:\Scal\times \Acal\to \RR^d$ being a possibly nonlinear feature mapping. 
We denote the corresponding model on $W$ by $P(W)$.
We can apply \pref{alg:main_version} and obtain its guarantee. Especially, since $\TV(P(W)(\cdot \mid s,a), P(W^{\star})(\cdot \mid s,a))^2=\Theta(\|(W-W^{\star})\phi(s,a)\|^2_2 ) $ \citep{devroye2018total}, $C^{\dagger}_{\pi^{*}}$ is upper-bounded by the relative condition number $d\bar  C_{\pi^{*},\phi}$ as follows. 
{\newedit 
\begin{lemma}[Model-based Concentrability Coefficient for KNRs ] \label{lem:knrs_basic}
In KNRs, we have 
\begin{align*}
    C^{\dagger}_{\pi^{*}}\leq d\bar  C_{\pi^{*},\phi}. 
\end{align*}
\end{lemma}
}

We tailor \pref{alg:main_version} to KNRs as follows to obtain a tighter guarantee.   First, MLE procedure is replaced with $\hat W_{\MLE}$ by regularized MLE:
\begin{align*}
     \hat W_{\MLE}=\argmin_{W\in \mathbb{R}^{d_{\Scal}\times d}}\EE_{\Dcal}[\|W\phi(s,a)-s'\|^2_2]+\lambda \|W\|^2_F, 
\end{align*}
where $\|\cdot\|_F$ is a Frobenius norm. Then, the final policy optimization procedure is 
\begin{align*}
    \hat \pi=\argmax_{\pi\in \Pi}\min_{W\in \Wcal_{\Dcal}}V^{\pi}_{P(W)},\,\mathrm{s.t.}, \Wcal_{\Dcal}=\{W \in \RR^{d_{\Scal}\times d}: \|(\hat W_{\MLE}-W)(\Sigma_n)^{1/2}\|_2 \leq \xi\}
\end{align*}
where $\Sigma_n=\sum_{i=1}^{n}\phi(s_i,a_i)\phi^{\top}(s_i,a_i) + \lambda I.$ We state the theoretical guarantee for KNRs below.

\begin{corollary}[PAC bound for KNRs]\label{cor:knrs}
Assume $\|\phi(s,a)\|_2\leq 1,\forall(s,a)\in \Scal\times \Acal$. 
We set $$\xi=   \sqrt{ 2\lambda \|W^\star\|^2_2  + 8 \zeta^2 \left(d_{\Scal} \ln(5) +  \ln(1/\delta) +  \bar \Ical_{n} \right) }, \quad   \bar \Ical_{n}=\ln\left( \det(\Sigma_{n}) / \det(\lambda \Ib) \right).$$  Suppose the KNR model is well-specified. By letting $\|W^\star\|^2_2=O(1),\zeta^2=O(1),\lambda=O(1)$, with probability $1-\delta$, for all $\pi^{*} \in \Pi$, we have 
\begin{align*}
 V^{\pi^{*}}_{P^{\star}}-V^{\hat \pi}_{P^{\star}} \leq c_1 (1-\gamma)^{-2} \min(d^{1/2},  \bar R)\sqrt{ \bar R }   \sqrt{\frac{d_{\Scal}\bar  C_{\pi^{*}, \phi}\ln (c_2n/\delta) }{n}}, \quad \text{where }\bar R :=  \mathrm{rank}[\Sigma_{\rho}]\{\mathrm{rank}[\Sigma_{\rho}] +\ln(c_2/\delta)\}.
\end{align*}
\end{corollary}

This implies CPPO can learn a policy that can compete against $\pi^{*}$ with partial coverage $\bar  C_{\pi^{*}}<\infty$. Then, we can also recover the result of \citet{ChangJonathanD2021MCSi} which proposes a reward penalty-based pessimistic offline RL algorithm. Note that the condition $\bar  C_{\pi^{*}}<\infty$ does not require $\Sigma_{\rho}$ to be full-rank. Also the bound uses $\rank[\Sigma_{\rho}]$ instead of $d$, which means that our bound is distribution dependent and is still valid even when $d=\infty$ as long as the offline data only concentrates on a low-dimensional subspace.

{\newedit  

\subsection{KNRs with RKHS (Gaussian Processes)  }\label{ape:rkhs}

We consider the ground truth model $P^{\star}(s'|s,a)$ defined as $s'=g^{\star}(s,a)+\epsilon,\,\epsilon\sim \Ncal(0,\upsilon^2\Ib)$ where each component in $g^\star$ belongs to an RKHS $\Hcal_k$ with a kernel $k(\cdot,\cdot)$ \citep{CuriSebastian2020EMRL}.  We assume $s\in\mathbb{R}^{d_{\Scal}}$. %

We denote a $d_{\Scal}$ dimensional RKHS as $\bigoplus \Hcal_k$. We also denote the corresponding model to $g$ by $P(g)$. We can learn $g^{\star}$ by regularized MLE (kernel ridge regression) using offline dataset $\Dcal$. We tailor \pref{alg:main_version} and \pref{alg:main_version2} as follows. 

First, by letting $x:=(s,a),x_i=(s_i,a_i)$, MLE procedure is replaced with 
 $\hat g_{\MLE}$ by regularized MLE: 
\begin{align*} 
    &\hat g_{\MLE}(\cdot)    =S(\Kb_{n}+\zeta^2 \Ib)^{-1}\bar k_{n}(\cdot) , \quad S=[s'_1,\cdots,s'_{n}]\in  \mathbb{R}^{d_{\Scal}\times n}, \quad \bar k_{n}(x)  =[k(x_1,x),\,\cdots,k(x_{n},x)]^{\top},\\
    &\{\Kb_{n}\}_{i,j}=k(x_i,x_j)\,(1\leq i\leq n,1\leq j\leq n),\, k_{n}(x,x')  =k(x,x')-\bar k_{n}(x)^{\top}(\Kb_{n}+\zeta^2 \Ib)^{-1}\bar k_{n}(x'), 
\end{align*}
where the notation $\|\cdot\|_{k_{n}}$ is a norm associated with an RKHS with a kernel $k_n(\cdot,\cdot)$. The final optimization procedure is replaced with 
\begin{align*}
    \hat \pi=\argmax_{\pi\in \Pi}\min_{g\in \Gcal_{\Dcal}}V^{\pi}_{P(g)},\mathrm{s.t.},\Gcal_{\Dcal}=\{g \in \bigoplus \Hcal_k : \sum_{i=1}^{d_{\Scal}}\|\hat g_{i,\MLE}-g_i\|^2_{k_{n}} \leq \xi^2\}. 
\end{align*}

We state the theoretical guarantee for KNRs with RKHS below. Before proceeding to the result, we prepare several notations and definitions. For simplicity, following \citet{Srinivas2010}, we suppose as follows: 
\begin{assum}\label{assm:kernel_mercer}
$k(x,x)\leq 1,\forall x\in \Scal\times \Acal$ and there exists a set of pairs of eigenvalues and eigenfunctions $\{ \mu_i, \psi_i \}_{i=1}^{\infty}$, where $\int \rho(x) \psi_i(x) \psi_i(x) dx  = 1$ for all $i$ and $\int \rho(x) \psi_i(x)\psi_j(x) dx = 0$ for $i\neq j$.
\end{assum}
The above is ensured by Mercer's theorem \citep{Rasmussen2005}. Eigenfunctions and eigenvalues essentially defines an infinite-dimensional feature mapping $\phi(x) := [ \sqrt{\mu_1} \psi_1(x), \dots, \sqrt{\mu_\infty} \psi_\infty(x) ]^{\top}$. 
By setting eigenvalues $\{\mu_1,\dots, \mu_\infty\}$  in non-increasing order, we define the effective dimension below: %
\begin{definition}[Effective dimension]
     $d^{*}  =   \min\{j \in \mathbb{N}: j\geq B(j+1)n/\zeta^2\},\, B(j)=\sum_{k=j}^{\infty} \mu_k$. 
\end{definition}
The effective dimension $d^{*}$ is commonly used and calculated for many kernels \citep{ZhangTong2005LBfK,BachFrancis2017OtEb,Valko2013}. 
In finite-dimensional linear kernels $\{x\mapsto a^{\top}\phi(x);a\in \RR^d\}$ ($k(x,x)=\phi^{\top}(x)\phi(x)$), we have $d^{*}\leq \rank[\Sigma_{\rho}]$. Thus, $d^{*}$ is regarded as a natural extension of $\rank[\Sigma_{\rho}]$ to infinite-dimensional models. Note that $d^*$ itself is offline distribution dependent, i.e., the eigenvalues and eigenfunctions are defined using the offline distribution $\rho$. 

With the above preparations in mind, we present the theoretical result below.
\begin{corollary}[PAC bound for RKHS models]\label{cor:gps}
Let $\Sigma_{\pi^{*}}=\EE_{(s,a) \sim d^{\pi^{*}}_{P^{\star}}}[\phi(s,a)\phi(s,a)^{\top}],\Sigma_{\rho}=\EE_{(s,a) \sim \rho}[\phi(s,a)\phi(s,a)^{\top}]$.
We set $\xi$:
\begin{align*}
     \xi=\sqrt{d_{\Scal}\{2+150 \ln^3(d_{\Scal}n/\delta)\mathcal{I}_{n}\}} ,\quad \mathcal{I}_{n}=\ln(\det(\Ib+\zeta^{-2}\Kb_{n}))
\end{align*} and $\upsilon^2=O(1)$.
With probability at least $1-\delta$, for all comparator policy $\pi^* \in \Pi$, we have:
\begin{align*}
      V^{\pi^{*}}_{P^{\star}}-V^{\hat \pi}_{P^{\star}} &\leq     c_1 (1-\gamma)^{-2}\{d^{*}+\ln(c_2/\delta)\}d^{*}\sqrt{\frac{d_{\Scal}\bar  C_{\pi^{*}, \phi}\ln^3(c_3d_{\Scal}n/\delta)\ln(n) }{n}}, 
\end{align*}
\end{corollary}

This implies the algorithm has a valid PAC guarantee under the partial coverage of MDPs with RKHS. %

}

{\newedit 
\section{Constrained Pessimistic Model-based Policy Optimization for (nonparametric) linear MDPs }\label{sec:linear_mdps}

CPPO cannot directly capture (nonparametric) linear MDPs in \citet{Jin2020}, which is different from the one in \citet{YangLinF2019RLiF} without any modification since MLE is no longer applicable to them. However, with slight modification, we can learn nonparametric linear MDPs from model-based viewpoints. Although the partial coverage results in linear MDPs have already been obtained in \citet{XieTengyang2021BPfO,zanette2021provable,zhang2021corruption}, in this section, we aim to demonstrate the wide applicability of the pessimistic model-based RL framework. 

We first define (nonparametric) linear MDPs. 

\begin{definition}[Nonparametric linear MDPs in \citet{Jin2020}]
Linear MDPs admit the following decomposition: 
\begin{align*}
 P^{\star}(s'\mid s,a)=\langle \mu^{\star}(s'), \phi(s,a)\rangle
\end{align*}
where $\phi:\Scal \times \Acal \to \RR^d$ is a known feature. Parameters $\theta^{\star} \in \RR^d$ and $\mu^{\star}(s'):\Scal \to \RR^d$ are unknown to learners. 
\end{definition}

In linear MDPs, the model $\Mcal$ is 
\begin{align*}
    \Mcal =\braces{\langle \mu(s'), \phi(s,a)\rangle : \left \langle \int \mu(s')\mathrm{d}\iota(s'), \phi(s,a) \right\rangle = 1\,(\forall (s,a)), \mu:\Scal \to \RR^d}. 
\end{align*}
Since the restriction on $\mu(s)$ is nonparametric, it is difficult to perform standard MLE. However, following \cite{lykouris2021corruption,neu2020unifying}, we can still learn models using other objective functions. 

As a first step, we introduce a witness function $v:\Scal \to [0,1]$ to facilitate the learning. Instead of directly estimating $\mu^{\star}(\cdot)$, we aim to estimate $\int \mu^{\star}(s)v(s)\mathrm{d}\iota(s)$. Especially, in the tabular case, for $\tilde s\in \Scal$, by taking $v(s)=I(s=\tilde s)$, it amounts to estimate $\mu(\tilde s)$. Informally, in the non-tabular case, by taking $v(s)$ as a Dirac delta at $\tilde s$, it amounts to estimate $\mu(\tilde s)$ as well. Then, since 
\begin{align*}
    \int P^{\star}(s'\mid s,a)v(s')\mathrm{d}\iota(s')= \left \langle \phi(s,a), \int \mu^{\star}(s')v(s')\mathrm{d}\iota(s') \right \rangle 
\end{align*}
it is natural to perform regularized least squares:
\begin{align*}
    \hat \theta_v = \argmin_{ \theta \in \RR^d  } \mathrm{E}_{\Dcal}[\{v(s') - \langle \phi(s,a), \theta \rangle\}^2] + \lambda \|\theta\|^2_2. 
\end{align*}
The analytical form of $\hat \theta_v$  is as follows:
\begin{align*}
    \hat \theta_v = \langle \phi(s,a) , (\Lambda_n/n)^{-1}\E_{\Dcal}[\phi(s,a)v(s')] \rangle,\quad \Lambda_n = n\E_{\Dcal}[\phi(s,a)\phi^{\top}(s,a)]+\lambda I. 
\end{align*}
Finally, after introducing certain function class $\Vcal\coloneqq \{s \mapsto \phi(s,\pi);\pi \in \Pi \}$, the estimator $\hat P$ is the one satisfying (not needed to be unique nor in $\Mcal$) $$\int \hat P(s'\mid s,a)v(s')\mathrm{d}\iota(s')=\langle \phi(s,a),   \hat \theta_v \rangle $$ for any $v \in \Vcal$ and $(s,a)\in \Scal \times \Acal$. This choice of $\Vcal$ is determined so that it includes a set of state value functions for each model in $\Mcal$. We remark the analog of MLE equipped with witness function classes is widely used, e.g., in \citet{Sun2019_model}. Using this $\hat P$, we introduce constrained pessimistic policy optimization for linear MDPs in \pref{alg:main_version3}. Note in \pref{alg:main_version3}, what we need to know is not $\hat P$ itself but $\int \hat P(\tilde s \mid s,a)v(\tilde s)\rd\iota(\tilde s)$.

\begin{algorithm}[!t]
{\newedit 
\caption{Constrained Pessimistic Policy Optimization for Nonparametric linear MDPs }\label{alg:main_version3}
\begin{algorithmic}[1]
    \STATE {\bf Require}: Models $\Mcal$, dataset $\Dcal$, parameter $\xi$, policy class $\Pi$, $\Vcal = \{\phi(\cdot,\pi) ; \pi \in \Pi\}$
    \STATE Constrained policy optimization:
    \vspace{-5pt}
    	\begin{align*}
		\hat\pi &= \argmax_{\pi\in\Pi} \min_{P\in  \Mcal_{\text{linear},\Dcal}} V^{\pi}_P, \\ 
\Mcal_{\text{linear},\Dcal} &= \left\{P \in \Mcal: \sup_{v \in \Vcal}\EE_{\Dcal}\left[ \left|\int \{\hat P(\tilde s\mid s,a)-P(\tilde s\mid s,a)\}v(\tilde s)\mathrm{d}\iota(\tilde s)\right |^2\right]\leq \zeta \right\}.
	\end{align*}
    \STATE \textbf{Return} $\hat\pi$%
\end{algorithmic}
}
\end{algorithm}

\begin{theorem}[PAC bound for linear MDPs]\label{thm:linear}
We assume the following assumptions regarding the norm: (1) $\sup_{(s,a)}\|\phi(s,a)\|\leq 1$, (2) $\|\int \mu^{\star}(s)v(s)\mathrm{d}\iota(s)\|_2 \leq \sqrt{d}$ for any $v:\Scal \to \RR$ such that $\|v\|_{\infty}\leq 1$ and (3) $\|\theta\|_2 \leq W$. We set $\zeta = (1-\gamma)^{-1}\sqrt{d^2 \ln(n|\Pi|W/\delta)/n }$. With probability at least $1-\delta$, for all comparator policy $\pi^{*} \in \Pi$, we have 
\begin{align*}
 V^{\pi^{*}}_{P^{\star}} -  V^{\hat \pi}_{P^{\star}}  \leq c_1 (1-\gamma)^{-2}  \sqrt{\frac{\bar C_{\pi^{*},\phi}\mathrm{rank}[\Sigma_{\rho}]^2d \ln(c_2n|\Pi|W/\delta)\}}{n}}.
\end{align*}
\end{theorem}

Compared to PAC bounds in other models in our article, \pref{thm:linear} incurs $\ln(|\Pi|)$. Thus, it requires that $\Pi$ is restricted. It is known that this dependence can be removed by using pessimistic model-free algorithms with a natural policy gradient \citep{XieTengyang2021BPfO,zanette2021provable}. Hence, our bound might be worse than their results in nonparametric linear MDPs. However, as we mention in \pref{sec:conversion}, their algorithm incurs $\ln(|\Pi|)$ in many other models such as finite models (finite $|\Mcal|$), KNRs, and linear mixture MDPs while CPPO does not incur $\ln(|\Pi|)$. This suggests that nonparametric linear MDPs are more amenable to model-free RL while KNRs and linear mixture MDPs are more amenable to model-based RL. 

}

{\newedit \section{Bayesian Offline RL: Policy Optimization via Posterior Sampling}
\label{sec:alg}
The mini-max constrained optimization step in \pref{alg:main_version} is not computationally efficient as it is equivalent to a version space based algorithm shown in Eq.~\ref{eq:version}.  In this section, we consider offline RL in the Bayesian setting, and study posterior sampling based offline RL algorithms. The goal here is to design offline RL algorithms that rely on posterior sampling rather than explicit pessimism. While as we will show, the benefit of leveraging posterior sampling is that we do not need to design pessimism or reward penalty, the downside is that we sacrifice from worst-case suboptimality gap to the Bayesian suboptimality gap. 

\subsection{Algorithm}

We consider posterior sampling together with incremental policy optimization procedure. \pref{alg:main_alg_bonus} summarizes the posterior sampling based policy optimization algorithm $\emph{PS-PO}$. The algorithm relies on two computational oracles, a posterior distribution update oracle, and a posterior sampling oracle. %
The algorithm consists of two procedures. The first procedure calls the posterior update oracle, i.e., given the prior distribution $\beta$, and given the offline dataset $\Dcal$, the posterior update gives the posterior distribution over models conditioned on the dataset $\Dcal$, i.e., we get $\beta(\cdot | \Dcal )$. Hereafter, We always assume that $\beta(\cdot | \Dcal )$ exists, i.e., the prior distribution $\beta$ is proper.

\begin{algorithm}[!t]\label{alg:pspo}
\caption{PS-PO: \textbf{P}olicy \textbf{O}ptimization with \textbf{P}osterior \textbf{S}ampling for Offline RL}\label{alg:main_alg_bonus}
\begin{algorithmic}[1]
    \STATE {\bf Require}: dataset $\Dcal$, prior distribution $\beta \in \Delta(\Mcal)$, learning rate $\eta$.
    \STATE \textcolor{blue}{Bayesian update:} Compute model posterior $\beta(\cdot | \Dcal) \in \Delta(\Mcal)$
    \STATE Initialize policy $\pi_0$ where $\pi_0(\cdot | s) = \text{Uniform}(\Acal)$
    	\FOR{$t = 0, \cdots ,{T-1}$}
		\STATE \textcolor{blue}{Posterior sampling:} $P_t \sim \beta(\cdot | \Dcal)$
		\STATE \textcolor{blue}{Policy update:} $\pi_{t+1}(a | s) \propto \pi_t(a|s) \exp(\eta A^{\pi_t}_{P_t}(s,a))$
     	\ENDFOR
    \STATE  {\bf Return} $\pi_{T}$.
\end{algorithmic}
\end{algorithm}

Once we have the posterior distribution $\beta(\cdot | \Dcal)$, the second procedure of our algorithm is to perform policy optimization with $\beta(\cdot|\Dcal)$.  More specifically, at iteration $t$ with the latest learned policy $\pi_t$, we \emph{sample} a model from $\beta(\cdot | \Dcal)$, i.e., $P_t \sim \beta(\cdot | \Dcal)$. We then update policy from $\pi_t$ to $\pi_{t+1}$ using incremental policy update, i.e., $\pi_{t+1}(a|s) \propto \pi_t(a|s)\exp\left( \eta A^{\pi_t}_{P_t}(s,a) \right), \forall s,a$, with $\eta \in \mathbb{R}^+$ being some learning rate. We emphasize that every iteration $t$, our algorithm samples a fresh model $P_t$ from $\beta(\cdot | \Dcal)$. Note that this new algorithm does not explicitly use any pessimism or reward penalty inside the algorithm. 

What is the intuition behind this algorithm, and what is the benefit of this algorithm compared to a na\"ive model-based policy optimization approach (i.e., simply training a model from $\Dcal$ and using that model over and over again during the entire policy optimization procedure such as the offline version of natural policy gradient \citep{pmlr-v125-agarwal20a})? \emph{The random sampling procedure prevents policy optimization from exploiting the error in a single model trained on $\Dcal$}. A sample $P_t$ is an accurate model under the space that is well covered by the offline data $\Dcal$, but can be inaccurate at the space that is not covered by the offline data $\Dcal$. Similarly, $P_{t+1}$ is accurate under the covered space as well. However, $P_t$ and $P_{t+1}$ could disagree with each other on the space that is not covered by the offline data. Thus, the random sampling procedure makes PG algorithm hard to consistently exploit model errors inside a single model. Yet PG algorithm can make progress inside the region that is well covered by the offline data since models sampled from the posterior distribution are accurate and all agree with each other in the covered region. 

\subsection{Analysis}
To analyze the Bayesian regret of PO-PS, we first introduce the concentrability coefficient and the relative condition number in the Bayesian setting.  Recall that given a model $P$, we denote $\pi(P) = \argmax_{\pi} V^{\pi}_{P}$ as the (global) optimal policy under model $P$. We define the following quantities related to partial coverage: 
\begin{align}\label{eq:Bayesian}
 C^{\dagger, \text{Bayes}}_{\beta}=\EE_{P^{\star}\sim \beta}[C^{\dagger}_{\pi(P^{\star}),P^{\star}} ],\quad  C^{ \text{Bayes}}_{\beta}=\EE_{P^{\star}\sim \beta}[C_{\pi(P^{\star}),P^{\star}} ],\quad \bar C^{ \text{Bayes}}_{\beta}=\EE_{P^{\star}\sim \beta}[\bar C_{\pi(P^{\star}),P^{\star}} ]
\end{align}
where
\begin{align*}
  C^{\dagger}_{\pi(P^{\star}),P^{\star}}&=\sup_{P'\in \Mcal}\frac{\EE_{(s,a)\sim d^{\pi(P^{\star})}_{P^{\star}}}[ \TV ({P}'(\cdot | s,a), P^{\star}(\cdot | s,a))^2] }{\EE_{(s,a)\sim \rho}[\TV ({P}'(\cdot | s,a), P^{\star}(\cdot | s,a))^2]} \\ 
   C_{\pi(P^{\star}),P^{\star}} &= \sup_{(s,a)}\frac{d^{\pi(P^{\star})}_{P^{\star}}(s,a)}{\rho(s,a)},\\
   \bar C_{\pi(P^{\star}),P^{\star}} &= \sup_{x\in\mathbb{R}^d} \frac{x^T \Sigma_{\pi^*} x}{ x^{\top} \Sigma_\rho, x}, \quad \Sigma_{\pi(P^{\star})}=\EE_{(s,a)\sim d^{\pi(P^{\star})}_{P^{\star}} }[\phi(s,a)\phi(s,a)^{\top}], \quad \Sigma_{\rho}=\EE_{(s,a)\sim \rho}[\phi(s,a)\phi(s,a)^{\top}]. 
\end{align*} 
Comparing to the frequentist quantities, the density ratio and relative condition number quantities are also averaged over the prior distribution. 
The partial coverage means these types of quantities are upper-bounded by some constants.

\subsubsection{The Implicit Pessimism in Posterior Sampling} 
Before diving into the analysis of PS-PO, we consider a simpler algorithm as a warm-up.  
This algorithm takes the model $P$ sampled from the posterior $\beta(\cdot\mid \Dcal)$ and outputs the optimal policy for this model $P$ (i.e., by using a planning oracle). Namely, the algorithm has the following two steps:
\begin{align*}
P\sim \beta(\cdot | \Dcal), \quad \pi(P) = \argmax_{\pi} V^{\pi}_{P}.
\end{align*}
To analyze the above two-step algorithm in the Bayesian setting, we first introduce some additional notations. We first define a function over the policy class depending on $\Dcal$, i.e.,  $L(\pi;\Dcal):\Pi \to [0,(1-\gamma)^{-1}]$. This function $L(\cdot; \Dcal)$ is fully determined by the dataset $\Dcal$.
Then, inspired by \citet{RussoDaniel2014LtOv}, for the model $P$ sampled from $\beta(\cdot | \Dcal)$, we have the following decomposition for Bayesian suboptimaligy gap:  %
\begin{align*} 
\mathbb{E}\left[  V^{\pi(P^{\star})}_{P^{\star}}   -  V^{\pi(P)}_{P^{\star}} \right] & =\mathbb{E}\left[  V^{\pi(P^{\star})}_{P^{\star}}  -L(\pi(P^{\star});\Dcal)+ L(\pi(P^{\star});\Dcal)- V^{\pi(P)}_{P^{\star}} \right]\\ 
&=\mathbb{E}\left[  V^{\pi(P^{\star})}_{P^{\star}}  -L(\pi(P^{\star});\Dcal)+ \EE[L(\pi(P^{\star});\Dcal)\mid \Dcal]- V^{\pi(P)}_{P^{\star}} \right]\\
&=\mathbb{E}\left[  V^{\pi(P^{\star})}_{P^{\star}}  -L(\pi(P^{\star});\Dcal)+ L(\pi(P);\Dcal)- V^{\pi(P)}_{P^{\star}} \right]. 
\end{align*}
We use $\mathbb{E}\left[ L(\pi(P^{\star});\Dcal)  | \Dcal \right] = \mathbb{E}\left[  L(\pi(P);\Dcal)   | \Dcal\right]$ as $P$ and $P^{\star}$ are independently and identically distributed from $\beta(\cdot | \Dcal)$. Then, given $P^{\star}$ and $\Dcal$ generated based on $P^{\star}$ (i.e., $P^{\star}\sim \beta, (s,a)\sim \rho,s'\sim P^{\star}(\cdot | s,a)$), if $L(\pi;\Dcal)$ gives a lower confidence bound of $V^{\pi}_{P^{\star}}$, such that $\forall \pi \in \Pi: V^{\pi}_{P^{\star}}\geq L(\pi;\Dcal)$, we have 
\begin{align}\label{eq:pessimism}
    \mathbb{E}\left[  V^{\pi(P^{\star})}_{P^{\star}}   -  V^{\pi(P)}_{P^{\star}} \right]\leq \mathbb{E}\left[  V^{\pi(P^{\star})}_{P^{\star}}  -L(\pi(P^{\star});\Dcal)\right]. 
\end{align}
This is summarized in the following theorem with the formalized definition of $L(\pi;\Dcal)$. 

\begin{assum}\label{assum:lcb}
Given a model $P^{\star}$ on the support $\{P^{\star}:\beta(P^{\star})>0\}$, let $\Dcal$ be the dataset generated following $P^{\star}$. We have a function $L(\pi;\Dcal):\Pi\to [0,(1-\gamma)^{-1}]$ s.t. $\mathrm{P}(L(\pi;\Dcal)\leq V^{\pi}_{P^{\star}},\forall \pi\in \Pi\mid P^{\star})\geq 1-\delta$.  We denote $\Lcal_{\Dcal}$ as the set the contains all such functions $L(\cdot;\Dcal)$.
\end{assum}
In the above assumption, the randomness in the high probability statement is with respect to the dataset $\Dcal$ conditioned on $P^{\star}$.

\begin{theorem}\label{thm:naive}
Suppose Assumption \pref{assum:lcb} holds. 
\begin{align*}
        \mathbb{E}\left[  V^{\pi(P^{\star})}_{P^{\star}}  -  V^{\pi(P)}_{P^{\star}} \right]\leq \mathbb{E}\left[  V^{\pi(P^{\star})}_{P^{\star}}  -L(\pi(P^{\star});\Dcal)\right]+2(1-\gamma)^{-1}\delta. 
\end{align*}
\end{theorem}
This result satisfies our desiderata, i.e., we can obtain the bound for Bayesian sumoptimality gap under the partial coverage as we only need to be concern about the distribution $d^{\pi(P^{\star})}_{P^{\star}}$ with $P^\star$ being sampled from the prior $\beta$, which allows us to use the quantities define in Eq.~\pref{eq:Bayesian}.

To obtain Bayesian suboptimality gap bounds from \pref{thm:naive} under the partial coverage, we need to design $L(\pi;\Dcal)$ on a case-by-case basis. The first choice is $\min_{M}V^{\pi}_{M\in \Mcal_{\Dcal}}$ in \pref{alg:main_version}, which satisfies the condition $\min_{M}V^{\pi}_{M\in \Mcal_{\Dcal}} \leq V^{\pi}_{P^{\star}},\forall \pi\in \Pi$.  %
Then, we can plug in the frequentist suboptimality gap result \pref{thm:version} into \pref{thm:naive}, which leads to the Bayesian suboptimality gap result under partial coverage. The second choice is $\min_{M}V^{\pi}_{M\in \bar \Mcal_{\Dcal}}$ in \pref{alg:main_version2}. Then, we can plug in the frequentist suboptimality gap result \pref{thm:version2} into \pref{thm:naive}, which again leads to the Bayesian suboptimality gap result under partial coverage. Another choice is a reward penalty \citep{ChangJonathanD2021MCSi}.  Given the dataset $\Dcal$, we compute a model estimator $\widehat{P}(\cdot | s,a)$ and a model uncertainty measure $\sigma(s,a)$ s.t.  $\forall s,a: \; \TV( \widehat{P}(\cdot | s,a) ,P(\cdot|s,a) ) \leq \sigma(s,a)$, then we can design a reward penalty $b(s,a) = H\sigma(s,a)$ so that $V^{\pi}_{\widehat{P}, r - b}$ satisfies the condition  $V^{\pi}_{\widehat{P}, r - b}\leq V^{\pi}_{P^{\star}},\forall \pi\in \Pi$, where $V^{\pi}_{\widehat{P}, r - b}$ is a policy value under a transition $\widehat{P}$, a reward $r-b$ and a policy $\pi$. Then, by translating the frequentist result of \citet{ChangJonathanD2021MCSi} into the Bayesian setting, we can obtain the Bayesian suboptimality bound under the partial coverage. We will see more specific bounds in \pref{sec:bounds}. 
\subsubsection{Analysis of PS-PO}
Now we are ready to analyze PS-PO where we combine the analysis of NPG with the above Bayesian analysis. As in the previous section, we introduce $L(\pi;\Dcal):\Pi \to[0,(1-\gamma)^{-1}]$ and $L(\cdot; \Dcal)$ is a mapping that is fully determined by the dataset $\Dcal$.

We start by bounding the per-iteration regret:
\begin{lemma}
[Per-iteration regret]\label{lem:regret}
Suppose Assumption \pref{assum:lcb} holds. For any iteration $t$, we have 
\begin{align*}
\mathbb{E}\left[ V^{\pi(P^{\star})}_{P^{\star}} - V^{\pi_t}_{P^{\star}} \right] \leq \mathbb{E}\left[ V^{\pi(P^{\star})}_{P^{\star}} - L( \pi(P^{\star});\Dcal) \right] + \mathbb{E}\left[ V^{\pi(P^{\star})}_{P^{\star}} - V^{\pi_t}_{P^{\star}}  \right]+2(1-\gamma)^{-1}\delta. 
\end{align*}
\end{lemma}
The proof is similarly done as the proof of \pref{thm:naive}. We use a key relation $\EE[L(\pi(P^{\star});\Dcal ) \mid \Dcal]=\EE[L(\pi(P_t);\Dcal ) \mid \Dcal]$. The first term is upper-bounded under the partial coverage following the argument after \pref{thm:naive}. The third term is negligible by taking sufficiently small $\delta$. Thus, we analyze the second term of r.h.s in detail. The second term corresponds to the regret term for the model-based policy optimization procedure. Recall that we update policy as $\pi_{t+1}(a|s) \propto \pi_t(a|s) \exp(\eta A^{\pi_t}_{P_t}(s,a))$.

\begin{lemma}\label{lem:npg}
Consider a fixed iteration $t$. Suppose $\eta<0.5(1-\gamma)$. We have:
\begin{align*}
\mathbb{E}\left[ V^{\pi(P^{\star})}_{P^{\star}} - V^{\pi_t}_{P^{\star}}  \right] \leq \mathbb{E}\left[ 4\eta (1-\gamma)^{-3} + \frac{(1-\gamma)^{-1}}{\eta} \mathbb{E}_{s\sim d^{\pi(P^{\star})}_{P^{\star}}} \left[\KL\left(\pi(P^{\star})(\cdot | s), \pi_{t+1}(\cdot | s)\right) -\KL\left( \pi(P^{\star})(\cdot|s) , \pi_{t}(\cdot | s)\right)  \right]  \right]
\end{align*}
\end{lemma}

By combining the above two lemmas and considering all iterations, we conclude the following general theorem. 
\begin{theorem}
\label{thm:bayesian_pspo}
Suppose Assumption \pref{assum:lcb}. When $\eta<0.5(1-\gamma)$, then, 
\begin{align*}
\mathbb{E}\left[  V^{\pi(P^{\star})}_{P^{\star}} - \max_{t\in[T]} V^{\pi_t}_{P^{\star}}     \right] \leq \underbrace{\min_{L\in \Lcal_{\Dcal}}  \mathbb{E} \left[  V^{\pi(P^{\star})}_{P^{\star}} - L( \pi(P^{\star});\Dcal)\right]}_{\text{(S1)}}+\underbrace{ 4(1-\gamma)^{-2} \sqrt{ \frac{ \ln(|\Acal|)}{T}}}_{\text{(S2)}}+2(1-\gamma)^{-1}\delta. 
\end{align*}
\end{theorem}

By taking sufficiently large $T$, the second term (S2) is negligible. Thus, the first term (S1) dominates the error which we will analyze in detail under the partial coverage. Note that the Bayesian result in \pref{thm:bayesian_pspo} allows us to pick the tightest lower confidence bound among all possible valid LCBs that satisfy Assumption \ref{assum:lcb}.

\subsubsection{Detailed Bounds on the Bayesian suboptimality Gap} \label{sec:bounds}
In this section, we specialize \pref{thm:bayesian_pspo} to concrete examples. 
Here, we use Bayesian concentrability coefficients defined in Eq.~\pref{eq:Bayesian}. We start with the general realizable mode class $\Mcal$. Here, we set $L(\pi;\Dcal):= \min_{P\in\Mcal_{\Dcal}} V^{\pi}_{P}$. Note that we have proved that given $P^\star$ and $\Dcal$ being generated based on $\Dcal$,
$\min_{P\in\Mcal_{\Dcal}} V^{\pi}_{P} \leq V^{\pi}_{P^\star}, \forall \pi$, with high probability. By plugging $\min_{P\in\Mcal_{\Dcal}} V^{\pi}_{P}$ into \pref{thm:bayesian_pspo}, we arrive at the following corollary. 

\begin{corollary}[PS-PO with General Function Class] \label{cor:discrete_bayes}
Suppose the partial coverage $C^{\dagger,\text{Bayes}}_{\beta}<\infty$. 
 \begin{align*}
     \mathbb{E}\left[  V^{\pi(P^{\star})}_{P^{\star}} - \max_{t\in[T]} V^{\pi_t}_{P^{\star}}     \right]\leq c_1(1-\gamma)^{-2}  \sqrt{\frac{C^{\dagger,\text{Bayes}}_{\beta}\ln(|\Mcal|n)}{n} }+ (1-\gamma)^{-2}\sqrt{ \frac{ \ln(|\Acal|)}{T}}. 
 \end{align*}
\end{corollary}

\begin{corollary}[PS-PO for Tabular MDPs] \label{cor:tablar_bayes}
Suppose the partial coverage $C^{\text{Bayes}}_{\beta}<\infty$. 
 \begin{align*}
    \mathbb{E}\left[  V^{\pi(P^{\star})}_{P^{\star}} - \max_{t\in[T]} V^{\pi_t}_{P^{\star}}     \right] \leq  c_1 (1-\gamma)^{-2}\sqrt{\frac{ C^{\Bayes}_{\beta}|\Scal|^2|\Acal|\ln(n|\Scal||\Acal|c_2)}{n}}
    +  (1-\gamma)^{-2}\sqrt{ \frac{ \ln(|\Acal|)}{T}}. 
 \end{align*}
\end{corollary}
\begin{corollary}[PS-PO for Linear Mixture MDPs]\label{cor:linear_mixture_bayes}
Suppose $\|\theta^{\star}\| \leq R$, $\sup_{(s,a)}\|\psi_V(s,a)\|_2\leq 1,\forall V\in \{\Scal \to [0,1]\}$.  Then, we have 
\begin{align*}
    \mathbb{E}\left[  V^{\pi(P^{\star})}_{P^{\star}} - \max_{t\in[T]} V^{\pi_t}_{P^{\star}}     \right] \leq   c_1 (1-\gamma)^{-2}  \sqrt{\frac{ dC^{\dagger,\Bayes}_{\beta} \ln(c_2 n\iota(\Scal)R)}{n} }    +  (1-\gamma)^{-2}\sqrt{ \frac{ \ln(|\Acal|)}{T}}.  
\end{align*}
\end{corollary} 

For KNRs with the known feature $\phi$, we can use the Bayesian relative condition number.  Here again we set $L(\pi;\Dcal) = \min_{W\in\Wcal_{\Dcal}} V^{\pi}_{P(W)}$.
\begin{corollary}[PS-PO for KNRs] \label{cor:linear_bayes}
 Assume $\|\phi(s,a)\|_2\leq 1,\forall(s,a)\in \Scal\times \Acal$. Suppose the partial coverage $\bar C^{\text{Bayes}}_{\beta}<\infty$. By letting $\|W^*\|^2_2=O(1),\upsilon^2=O(1)$,  we have 
\begin{align*}
     \mathbb{E}\left[  V^{\pi(P^{\star})}_{P^{\star}} - \max_{t\in[T]} V^{\pi_t}_{P^{\star}}     \right] \leq c_1 (1-\gamma)^{-2} \min(d^{1/2},  \bar R)\sqrt{ \bar R }   \sqrt{\frac{d_{\Scal}\bar C^{\Bayes}_{\beta}\ln (1+n) }{n}}+  (1-\gamma)^{-2} \sqrt{ \frac{ \ln(|\Acal|)}{T}}, 
\end{align*}
where $\bar R=  \mathrm{rank}[\Sigma_{\rho}]\{\mathrm{rank}[\Sigma_{\rho}] +\ln(c_2n)\}$. 
\end{corollary} 
Similarly, we can also extend the above result to KNRs with infinite-dimensional $\phi$ based on the result of Corollary \pref{cor:gps} by using the effective dimension $d^*$. 

Finally, for low-rank MDPs, we use $L(\pi;\Dcal) = \min_{P\in\Mcal_{\Dcal}} V^{\pi}_{P}$.
\begin{corollary}[PS-PO for low-rank MDPs] \label{cor:low_rank_bayes}
Suppose (a): $\|\phi(s,a)\|_2\leq 1, \forall (s,a)\in \Scal\times\Acal,\forall \phi \in \Phi$, $\int \mu(s')^{\top}\phi(s,a)\mathrm{d}\iota(s')=1$ and $\int \|\mu(s)\|_2 \mathrm{d}\iota(s)\leq \sqrt{d},\forall \mu\in \Psi,\phi\in \Phi$, (b) $\rho(s,a)=d^{\pi_b}_{P^{\star}}(s,a)$.  We have 
\begin{align}\textstyle     
\mathbb{E}\left[  V^{\pi(P^{\star})}_{P^{\star}} - \max_{t\in[T]} V^{\pi_t}_{P^{\star}}     \right] &\leq    c_3 \sqrt{\bar C^{\Bayes}_{\beta}\omega_{\pi^{*}}\rank(\Sigma_{\rho}) \frac{\ln(|\Psi||\Phi|c_4/\delta)}{(1-\gamma)^4 n}}+(1-\gamma)^{-2} \sqrt{ \frac{ \ln(|\Acal|)}{T}}, \,\omega_{\pi^{*}}=\prns{\max_{(s,a)}\frac{\pi^{*}(a\mid s)}{\pi_b(a\mid s)}}. 
\end{align}

\end{corollary}
}

\section{Conclusion}\label{sec:conclusion}

We study model-based offline RL with function approximation under partial coverage. We show that for the model-based setting, realizability in function class and partial coverage together are enough to learn a policy that is comparable to \emph{any} policies (including history-dependent policies) covered by the offline distribution. Our result demonstrates a sharp contrast to model-free offline RL approaches which often require additional structural conditions in the function class (e.g., Bellman completion) and have restrictions on the pool of candidate policies that they can compete against.

{Some readers might wonder whether CPPO-TV and CPPO-LR is computationally efficient. The minimax optimization problem $\argmax_{\pi \in \Pi}\min_{P\in M}V^{\pi}_P$ fits into a framework of planning on robust MDPs \citep{nilim2005robust,iyengar2005robust}. By introducing a robust Bellman equation, they proposed value iteration and policy iteration algorithms and showed that algorithms are practically tractable in the tabular setting. In the non-tabular setting, \citet{lim2019kernel,tamar2014scaling} propose the extension using function approximation. Thus, we can apply their methods to approximately solve the minimax optimization problem in a model-free fashion. We leave the formal theoretical justification when using these approximation planning algorithms as an important direction for future work. {\ouredit As a first step, we propose a natural policy gradient based policy optimization method based on posterior sampling in \pref{sec:alg}. In some models such as low-rank MDPs, follow-up works propose computationally efficient algorithms \citep{uehara2021representation,zhang2022making,qiu2022contrastive} . }

\section*{Acknowledgement}
~
The authors would like to thank Nan Jiang, Tengyang Xie, Audrey Huang, Jinglin Chen, Runzhe Wu for their valuable feedback.

Masatoshi Uehara was partially supported by Masason foundation. 

\bibliography{reference,reference2}

\newpage 
\appendix

\section{Comparison to \texorpdfstring{\citet{XieTengyang2021BPfO}}{a}}\label{ape:comparison}

We compare a result in \citep{XieTengyang2021BPfO} to our result in detail. Let $\Fcal$ be a function class for $Q$-functions. Here, we consider a more general version of their algorithm by replacing the original $\mathcal{E}(f,\pi;\Dcal)$ in their algorithm with 
\begin{align*}
  \mathcal{E}(f,\pi;\Dcal):= \mathcal{L}(f,f;\pi,\Dcal)-\min_{g\in \Gcal}\mathcal{L}(g,f;\pi,\Dcal). 
\end{align*}
In their original algorithm, they set $\Gcal=\Fcal$. Here, we consider the version such that a discriminator class $\Gcal$ can be different from $\Fcal$. 

They show the PAC result under partial coverage as follows. Here, $\Tcal^{\pi}_{P^{\star}}$ is a Bellman operator under $\pi$ and $P^{\star}$:
\begin{align*}
    \Tcal^{\pi}_{P^{\star}}: \{\Scal \times \Acal \to \RR\}\ni f\mapsto r(s,a)+\E_{P^{*}(s'\mid s,a)}[f(s',\pi)]\in\{\Scal \times \Acal \to \RR\}. 
\end{align*}
\begin{theorem}[Extension of Result in \citep{XieTengyang2021BPfO} ]\label{thm:model_free}
Suppose realizaibility $Q^{\pi}_{P^{\star}}\in \Fcal,\forall \pi \in \Pi$ and closeness $\max_{f\in\Fcal} \min_{g\in\Gcal} \mathbb{E}_{s,a\sim \rho} [ (g- \Tcal^{\pi}_{P^{\star}} f)^2(s,a) ] = 0,\forall \pi \in \Pi $. Then, with $1-\delta$, for any $\pi^{*}\in \Pi$, we have 
\begin{align*}
        V^{\pi^{*}}_{P^{\star}}-V^{\hat \pi}_{P^{\star}}=  O(\sqrt{C^{\diamond}\ln(|\Pi||\Fcal||\Gcal|/\delta)/n}),\quad C^{\diamond}=\sup_{f\in \Fcal}\frac{\EE_{(s,a)\sim d^{\pi^{*}}_{P^{\star}}}[(f-\Tcal f)^2(s,a)] }{\EE_{(s,a)\sim \rho}[(f-\Tcal f)^2(s,a)]}. 
\end{align*}
\end{theorem}
By combining this result with the conversion from model-free results to model-based results in \citep[Corollary 6]{ChenJinglin2019ICiB}, we can obtain the following result under partial coverage. 
\begin{theorem}\label{thm:model_free2}(\textbf{PAC guarantee from the direct application of \citep{XieTengyang2021BPfO} to mode-based RL  })
Assume $P^{\star}\in \Mcal$. Then, there exists an algorithm s.t. with $1-\delta$, for any policy $\pi^{\star}\in \Pi$,
\begin{align*}
        V^{\pi^{*}}_{P^{\star}}-V^{\hat \pi}_{P^{\star}}=  O(\sqrt{C^{\diamond}\ln(|\Pi||\Mcal|/\delta)/n}). 
\end{align*}
\end{theorem}
\begin{proof}[Proof of \pref{thm:model_free2}]
Given a model class $\Mcal$, consider the following reduction. We define a $Q$-function class: 
\begin{align*}
    \Fcal=\{q^{\pi}_{P} \mid  \pi \in \Pi, P \in \Mcal \}. 
\end{align*}
Then, we define a discriminator class $\Gcal$: 
\begin{align*}
    \Gcal=\{\Tcal^{\pi'}_{P'}q^{\pi}_{P} \mid  \pi \in \Pi, \pi'\in \Pi, P \in \Mcal ,P'\in \Mcal\}. 
\end{align*}

The above satisfies the realizability $Q^{\pi}_{P^{\star}}\in \Fcal,\forall \pi \in \Pi$ and the closedness $\Tcal^{\pi}_{P^{\star}}\Fcal \subset \Gcal,\forall \pi\in \Pi$. Thus, the assumptions in \pref{thm:model_free} are satisfied. Then, we have 
\begin{align*}
 V^{\pi^{*}}_{P^{\star}}-V^{\hat \pi}_{P^{\star}}&=O(\sqrt{C^{\diamond}\ln(|\Pi||\Fcal||\Gcal|/\delta)/n}) \\
      &= O(\sqrt{C^{\diamond}\ln(|\Pi||\Mcal|/\delta)/n}), 
\end{align*}
noting $|\Fcal|=|\Pi||\Mcal|$ and $|\Gcal|=|\Pi|^2|\Mcal|^2$. 

\end{proof} 

As we mentioned, this is  worse than our result since it includes $|\Pi|$. Besides, the algorithm can only compete against policies restricted in $\Pi$, while our algorithm works for the unrestricted policy class $\Pi$ which could even include history dependent policies. For completeness, we give the proof as follows. 

We remark that their results (Theorem 4.1) with NPG that can possibly compete with any stochastic policies, are not applicable here. This is because they need an assumption that the comparator policy $\pi^{*}$ needs to satisfy  $Q^{\pi^{*}}_{P^{\star}}\in \Fcal$ and $\max_{f\in\Fcal} \min_{g\in\Gcal} \mathbb{E}_{s,a\sim \rho} [ (g- \Tcal^{\pi^{*}}_{P^{\star}} f)^2(s,a) ] = 0$, which does not hold for the corresponding Q-function class $\Fcal$ after the conversion. %
As a notable exception, when the model is a linear Bellman-complete MDP \citep{zanette2021provable}, any stochastic policies satisfy the Bellman completeness for the linear Q-function class; then, their algorithms can learn policies that can compete with any stochastic policies satisfying partial coverage.

\section{Missing Proofs in \pref{sec:version}}

Below we use $c,c_1,c_2,\cdots$ to denote universal constants. For a $d$-dimensional vector $a$ and a matrix $A\in \RR^{d\times d}$, we denote $\|a\|^2_{A}=a^{\top}Aa$. Here, $a \lesssim B$ means $a\leq c B$ for some universal constant. $c$ 

\subsection{Proofs for General Function Approximation for CPPO-TV (Proof of \pref{thm:version})}\label{ape:gene}

From \pref{lem:mle}, the MLE guarantee gives us the following generalization bound: with probability $1-\delta$, 
\begin{align}\label{eq:mle_version}
\mathbb{E}_{s,a\sim \rho} [\TV(\widehat{P}_{\MLE}(\cdot \mid s,a),P^\star(\cdot \mid s,a))^2] \lesssim \frac{ \ln(|\Mcal| / \delta) }{n}. 
\end{align}
Letting
\begin{align*}
     A(P) \coloneqq |\EE_{s,a\sim \rho} [\TV(P(\cdot \mid s,a),P^\star(\cdot \mid s,a))^2] -\EE_{\Dcal} [\TV(P(\cdot \mid s,a),P^\star(\cdot \mid s,a))^2] |. 
\end{align*}
with probability $1-\delta$, from union bound and Bernstein's inequality, we also have
\begin{align}\label{eq:bernstein_easy}
    A(P)\leq \sqrt{\frac{c_1\mathrm{var}_{(s,a)\sim \rho}[\TV(P(\cdot \mid s,a),P^\star(\cdot \mid s,a))^2] \ln(|\Mcal|/\delta)}{n}}+\frac{c_2\ln(|\Mcal|/\delta)}{n},\forall P\in \Mcal. 
\end{align}
Hereafter, we condition on the above two events. Recall that we construct the version space using $\Dcal$ and $\widehat{P}_{\MLE}$ as follows:
\begin{align*}
\Mcal_{\Dcal} := \left\{ P \in \Mcal: \mathbb{E}_{\Dcal} [\TV(P(\cdot \mid s,a),\hat P_{\MLE}(\cdot \mid s,a))^2] \leq \xi \right\}.
\end{align*}

\paragraph{First Step: Show $P^{\star}\in \Mcal_{\Dcal}$ in high-probability. }

We set $\xi= c \frac{ \ln(|\Mcal| / \delta) }{n}$. Conditioning on the above two events equations \pref{eq:mle_version} and \pref{eq:bernstein_easy}, we prove $P^{\star}\in \Mcal_{\Dcal}$. This is proved by 
\begin{align*}
    & \mathbb{E}_{\Dcal} [\TV(\widehat{P}_{\MLE}(\cdot \mid s,a),P^\star(\cdot \mid s,a))^2]\\
    &=     \mathbb{E}_{\Dcal} [\TV(\widehat{P}_{\MLE}(\cdot \mid s,a),P^\star(\cdot \mid s,a))^2]-\mathbb{E}_{(s,a)\sim \rho} [\TV(\widehat{P}_{\MLE}(\cdot \mid s,a),P^\star(\cdot \mid s,a))^2] \\
    & +\mathbb{E}_{(s,a)\sim \rho} [\TV(\widehat{P}_{\MLE}(\cdot \mid s,a),P^\star(\cdot \mid s,a))^2] \\
    &=  \mathbb{E}_{\Dcal} [\TV(\widehat{P}_{\MLE}(\cdot \mid s,a),P^\star(\cdot \mid s,a))^2]-\mathbb{E}_{(s,a)\sim \rho} [\TV(\widehat{P}_{\MLE}(\cdot \mid s,a),P^\star(\cdot \mid s,a))^2]+  \frac{c_1 \ln(|\Mcal| / \delta) }{n}\\ 
    &\lesssim  \sqrt{\frac{\mathrm{var}_{(s,a)\sim \rho}[\TV(\widehat{P}_{\MLE}(\cdot \mid s,a),P^\star(\cdot \mid s,a))^2]\ln(|\Mcal|/\delta)}{n}} +  \frac{\ln(|\Mcal| / \delta) }{n} \tag{From \pref{eq:bernstein_easy}}\\ 
     &\lesssim  \sqrt{\frac{\mathrm{E}_{(s,a)\sim \rho}[\TV(\widehat{P}_{\MLE}(\cdot \mid s,a),P^\star(\cdot \mid s,a))^2]\ln(|\Mcal|/\delta)}{n}} +  \frac{\ln(|\Mcal| / \delta) }{n} \tag{$\TV(\widehat{P}_{\MLE}(\cdot \mid s,a),P^\star(\cdot \mid s,a))^2\leq 4$}\\
     &\lesssim \frac{1}{n}\ln(|\Mcal| / \delta) \tag{Plug in MLE guarantee}. 
\end{align*}

\paragraph{Second Step: Show $ \EE_{s,a\sim \rho} [\TV(P(\cdot \mid s,a),P^\star(\cdot \mid s,a))^2]\leq  c\xi,\quad \forall P\in \Mcal_{\Dcal}$ in high probability.  }
We show for any $P \in \Mcal_{\Dcal}$, the distance between $P^{\star}$ is sufficiently controlled in terms of TV distance. More concretely (conditioning on the above two events  \pref{eq:mle_version} and \pref{eq:bernstein_easy} ), we show 
\begin{align*}
    \EE_{s,a\sim \rho} [\TV(P(\cdot \mid s,a),P^\star(\cdot \mid s,a))^2]\lesssim \xi,\quad \forall P\in \Mcal_{\Dcal}. 
\end{align*}

In order to observe this, for any $P \in \Mcal_{\Dcal}$, we have 
\begin{align*}  
& \EE_{\Dcal} [\TV(P(\cdot \mid s,a),P^\star(\cdot \mid s,a))^2]  \nonumber \\ 
&\leq 2\EE_{\Dcal} [\TV(\widehat{P}_{\MLE}(\cdot \mid s,a),P(\cdot \mid s,a))^2] + 2\EE_{\Dcal} [\TV(\widehat{P}_{\MLE}(\cdot \mid s,a),P^\star(\cdot \mid s,a))^2] \leq 4 \xi  \tag{From $(a+b)^2\leq 2a^2+2b^2$.}
\end{align*} 
Thus,  we have:
\begin{align}
&\EE_{s,a\sim \rho} [\TV(P(\cdot \mid s,a),P^\star(\cdot \mid s,a))^2] \nonumber \\
&= \EE_{s,a\sim \rho} [\TV(P(\cdot \mid s,a),P^\star(\cdot \mid s,a))^2] -\EE_{\Dcal} [\TV(P(\cdot \mid s,a),P^\star(\cdot \mid s,a))^2] +\EE_{\Dcal}[\TV(P(\cdot \mid s,a),P^\star(\cdot \mid s,a))^2]  \nonumber  \\ 
&\leq  A(P)+ c\xi .   \label{eq:key_version}
\end{align}
Here, from \pref{eq:bernstein_easy}, we have
\begin{align*}
    A(P)\leq \sqrt{\frac{c_1\mathrm{var}_{(s,a)\sim \rho}[\TV(P(\cdot \mid s,a),P^\star(\cdot \mid s,a))^2] ]\ln(|\Mcal|/\delta)}{n}}+\frac{c_2\ln(|\Mcal|/\delta)}{n},\forall P\in \Mcal_{\Dcal}. 
\end{align*}
Then, for any $P\in \Mcal_{\Dcal}$, we have 
\begin{align*}
    A(P) &\leq \sqrt{\frac{c_1\mathrm{E}_{(s,a)\sim \rho}[\TV(P(\cdot \mid s,a),P^\star(\cdot \mid s,a))^4] \ln(|\Mcal|/\delta)}{n}}+\frac{c_2\ln(|\Mcal|/\delta)}{n}\\  
    &\leq \sqrt{\frac{4c_1\mathrm{E}_{(s,a)\sim \rho}[\TV(P(\cdot \mid s,a),P^\star(\cdot \mid s,a))^2] \ln(|\Mcal|/\delta)}{n}}+\frac{c_2\ln(|\Mcal|/\delta)}{n} \tag{$[\TV(P(\cdot \mid s,a),P^\star(\cdot \mid s,a))^2]\leq 4$. }\\ 
    &\leq \sqrt{\frac{4c_1(A(P)+c\xi)\ln(|\Mcal|/\delta)}{n}}+\frac{c_2\ln(|\Mcal|/\delta)}{n}. 
\end{align*}
From $(a+b)^2\leq 2a^2+2b^2$, 
\begin{align*}
    A^2(P) & \lesssim \prns{\sqrt{\frac{c (A(P)+\xi)\ln(|\Mcal|/\delta)}{n}}+\frac{c \ln(|\Mcal|/\delta)}{n}}^2\lesssim \frac{(A(P)+\xi)\ln(|\Mcal|/\delta)}{n}+\braces{\frac{c\ln( |\Mcal|/\delta)}{n}}^2 \\
       &\lesssim \frac{(A(P)+\xi)\ln(|\Mcal|/\delta)}{n}  \tag{$\xi$ includes $\ln(|\Mcal|/\delta)$ } \\
       &\lesssim  \frac{(A(P)+1/n\ln(|\Mcal|/\delta) )\ln(|\Mcal|/\delta)}{n}. 
\end{align*}
Then, we have
\begin{align*}
    A^2(P)-B_1 A(P)- B_2\leq 0,\quad B_1=c\ln(|\Mcal|/\delta)/n,\quad B_2=c(1/n)^2 \ln(|\Mcal|/\delta)^2. 
\end{align*}
This concludes 
\begin{align*}
    0\leq A(P)\leq \frac{B_1+\sqrt{B^2_1+4B_2}}{2}\leq c(B_1+\sqrt{B_2})\leq  c\frac{\ln(|\Mcal|/\delta)}{n}\lesssim \xi.
\end{align*}
Thus, by using the above $A(P)\lesssim \xi (P\in \Mcal_{\Dcal})$ and \pref{eq:key_version}, with probability $1-\delta$,  we have:
\begin{align*}
\EE_{s,a\sim \rho} [\TV(P(\cdot \mid s,a),P^\star(\cdot \mid s,a))^2] \leq  A(P)+ c\xi \lesssim \xi,\quad P\in \Mcal_{\Dcal}. 
\end{align*}

\paragraph{Third Step: Calculate the final error bound taking the distribution shift into account.}

For any $P\in \Mcal_{\Dcal}$, we prove 
\begin{align}\label{eq:third}
           V^{\pi^{*}}_{P^{\star}}-V^{\pi^{*}}_{P} \leq (1-\gamma)^{-2} c\sqrt{C^{\dagger}_{\pi^{*}}}\sqrt{\frac{\ln(|\Mcal|/\delta)}{n}}. 
\end{align}
For any $P\in \Mcal_{\Dcal}$, this is proved as follows:  
\begin{align*}
  V^{\pi^{*}}_{P^{\star}}-       V^{\pi^{*}}_{P} & \leq (1-\gamma)^{-2}\EE_{(s,a)\sim d^{\pi^{*}}_{P^{\star}}}[\TV(P(\cdot \mid s,a),P^\star(\cdot \mid s,a))] \tag{Simulation lemma, \pref{lem:simulation}}\\
        &\leq  (1-\gamma)^{-2}\sqrt{\EE_{(s,a)\sim d^{\pi^{*}}_{P^{\star}}}[\TV(P(\cdot \mid s,a),P^\star(\cdot \mid s,a))^2]}\\
       &\leq (1-\gamma)^{-2} \sqrt{C^{\dagger}_{\pi^{*}}\EE_{(s,a)\sim \rho}[\TV(P(\cdot \mid s,a),P^\star(\cdot \mid s,a))^2]} \\
        &\leq  c(1-\gamma)^{-2} \sqrt{C^{\dagger}_{\pi^{*}}}\sqrt{\frac{\ln(|\Mcal|/\delta)}{n}} \tag{Based on the consequence of the second step}. 
\end{align*}

Combining all things together, with probability $1-\delta$, for any $\pi^{*}\in \Pi$, we have 
\begin{align*}
    V^{\pi^{*}}_{P^{\star}}-V^{\hat \pi}_{P^{\star}}&\leq    V^{\pi^{*}}_{P^{\star}}-\min_{P\in \Mcal_{\Dcal}}V^{\pi^{*}}_{P}+ \min_{P\in \Mcal_{\Dcal}}V^{\pi^{*}}_{P}- V^{\hat \pi}_{P^{\star}}\\ 
    &\leq    V^{\pi^{*}}_{P^{\star}}-\min_{P\in \Mcal_{\Dcal}}V^{\pi^{*}}_{P}+ \min_{P\in \Mcal_{\Dcal}}V^{\hat \pi}_{P}- V^{\hat \pi}_{P^{\star}} \tag{definition of $\hat \pi$}\\ 
      &\leq  V^{\pi^{*}}_{P^{\star}}-\min_{P\in \Mcal_{\Dcal}}V^{\pi^{*}}_{P}  \tag{Fist step, $P^{\star}\in \Mcal_{\Dcal}$}\\
      &\lesssim  (1-\gamma)^{-2} c_1\sqrt{C^{\dagger}_{\pi^{*}}}\sqrt{\frac{\ln(|\Mcal|c_2/\delta)}{n}}.   \tag{From \pref{eq:third}}
\end{align*}

\begin{remark}[To compete with all history-dependent polices]\label{rem:history}
Consider the case where $\Pi$ is all Markovian polices. We want to show we can compete with all history-dependent non-Markovian polices: 
\begin{align*}
   \bar \Pi=\braces{ \prod_{i=1}^{\infty} \pi_i \mid \pi_i\in \bracks{\prns{\prod_{k=1}^{i-1} \Scal\times \Acal}\to \Delta(\Acal)} }.
\end{align*}
We take an element $\pi^{*}$ from $\bar \Pi$. Then, $V^{\pi^{*}}_{P^{\star}}$ and $d^{\pi^{*}}_{P^{\star}}$ are still well-defined. Then, every step in the proof still holds. The only step we need to check carefully is this line: 
\begin{align*}
    V^{\pi^{*}}_{P^{\star}}-V^{\hat \pi}_{P^{\star}}&\leq    V^{\pi^{*}}_{P^{\star}}-\min_{P\in \Mcal_{\Dcal}}V^{\pi^{*}}_{P}+ \min_{P\in \Mcal_{\Dcal}}V^{\pi^{*}}_{P}- V^{\hat \pi}_{P^{\star}}\\ 
    &\leq    V^{\pi^{*}}_{P^{\star}}-\min_{P\in \Mcal_{\Dcal}}V^{\pi^{*}}_{P}+ \min_{P\in \Mcal_{\Dcal}}V^{\hat \pi}_{P}- V^{\hat \pi}_{P^{\star}}. 
\end{align*}
This is proved by $\max_{\pi \in \bar \Pi}V^{\pi}_{P}= \max_{\pi \in \Pi}V^{\pi}_{P}$ for any $P$. 

\end{remark}

{\newedit 
\subsection{Proofs for general function approximation for CPPO-LR with infinite hypothesis class (Proof of \pref{thm:version3}) }

We firsts show two lemmas as building blocks to prove the main statement.

\begin{lemma}\label{lem:first_type}
Set $\epsilon = 1/(n\iota(\Scal))$. With probability $1-\delta$, for any $P\in \Mcal$, 
\begin{align*}
    \EE_{(s,a)\sim \rho}[ \text{TV}( P(\cdot \mid s,a),P^{\star}(\cdot \mid s,a) )^2 ] \leq    2\EE_{\Dcal}[\log(P^{\star}(s'\mid s,a)/P(s'\mid s,a)) ] + 4n^{-1}\log (c_1 N_{[]}(\epsilon, \Mcal, \|\cdot\|_{\infty}/\delta). 
\end{align*}
\end{lemma}
\begin{proof}
 Take a $\epsilon = 1/(n\iota(\Scal))$-bracket $\{[u_i,l_i]\}$ and denote the set of upper bounds $\{l_i\}$ by $\tilde \Mcal$. Using the proof of \citet[Lemma 25 and Theorem 21]{Agarwal2020_flambe}, for any $\tilde P\in \tilde \Mcal$, 
\begin{align*}
   \EE_{(s,a)\sim \rho}[\text{TV}( \tilde P(\cdot \mid s,a),P^{\star}(\cdot \mid s,a) )^2  ] \leq  \EE_{\Dcal}[\log(P^{\star}(s'\mid s,a)/\tilde P(s'\mid s,a)) ] + 2n^{-1}\log (|\tilde \Mcal|/\delta). 
\end{align*}
Here, for any $P \in \Mcal$, we can take $\tilde P \in \tilde \Mcal$ such that 
\begin{align}
    \EE_{\Dcal}[\log(P^{\star}(s'\mid s,a)/\tilde P(s'\mid s,a)) ]\leq \EE_{\Dcal}[\log(P^{\star}(s'\mid s,a)/  P(s'\mid s,a)) ]. 
\end{align}
Besides, it satisfies 
\begin{align*}
   &\EE_{(s,a)\sim \rho}[\text{TV}( P(\cdot \mid s,a),P^{\star}(\cdot \mid s,a) )^2 ]\\
   &\leq  2\EE_{(s,a)\sim \rho}[\text{TV}( \tilde P(\cdot \mid s,a),P^{\star}(\cdot \mid s,a) )^2 ]+ 2\EE_{(s,a)\sim \rho}[\text{TV}( P(\cdot \mid s,a),\tilde P(\cdot \mid s,a) )^2 ] \\
   &\leq  2\EE_{(s,a)\sim \rho}[\text{TV}( \tilde P(\cdot \mid s,a),P^{\star}(\cdot \mid s,a) )^2 ]+ 2n^{-2}. 
\end{align*}
This concludes the statement. In the last line, we use 
\begin{align}\label{eq:tv}
    \text{TV}( \tilde P(\cdot \mid s,a),P^{\star}(\cdot \mid s,a) ) &=0.5 \int_{s'\in \Scal} | \tilde P(s' \mid s,a)-P^{\star}(s' \mid s,a)|\mathrm{d}\iota(s')\\
    &\leq 0.5 \iota(\Scal)/n \times \iota(\Scal) = 0.5/n. \nonumber
\end{align}
\end{proof}

Next, we show the following lemma.  

\begin{lemma}\label{lem:second_type}
Set $\epsilon = 1/(n\iota(\Scal))$. With probability $1-\delta$, for any $P \in \Mcal$, we have 
 \begin{align*}
     \EE_{\Dcal}[\log (P^{\star}/P)(s'\mid s,a) ]\geq -n^{-1} \log(N_{[]}(\epsilon,\Mcal,\|\cdot\|_{\infty})/\delta). 
 \end{align*}
\end{lemma}
\begin{proof}
We use Cramer-Chernoff's method. Take a  $1/(n\iota(\Scal))$-bracket $\{[u_i,l_i]\}$ and denote the set of upper bounds $\{l_i\}$ by $\tilde M$ and denote it by $\tilde \Mcal$. For $\tilde P\in \tilde \Mcal$, we have 
\begin{align*}
    & \EE\bracks{\exp\prns{ \sum_{i=1}^n \log \bracks{ \frac{\tilde P(s'^{(i)}\mid s^{(i)},a^{(i)})}{P^{\star}(s'^{(i)}\mid s^{(i)},a^{(i)})} } } } \\ 
    & \leq \EE\bracks{\exp\prns{ \sum_{i=1}^{n-1} \log \bracks{ \frac{\tilde P(s'^{(i)}\mid s^{(i)},a^{(i)})}{P^{\star}(s'^{(i)}\mid s^{(i)},a^{(i)})} } }\frac{\tilde P(s'^{(n)}\mid s^{(n)},a^{(n)})}{P^{\star}(s'^{(n)}\mid s^{(n)},a^{(n)})}  }\\
    & \leq \EE\bracks{\exp\prns{ \sum_{i=1}^{n-1} \log \bracks{ \frac{\tilde P(s'^{(i)}\mid s^{(i)},a^{(i)})}{P^{\star}(s'^{(i)}\mid s^{(i)},a^{(i)})} } }} \EE\bracks{\frac{\tilde P(s'^{(n)}\mid s^{(n)},a^{(n)})}{P^{\star}(s'^{(n)}\mid s^{(n)},a^{(n)})}  }\\
     & \leq \EE\bracks{\exp\prns{ \sum_{i=1}^{n-1} \log \bracks{ \frac{\tilde P(s'^{(i)}\mid s^{(i)},a^{(i)})}{P^{\star}(s'^{(i)}\mid s^{(i)},a^{(i)})} } }}\{1+1/n\} \tag{Use \pref{eq:tv}}\\
    &\leq ....\leq  (1+1/n)^n\leq \epsilon. 
\end{align*}
Hence by Markov's inequality, we have 
\begin{align*}
    \PP\prns{ \sum_{i=1}^n \log \bracks{ \frac{\tilde P(s'^{(i)}\mid s^{(i)},a)}{P^{\star}(s'\mid s,a)} } > \log(1/\delta)} \leq e \delta. 
\end{align*}
By taking a union bound, for any $\tilde P \in \tilde \Mcal$, we obtain 
\begin{align*}
    \PP\prns{ \sum_{i=1}^n \log \bracks{ \frac{\tilde P(s'^{(i)}\mid s^{(i)},a^{(i)})}{P^{\star}(s'^{(i)}\mid s^{(i)},a^{(i)})} } > \log( |\tilde \Mcal |/\delta)} \leq e \delta. 
\end{align*}
Finally, noting for any $P \in \Mcal$, there exists $\tilde P\in \tilde \Mcal$ s.t. $P(s'\mid s,a)\leq \tilde P(s'\mid s,a)$, we have for any $ P \in \tilde \Mcal$,  
\begin{align*}
    \PP\prns{ \sum_{i=1}^n \log \bracks{ \frac{P(s'^{(i)}\mid s^{(i)},a^{(i)})}{P^{\star}(s'^{(i)}\mid s^{(i)},a^{(i)})} } > \log( |\tilde \Mcal |/\delta)} \leq e \delta. 
\end{align*}
    
\end{proof}

We condition on events where \pref{lem:first_type} and \pref{lem:second_type} hold. 

\paragraph{First Step (pessimism).}

Lemma~\ref{lem:second_type} tells us that  $P^{\star} \in \bar \Mcal_{\Dcal}$.  

\paragraph{Second Step.}

Lemma~\ref{lem:first_type} implies for any $P \in \bar \Mcal_{\Dcal}$, 
\begin{align*}
    \EE_{(s,a)\sim \rho}[\text{TV}( P(\cdot \mid s,a),P^{\star}(\cdot \mid s,a) )^2 ]\leq n^{-1} \log(N_{[]}(\epsilon,\Mcal,\|\cdot\|_{\infty})/\delta) 
\end{align*}
using the definition of $P \in \bar \Mcal_{\Dcal}$.

\paragraph{Third step: calculate the final bound taking the distribution shift into account.}

For any $P\in \bar \Mcal_{\Dcal}$, we prove 
\begin{align}\label{eq:third_lr}
           V^{\pi^{*}}_{P^{\star}}-V^{\pi^{*}}_{P} \leq (1-\gamma)^{-2} c\sqrt{C^{\dagger}_{\pi^{*}}}\sqrt{\frac{\ln(N_{[]}(\epsilon,\Mcal,\|\cdot\|_{\infty})/\delta)}{n}}. 
\end{align}
For any $P\in \bar \Mcal_{\Dcal}$, this is proved as follows:  
\begin{align*}
  V^{\pi^{*}}_{P^{\star}}-       V^{\pi^{*}}_{P} & \leq (1-\gamma)^{-2}\EE_{(s,a)\sim d^{\pi^{*}}_{P^{\star}}}[\TV(P(\cdot \mid s,a),P^\star(\cdot \mid s,a))] \tag{Simulation lemma}\\
        &\leq  (1-\gamma)^{-2}\sqrt{\EE_{(s,a)\sim d^{\pi^{*}}_{P^{\star}}}[\TV(P(\cdot \mid s,a),P^\star(\cdot \mid s,a))^2]}\\
       &\leq (1-\gamma)^{-2} \sqrt{C^{\dagger}_{\pi^{*}}\EE_{(s,a)\sim \rho}[\TV(P(\cdot \mid s,a),P^\star(\cdot \mid s,a))^2]} \\
        &\leq  c(1-\gamma)^{-2} \sqrt{C^{\dagger}_{\pi^{*}}}\sqrt{\frac{\ln(N_{[]}(\epsilon,\Mcal,\|\cdot\|_{\infty})/\delta)}{n}} \tag{Based on the consequence of the second step}. 
\end{align*}

Combining all things together, with probability $1-\delta$, for any $\pi^{*}\in \Pi$, we have 
\begin{align*}
    V^{\pi^{*}}_{P^{\star}}-V^{\hat \pi}_{P^{\star}}&\leq    V^{\pi^{*}}_{P^{\star}}-\min_{P\in \Mcal_{\Dcal}}V^{\pi^{*}}_{P}+ \min_{P\in \Mcal_{\Dcal}}V^{\pi^{*}}_{P}- V^{\hat \pi}_{P^{\star}}\\ 
    &\leq    V^{\pi^{*}}_{P^{\star}}-\min_{P\in \Mcal_{\Dcal}}V^{\pi^{*}}_{P}+ \min_{P\in \Mcal_{\Dcal}}V^{\hat \pi}_{P}- V^{\hat \pi}_{P^{\star}} \tag{definition of $\hat \pi$}\\ 
      &\leq  V^{\pi^{*}}_{P^{\star}}-\min_{P\in \Mcal_{\Dcal}}V^{\pi^{*}}_{P}  \tag{First step, $P^{\star}\in \Mcal_{\Dcal}$}\\
      &\lesssim  (1-\gamma)^{-2} c_1\sqrt{C^{\dagger}_{\pi^{*}}}\sqrt{\frac{\ln( N_{[]}(\epsilon,\Mcal,\|\cdot\|_{\infty}) c_2/\delta)}{n}}.   \tag{From \pref{eq:third_lr}}
\end{align*}
}

\section{Missing Proofs in \pref{sec:examples}}

\subsection{Proofs for Tabular MDPs (Proof of Corollary \ref{cor:tabular})}

Here, we show the result for CPPO-TV. The result in CPPO-LR is obtained in the proof of Corollary \ref{cor:linear_mixture}. We prove in a similar way as \pref{thm:version}. 

\paragraph{First step.}
We set $\xi=c \frac{|\Scal|^2|\Acal|\ln (n|\Scal|\Acal|c_2/\delta)\}}{n}. $  Then, from \pref{lem:tabular_mle}, with probability $1-\delta$, we can show $P^{\star}\in \Mcal_{\Dcal}$ since 
\begin{align*}
     \mathbb{E}_{s,a\sim \Dcal}\bracks{\TV(\widehat{P}_{\MLE}(\cdot \mid s,a),P^\star(\cdot \mid s,a))^2} \leq \xi. 
\end{align*}
Hereafter, we condition on the above event. 

\paragraph{Second step. }

Following the second step in the proof of \pref{thm:version} based on \pref{eq:key_version}, for any $P\in \Mcal_{\Dcal}$, we have
\begin{align}\label{eq:second_bound}
    \EE_{s,a\sim \rho}\bracks{ \TV(P(\cdot \mid s,a),P^\star(\cdot \mid s,a))^2}\leq  c\xi+ A(P) 
\end{align}
where
\begin{align*}
    A(P)\coloneqq |\EE_{s,a\sim \rho} [\TV(P(\cdot \mid s,a),P^\star(\cdot \mid s,a))^2] -\EE_{\Dcal} [\TV(P(\cdot \mid s,a),P^\star(\cdot \mid s,a))^2]|. 
\end{align*}
Our goal here is showing with probability $1-\delta$, 
\begin{align}\label{eq:empirical}
   A(P) \lesssim \xi,\forall P\in \Mcal_{\Dcal}. 
\end{align}

To prove \pref{eq:empirical}, consider an $\epsilon$-net $\{P_1(s,a),\cdots,P_{K}(s,a)\}$ covering a simplex in terms of $\|\cdot\|_{1}$ \footnote{In the tabular setting, since the state space is countable, it is equivalent to L1 distance.} for each fixed pair $(s,a)\in \Scal \times \Acal$. We take $\epsilon=1/n$. Since the covering number $K$ is upper-bounded by $(c/\epsilon)^{|\Scal|}$ \citep[Lemma 5.7]{WainwrightMartinJ2019HS:A}, we can obtain $\bar M=\{P_1,\cdots,P_{K^{|\Scal| \times|\Acal|} } \}$ s.t. for any possible $P\subset \Scal\times \Acal\to \Delta(\Scal)$, there exists $P_i$ s.t. $$\TV(P_i(\cdot \mid s,a),P(\cdot \mid s,a))\leq \epsilon,\forall (s,a).$$ 
This implies for any $P\subset \Scal\times \Acal\to \Delta(\Scal)$, there exists $P_i(\cdot\mid s,a)$ s.t. $\forall(s,a)$, 
\begin{align}
   &| \TV(P(\cdot \mid s,a),P^\star(\cdot \mid s,a))^2-\TV(P_i(\cdot \mid s,a),P^\star(\cdot \mid s,a))^2| \nonumber  \\
& \leq 4 | \TV(P(\cdot \mid s,a),P^\star(\cdot \mid s,a))-\TV(P_i(\cdot \mid s,a),P^\star(\cdot \mid s,a))| \tag{$a^2-b^2=(a-b)(a+b)$}\\
   &\leq 4 \TV(P\cdot \mid s,a),P_i(\cdot \mid s,a)) \tag{$|\|a\|-\|b\||\leq \|a-b\|$ }\\
   &\leq 4 \epsilon.  \label{eq:epsilon}
\end{align}

\begin{figure}[!tbp]
    \centering
    \includegraphics[width=0.5\textwidth]{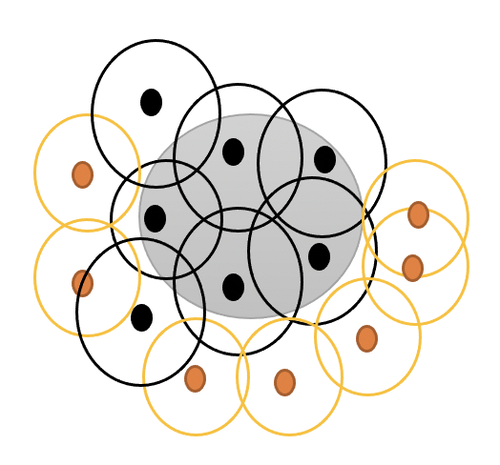}
    \caption{$\Mcal_{\Dcal}$ is colored in gray. $\Mcal'$ corresponds to the set of black dots. Orange dots correspond to $\bar \Mcal$, which do not belong to $\Mcal'$. }
    \label{fig:illusration}
\end{figure}
We often use this property \pref{eq:epsilon} hereafter.   

Next, we define $\Mcal'\subset \bar M$ so that it covers $\Mcal_{\Dcal}$. Concretely, we define $\Mcal'$: 
\begin{align}\label{eq:impo_fact}
    \Mcal'=\{P\in \bar M: \exists P''\in \Mcal_{\Dcal} , \TV(P(\cdot \mid s,a),P''(\cdot \mid s,a))\leq \epsilon \quad \forall(s,a) \}. 
\end{align}
The construction is illustrated  in \pref{fig:illusration}. Here, from the definition, for any $P\in \Mcal_{\Dcal}$ , we can also find $P'\in \Mcal'$ s.t. $$\TV(P(\cdot \mid s,a),P'(\cdot \mid s,a))\leq \epsilon, \forall(s,a).$$
This is because from the definition of $\bar M$, we can always find $P\in \bar \Mcal$ satisfying the above. Such $P$ belongs to $\Mcal'$ from the definition of $\Mcal'$. We use this fact later. 

 Then, from \pref{eq:impo_fact} and recalling \pref{eq:second_bound}, we have 
\begin{align}\label{eq:third_bound}
    \EE_{s,a\sim \rho}\bracks{ \TV(P(\cdot \mid s,a),P^\star(\cdot \mid s,a))^2}\lesssim  \xi+ A(P),\quad \forall P \in \Mcal'. 
\end{align}
because
\begin{align*}
        &\EE_{s,a\sim \rho}\bracks{ \TV(P(\cdot \mid s,a),P^\star(\cdot \mid s,a))^2}\\
        &\leq \EE_{s,a\sim \rho}\bracks{ \TV(P(\cdot \mid s,a),P''(\cdot \mid s,a))^2}+\EE_{s,a\sim \rho}\bracks{ \TV(P''(\cdot \mid s,a),P^\star(\cdot \mid s,a))^2}\\
        & \leq \EE_{s,a\sim \rho}\bracks{\TV(P''(\cdot \mid s,a),P^\star(\cdot \mid s,a))^2 }+\epsilon^2 \tag{Take some $P''\in \Mcal_{\Dcal}$ noting \pref{eq:impo_fact}}\\
        &\leq c\xi+ A(P).  \tag{From \pref{eq:second_bound}} 
\end{align*}

Then, with probability $1-\delta$, from Bernstein's inequality, we have
\begin{align*}
    A(P)   \leq \sqrt{c\frac{\mathrm{var}[ \TV(P(\cdot \mid s,a),P^\star(\cdot \mid s,a))^2]\ln(K^{|\Scal| \times|\Acal|}/\delta)}{n}}+\frac{c \ln(K^{|\Scal| \times|\Acal|}/\delta)}{n},\forall P\in \Mcal . 
\end{align*}
Hereafter, we condition on the above event. Based on \pref{eq:third_bound},  we can state 
\begin{align*}
    \mathrm{var}[ \TV(P(\cdot \mid s,a),P^\star(\cdot \mid s,a))^2] \lesssim \mathrm{E}[ \TV(P(\cdot \mid s,a),P^\star(\cdot \mid s,a))^2]\lesssim \xi+ A(P),\quad \forall P \in \Mcal', 
\end{align*}
with probability $1-\delta$. Following the argument of \pref{thm:version}, for $P\in \Mcal'$, we have 
\begin{align*}
    A^2(P)-A(P) B_1-B_2\leq 0, \quad B_1=\frac{\ln(K^{|\Scal| \times|\Acal|}/\delta)}{n}, \quad B_2= \xi \frac{\ln(K^{|\Scal| \times|\Acal|}/\delta)}{n}+\prns{\frac{\ln(K^{|\Scal| \times|\Acal|}/\delta)}{n}}^2. 
\end{align*}
Then, with probability $1-\delta$, we have 
\begin{align}\label{eq:union}
    A(P)\leq \frac{\ln(K^{|\Scal| \times|\Acal|}/\delta)}{n}+\sqrt{\frac{\ln(K^{|\Scal| \times|\Acal|}/\delta)}{n}}\xi^{1/2}\lesssim \xi,\quad \forall P\in \Mcal'. 
\end{align}
 This shows  for any $P\in \Mcal_{\Dcal}$, we have 
\begin{align*}
      &|\{\EE_{\Dcal}-\EE_{(s,a)\sim \rho} \}[\TV(P(\cdot \mid s,a),P^\star(\cdot \mid s,a))^2]| \\
      &\leq |\{\EE_{\Dcal}-\EE_{(s,a)\sim \rho} \}[\TV(P'(\cdot \mid s,a),P(\cdot \mid s,a))^2+\TV(P'(\cdot \mid s,a),P^\star(\cdot \mid s,a))^2 ]| \tag{We take $P'\in \Mcal'$ such that \pref{eq:impo_fact}} \\
      &\leq |\{\EE_{\Dcal}-\EE_{(s,a)\sim \rho} \}[\TV(P'(\cdot \mid s,a),P^\star(\cdot \mid s,a))^2]+8\epsilon \tag{From the definition of $\Mcal'$} \\ 
      &\lesssim  \xi.  \tag{From \pref{eq:union} and $P'\in \Mcal'$}
\end{align*}
Thus, \pref{eq:empirical} is proved. 

\paragraph{Third step.  }

We follow the third step of \pref{thm:version}: 
\begin{align*}
        V^{\pi^{*}}_{P^{\star}}-V^{\hat \pi}_{P^{\star}}  \lesssim (1-\gamma)^{-2}\sqrt{C^{\dagger}_{\pi^{*}}\xi}. 
\end{align*}

\subsection{Proofs for Linear Mixture  MDPs (Proof of Corollary \ref{cor:linear_mixture})}

We follow the way in Theorem \ref{thm:version2}. Let $P(\theta)=\theta^{\top}\psi(s,a,s')$.

We first calculate the bracketing number. By letting $\theta^{(1)},\cdots,\theta^{(K)}$ be an $\epsilon$-cover of the $d$-dimensional ball with a radius $R$, i.e, $B_d(R)$, we have the brackets $\{[P(\theta^{(i)})-\epsilon,P(\theta^{(i)})+\epsilon]\}_{i=1}^{K}$, which cover $\Mcal_{mix}$. This is because for any $P(\theta)\in \Mcal_{mix}$, we can take $\theta^{(i)}$ s.t. $\|\theta-\theta^{(i)}\|_2\leq \epsilon $, then, 
\begin{align*}
    P(\theta^{(i)})-\epsilon<P(\theta)<     P(\theta^{(i)})+\epsilon,\quad \forall(s,a,s')
\end{align*}
noting 
\begin{align}\label{eq:braket}
|P(\theta)(s,a,s') -P(\theta^{(i)})(s,a,s')|\leq \|\theta-\theta^{(i)}\|_2\leq \epsilon, \quad \forall(s,a,s')
\end{align}
The last equality is from \pref{lem:mixture}. 

The brackets above are size of $\epsilon$. Therefore, we have
\begin{align*}
    \Ncal_{[]}(\epsilon,\Mcal_{mix},\|\cdot\|_2)\leq \Ncal(\epsilon, B_d(cR),\|\cdot\|_2), 
\end{align*}
where $\Ncal(\epsilon, B_d(cR),\|\cdot\|_2)$ is a covering number of $ B_d(cR)$ w.r.t $\|\cdot\|_2$. This is upper-bounded by $(c R/\epsilon)^d$ \citep[Lemma 5.7]{WainwrightMartinJ2019HS:A}.  

\paragraph{First Step (pessimism).}

Lemma 2 tells us that  $P(\theta^{\star}) \in \Mcal_{\Dcal}$.  

\paragraph{Second Step.}

Lemma 1 implies for any $P(\theta) \in \Mcal_{\Dcal}$, 
\begin{align*}
    \EE_{(s,a)\sim \rho}[\TV(P(\theta)(\cdot \mid s,a), P(\theta^{\star})(\cdot \mid s,a))^2 ]\leq \beta. 
\end{align*}
using the definition of $P(\theta) \in \Mcal_{\Dcal}$. 

\paragraph{Third step: distribution shift part.}

Here, for $P\in \Mcal_{\Dcal}$ we prove 
\begin{align}\label{eq:inter_gooal}
      V^{\pi^{*}}_{P^{\star}}-       V^{\pi^{*}}_{P} &\lesssim (1-\gamma)^{-2}\sqrt{dC_{\pi^{*},\mathrm{mix}}\beta},\\
       V^{\pi^{*}}_{P^{\star}}-       V^{\pi^{*}}_{P} &\lesssim (1-\gamma)^{-2}\sqrt{C^{\dagger}_{\pi^{*}}\beta}. \label{eq:inter_gooal2}
\end{align}
Following the third step of the proof of Theorem~\ref{thm:version3}, this immediately concludes the bound 
\begin{align*}
      V^{\pi^{*}}_{P^{\star}}-        V^{\hat \pi}_{P^{\star}} &\lesssim (1-\gamma)^{-2}\sqrt{d C_{\pi^{*},\mathrm{mix}}\beta},\\
        V^{\pi^{*}}_{P^{\star}}-        V^{\hat \pi}_{P^{\star}}&\lesssim (1-\gamma)^{-2}\sqrt{C^{\dagger}_{\pi^{*}}\beta}. 
\end{align*}

Since \pref{eq:inter_gooal2} is obvious from simulation lemma, we only prove \pref{eq:inter_gooal}. To prove \pref{eq:inter_gooal}, we take a distribution $P(\theta)\in \Mcal_{\Dcal}$. First, recall for $P(\theta)\in \Mcal_{\Dcal}$,  we have 
\begin{align*}
   \EE_{(s,a)\sim \rho}[\TV(P(\theta^{\star})(\cdot \mid s,a),P(\theta)(\cdot \mid s,a))^2]\lesssim  \beta. 
\end{align*}
From the third statement of \pref{lem:mixture}, for any $V:\Scal \to [0,1]$, we have 
\begin{align*}
   \E_{(s,a)\sim \rho}[ |(\theta-\theta^{*})^{\top}\psi_{V}(s,a)|^2]\lesssim \beta.  
\end{align*}
Thus,
\begin{align*}
 \forall V:\Scal \to [0,1],\quad (\theta-\theta^{*})^{\top}  \Sigma_{\rho,V }(\theta-\theta^{*})\lesssim \beta,\quad \Sigma_{\rho,V }=\E_{(s,a)\sim \rho}[\psi_{V}(s,a)\psi^{\top}_{V}(s,a)]. 
\end{align*}
Here, we have 
\begin{align*}
  V^{\pi^{*}}_{P^{\star}}-       V^{\pi^{*}}_{P} & \leq (1-\gamma)^{-1}\left|\EE_{(s,a)\sim d^{\pi^{*}}_{P^{\star}}}\bracks{\int \{P(s' \mid s,a) - P^\star(s' \mid s,a)\}V^{\pi^{*}}_{P}(s')\rd(s')  }\right| \tag{Simulation lemma, \pref{lem:simulation}}\\
    &\leq  (1-\gamma)^{-1}\left|\EE_{(s,a)\sim d^{\pi^{*}}_{P^{\star}}}\bracks{(\theta-\theta^{*})\psi_{V^{\pi^{*}}_{P}}(s,a) }\right| \tag{Recall $\psi_{V}=\int \psi(s,a,s')V^{\pi^{*}}_{P}(s')\rd(s')$} \\ 
    &\leq  (1-\gamma)^{-1}\underbrace{\|\theta-\theta^{*}\|_{\lambda I+\Sigma_{\rho,V^{\pi^{*}}_{P} }}}_{(a)}\underbrace{\EE_{(s,a)\sim d^{\pi^{*}}_{P^{\star}}}  \bracks{\|\psi_{V^{\pi^{*}}_{P}}(s,a)\|_{(\Sigma_{\rho,V^{\pi^{*}}_{P} }+\lambda I)^{-1}} }}_{(b)}.  \tag{CS inequality}
\end{align*}
The first term (a) is upper-bounded by $\sqrt{\{(1-\gamma)^{-2} \beta+\lambda R^2 \}}$ noting $0\leq V^{\pi^*}_P\leq (1-\gamma)^{-1}$. The term (b) is  upper-bounded by 
\begin{align*}
    \EE_{(s,a)\sim d^{\pi^{*}}_{P^{\star}}}  \bracks{\|\psi_{V^{\pi^{*}}_{P}}(s,a)\|_{{(\Sigma_{\rho,V^{\pi^{*}}_{P} }+\lambda I)^{-1}} } } &\leq     \EE_{(s,a)\sim d^{\pi^{*}}_{P^{\star}}}  \tag{Jensen's inequality} \bracks{\|\psi_{V^{\pi^{*}}_{P}}(s,a)\|^2_{{(\Sigma_{\rho,V^{\pi^{*}}_{P} }+\lambda I)^{-1}} } }^{1/2} \\
    &=  \sqrt{\Tr( \Sigma_{d^{\pi^{*}}_{P^{\star}},V^{\pi^{*}}_{P} }(\lambda I+\Sigma_{\rho,V^{\pi^{*}}_{P} })^{-1} )  }\\
    &\leq \sqrt{C_{\pi^{*},\mathrm{mix}}\Tr( \Sigma_{\rho,V^{\pi^{*}}_{P} }(\lambda I+\Sigma_{\rho,V^{\pi^{*}}_{P} })^{-1} )  } \tag{From \pref{lem:distribution_shift}}\\
    &\leq  \sqrt{C_{\pi^{*},\mathrm{mix}} \rank(\Sigma_{\rho,V^{\pi^{*}}_{P} })}\leq \sqrt{C_{\pi^{*},\mathrm{mix}}d } . 
\end{align*}
By taking $\lambda$ s.t. $\lambda R^2 \lesssim (1-\gamma)^{-2} \beta$, \pref{eq:inter_gooal} is proved. 

For linear MDPs, from the fourth statement of \pref{lem:mixture}, $ C_{\pi^{*},\mathrm{mix}}\leq \bar C_{\pi^{*}}$. Then, the statement is concluded.

\subsection{Proofs for Low-rank MDPs (Proof of \pref{thm:low_rank})}

Until the second step, we can perform the same analysis as \pref{thm:version}. More concretely, with probability $1-\delta$, we have $P^{\star}\in \Mcal_{\Dcal}$ and 
\begin{align}\label{eq:second_step_guara}
    \EE_{s,a\sim \rho} [\TV(P(\cdot \mid s,a),P^\star(\cdot \mid s,a))^2]\leq  \xi,\quad \forall P\in \Mcal_{\Dcal}, \xi\coloneqq c\frac{\ln(|\Mcal|/\delta)}{n}. 
\end{align}
Hereafter, we condition on the above event. 

Letting $f(s,a)=\TV(P(\cdot \mid s,a), P^\star(\cdot \mid s,a))$, we use \pref{lem:useful_offline2}, which will be showed later. Then, 
\begin{align*}
    \E_{(s,a)\sim d^{\pi}_{P^{\star}}}[f(s,a)]\leq \E_{(s,a)\sim d^{\pi}_{P^{\star}}}[\|\phi^{\star}(s,a)\|_{\Sigma^{-1}_{\rho,\phi^{\star}}}]\sqrt{n\gamma\omega_{\pi}\E_{\rho}[f^2(s,a)]+4\gamma^2 \lambda d}+\sqrt{(1-\gamma)\omega_{\pi}\E_{\rho}[f^2(s,a)]} 
\end{align*}
where  $\Sigma_{\rho,\phi^{\star}}=n \E_{\rho}[\phi^{\star}{\phi^{\star}}^{\top}]+\lambda I$. 
We consider how to bound $\E_{(s,a)\sim d^{\pi}_{P^{\star}}}[\|\phi^{\star}(s,a)\|_{\Sigma^{-1}_{\rho,\phi^{\star}}}]$. This is upper-bounded by
\begin{align*}
    \E_{(s,a)\sim d^{\pi}_{P^{\star}}}[\|\phi^{\star}(s,a)\|_{\Sigma^{-1}_{\rho,\phi^{\star}}}] &\leq \sqrt{\tr(\E_{(s,a)\sim d^{\pi}_{P^{\star}}}[\phi^{\star}{\phi^{\star}}^{\top}]\Sigma^{-1}_{\rho,\phi^{\star}})}\\ 
    &\leq \sqrt{\bar C_{\pi,\phi^{\star}}\tr(\E_{(s,a)\sim \rho}[\phi^{\star}{\phi^{\star}}^{\top}]\Sigma^{-1}_{\rho,\phi^{\star}})} \tag{From \pref{lem:distribution_shift}}\\ 
    &\leq \sqrt{\bar C_{\pi,\phi^{\star}}\rank(\Sigma_{\rho})/n}. 
\end{align*}
Here, in the last line, by letting the SVD of $\Sigma_{\rho}=\E_{\rho}[\phi\phi^{\top}]$ be $U\tilde \Sigma_{\rho}U^{\top}$ where $\tilde \Sigma_{\rho}$ is a $d\times d$ diagonal matrix and $U$ is a  $d\times d$ orthogonal matrix ,  we use 
\begin{align*}
    \tr\left( \Sigma_{\rho} \Sigma^{-1}_{\rho,\phi^{\star}}  \right)&=\mathrm{tr}(U\tilde \Sigma_{\rho} U^{\top}\{n U\tilde \Sigma_{\rho} U^{\top}+\lambda I\}^{-1})= \mathrm{tr}(\tilde \Sigma_{\rho} U^{\top}\{n U\tilde \Sigma_{\rho} U^{\top}+\lambda I\}^{-1}U)\\
    &= \mathrm{tr}(\tilde \Sigma_{\rho} U^{\top}\{U\{n \tilde\Sigma_{\rho}+\lambda I\} U^{\top}\}^{-1}U)\\
      &= \mathrm{tr}(\tilde \Sigma_{\rho} U^{\top}U\{n \tilde\Sigma_{\rho}+\lambda I \}^{-1}U^{\top}U)\\
          &= \mathrm{tr}(\tilde \Sigma_{\rho} \{n \tilde\Sigma_{\rho}+\lambda I\})^{-1}\leq \rank(\Sigma_{\rho})/n. 
\end{align*}

Hence, when $P\in \Mcal_{\Dcal}$, by setting $\lambda$ s.t. $\lambda d\lesssim n\omega_{\pi}\xi $, we have
\begin{align*}
      \E_{(s,a)\sim d^{\pi}_{P^{\star}}}[f(s,a)]\leq \sqrt{\frac{\gamma\bar C_{\pi^{*},\phi^{\star}}\rank(\Sigma_{\rho}) \omega_{\pi}\ln(|\Mcal|/\delta)}{n}}+\sqrt{\frac{(1-\gamma)\omega_{\pi}\ln(|\Mcal|/\delta)}{n}}. 
\end{align*}
We use \pref{eq:second_step_guara} here.

Finally, 
\begin{align*}
    & V^{\pi^{*}}_{P^{\star}}-V^{\hat \pi}_{P^{\star}}\\
    &\leq  V^{\pi^{*}}_{P^{\star}}-\min_{P\in \Mcal_{\Dcal}}V^{\pi^{*}}_{P}  \tag{Recall the proof of the third step in the proof of \pref{thm:version}}\\
    &\leq  (1-\gamma)^{-2} \mathbb{E}_{s,a\sim d_{P^\star}^{\pi^\star}} \TV({P'}(s,a) ,P^\star(\cdot \mid s,a)) \tag{$P'=\argmin_{P\in \Mcal_{\Dcal}} V^{\pi^{*}}_{P} $}\\
    &\lesssim \sqrt{\frac{\bar C_{\pi^{*},\phi^{\star}}\rank(\Sigma_{\rho}) \omega_{\pi^{*}}\ln(|\Mcal|/\delta)}{n(1-\gamma)^4}}.
\end{align*}

The following inequality is an important lemma to connect $\EE_{(s,a)\sim d^{\pi}_{P^{\star} }}\braces{f(s,a)}$  with an elliptical potential $\EE_{(\tilde s,\tilde a)\sim d^{\pi}_{P^{\star} }} \|\phi^{\star}(\tilde s,\tilde a)\|_{\Sigma^{-1}_{\rho,\phi^{\star}}}$. 

\begin{lemma}[One-step back inequality]\label{lem:useful_offline2}
Take any $f\subset \Scal\times \Acal \to \RR$ s.t.  $\|f\|_{\infty}\leq B$ and $0<\lambda\in \RR$. Letting  $\omega=\max_{s,a}(\pi(a\mid s)/\pi_b(a\mid s))$, for any policy $\pi$,  we have 
\begin{align*}
|\EE_{(s,a)\sim d^{\pi}_{P^{\star} }}\braces{f(s,a)}| & \leq 
\EE_{(\tilde s,\tilde a)\sim d^{\pi}_{P^{\star} }} \|\phi^{\star}(\tilde s,\tilde a)\|_{\Sigma^{-1}} \sqrt{\braces{n\omega_{\pi}\gamma\EE_{ (s,a)\sim \rho}\bracks{f^2(s,a) }}+\gamma^2 \lambda d B^2}
 \\
&+ \sqrt{(1-\gamma) \omega_{\pi}\EE_{ (s,a)\sim \rho}\bracks{f^2(s,a) } } . 
\end{align*}
where $\Sigma=n \E_{(s,a)\sim \rho}[\phi^{\star}(s,a){\phi^{\star}}^{\top}(s,a)]+\lambda I$. 

\end{lemma}
\begin{proof}[Proof of Lemma~\ref{lem:useful_offline2}]

First, we have an equality: 
\begin{align}\label{eq:first_offline}
    \EE_{(s,a)\sim d^{\pi}_{P^{\star} }}\braces{f(s,a)}=\gamma   \EE_{(\tilde s,\tilde a)\sim d^{\pi}_{P^{\star} },s\sim P^{\star}(\tilde s,\tilde a)}\braces{f(s,a)}+(1-\gamma)\EE_{s \sim d_0, a\sim \pi(s_0)}\braces{f(s,a)}. 
\end{align}

The second term in \pref{eq:first_offline}  is upper-bounded by 
\begin{align*}
\EE_{s \sim d_0, a\sim \pi(s_0)}\braces{f(s,a)}\leq \EE_{s \sim d_0, a\sim \pi(s_0)}\braces{f^2(s,a)}\}^{1/2}=  \sqrt{\omega_{\pi}\EE_{ (s,a)\sim \rho}\bracks{f^2(s,a) }/(1-\gamma)  } . 
\end{align*}

Next we consider the first term in \pref{eq:first_offline}. By CS inequality, we have 
\begin{align*}
    &\left |\EE_{(\tilde s,\tilde a)\sim d^{\pi}_{P^{\star} },s\sim P^{\star}(\tilde s,\tilde a)}\braces{f(s,a)}\right|=\left|\EE_{(\tilde s,\tilde a)\sim d^{\pi}_{P^{\star} }}\phi^{\star}(\tilde s,\tilde a)^{\top}\int \hat \mu(s)\pi(a\mid s)f(s,a) d(s,a)\right|\\ 
   &\leq \EE_{(\tilde s,\tilde a)\sim d^{\pi}_{P^{\star} }} \|\phi^{\star}(\tilde s,\tilde a)\|_{\Sigma_{\rho,\phi^{\star}}^{-1}}\|\int \hat \mu(s)\pi(a\mid s)f(s,a) d(s,a)\|_{\Sigma_{\rho,\phi^{\star}}}.
\end{align*} 
Then, 
\begin{align*}
  & \|\int \hat \mu(s)\pi(a\mid s)f(s,a) d(s,a)\|^2_{\Sigma_{\rho,\phi^{\star}}}\\
&\leq  \braces{\int \hat \mu(s)\pi(a\mid s)f(s,a) d(s,a)}^{\top}\braces{n \EE_{(s,a)\sim \rho}[\phi^{\star} {\phi^{\star}}^{\top}]+\lambda I  }\braces{\int \hat \mu(s)\pi(a\mid s)f(s,a) d(s,a)}\\
&\leq  n \braces{\EE_{(\tilde s,\tilde a)\sim \rho}\bracks{\int \hat \mu(s)^{\top}\phi^{\star}(\tilde s,\tilde a)\pi(a\mid s)f(s,a) d(s,a)}}^2+ B^2\lambda d \tag{Use the assumption $\|f(s,a)\|_{\infty}\leq B$ and $\|\int \hat \mu(s)\rd(s)\|_2\leq \sqrt{d}$ }\\
&=  n \braces{\EE_{(\tilde s,\tilde a)\sim \rho, s\sim P^{\star}(\tilde s,\tilde a), a\sim \pi(s)}\bracks{f(s,a) }}^2+ B^2\lambda d \\
&\leq  n \braces{\EE_{(\tilde s,\tilde a)\sim \rho, s\sim P^{\star}(\tilde s,\tilde a), a\sim \pi(s)}\bracks{f^2(s,a) }}+ B^2\lambda d.  \tag{Jensen} 
\end{align*}
Finally, the the first term in \pref{eq:first_offline} is upper-bounded by 
\begin{align*}
   & n  \braces{\EE_{(\tilde s,\tilde a)\sim \rho, s\sim P^{\star}(\tilde s,\tilde a), a\sim \pi(s)}\bracks{f^2(s,a) }}+ \lambda d B^2 \\
   &\leq n\omega_{\pi} \braces{\EE_{(\tilde s,\tilde a)\sim \rho, s\sim P^{\star}(\tilde s,\tilde a), a\sim \pi_b(s)}\bracks{f^2(s,a) }}+ \lambda d  B^2  \tag{Importance sampling}\\
     &\leq n\omega_{\pi} \braces{\frac{1}{\gamma}\EE_{ (s,a)\sim \rho}\bracks{f^2(s,a) }}+ \lambda d B^2. \tag{Definition of $\rho$}
\end{align*}
The final statement is immediately concluded. 

\end{proof}

\subsection{Proofs for Factored MDPs (Proof of \pref{thm:factoed})}\label{subsec:factored_mdp}

We focus on the proof of modified CPPO-TV. The proof of CPPO-LR is similarly completed. 

We denote the constrained set as $\Mcal_{\Dcal}$:
\begin{align*}
    \Mcal_{\Dcal}= \braces{P=\prod_i P_i \mid  \EE_{\Dcal}\bracks{\TV(  \widehat{P}_{\MLE,i}(\cdot \mid s[pa_i],a), P_i(\cdot \mid s[pa_i],a))^2 }\leq \xi_i,\forall i \in [1,\cdots,d] }. 
\end{align*}
Following the first step in the proof of Corollary \pref{cor:tabular}, with probability $1-\delta$, the product $\prod_i P^\star_i$ is in $\Mcal_{\Dcal}$, i.e., 
\begin{align*}\textstyle
     \mathbb{E}_{s,a\sim \Dcal}\bracks{\TV( \widehat{P}_{\MLE,i}(\cdot \mid s[pa_i],a), P^\star_i(\cdot \mid s[pa_i],a))^2 } \leq \xi_i, \forall i \in [1,\cdots,d], \quad \xi_i=\sqrt{\frac{L_i\log(L_i d/\delta)}{n}}. 
\end{align*}
Note $d$ comes from the union bound. Besides, following the second step in the proof of Corollary \pref{cor:tabular}, for any $P\in \Mcal_{\Dcal}$, with probability $1-\delta$, 
\begin{align*}
     \mathbb{E}_{s,a\sim \rho}\bracks{TV( \widehat{P}_{i}(\cdot \mid s[pa_i],a), P^\star_i(\cdot \mid s[pa_i],a))^2} \leq \xi_i,\forall i \in [1,\cdots,d]. 
\end{align*}

After conditioning on the above two events, then, for any $P\in \Mcal_{\Dcal}$ and $\pi^{\star}\in \Pi$, we have 
\begin{align*}
  V^{\pi^{*}}_{P^{\star}}-       V^{\pi^{*}}_{P} & \leq (1-\gamma)^{-2}\EE_{(s,a)\sim d^{\pi^{*}}_{P^{\star}}}[\TV(P(\cdot \mid s,a),P^\star(\cdot \mid s,a))] \tag{Simulation lemma, \pref{lem:simulation}}\\
  & \leq (1-\gamma)^{-2}\EE_{(s,a)\sim d^{\pi^{*}}_{P^{\star}}}[\sum_i \TV(P_i(\cdot \mid s[pa_i],a),P^\star_i(\cdot \mid s[pa_i],a))] \\
    & \leq (1-\gamma)^{-2}\sum_i \sqrt{\EE_{(s,a)\sim\rho}\bracks{\prns{\frac{d^{\pi^{*}}_{P^{\star}}(s[pa_i],a)}{\rho(s[pa_i],a)}}^2 }\EE_{(s,a)\sim \rho}[\TV(P_i(\cdot \mid s[pa_i],a),P^\star_i(\cdot \mid s[pa_i],a))^2]} \tag{CS inequality} \\
       & \leq (1-\gamma)^{-2}\sum_i \sqrt{\ddot C_{\pi^{*},\infty}\EE_{(s,a)\sim \rho}[\TV(P_i(\cdot \mid s,a),P^\star_i(\cdot \mid s,a))^2]} \leq (1-\gamma)^{-2}\sum_i \sqrt{\ddot C_{\pi^{*},\infty}\xi_i} \\ 
           & \leq (1-\gamma)^{-2} \sqrt{d \ddot C_{\pi^{*},\infty}\sum_i\xi_i}  \tag{CS inequality} \\ 
        &\leq c(1-\gamma)^{-2} \sqrt{d \ddot{C}_{\pi^{*},\infty}\frac{L\ln(Lnd/\delta)}{n}}. 
\end{align*}
Here, recall  
\begin{align*}
    \ddot{C}_{\pi^{*},\infty}=\max_{i\in [1,\cdots,d]} \EE_{(s,a)\sim\rho}\bracks{\prns{\frac{d^{\pi^{*}}_{P^{\star}}(s[pa_i],a)}{\rho(s[pa_i],a)}}^2 }. 
\end{align*}
Following the third step in the proof of Corollary \ref{cor:tabular}, the statement is concluded. 

\paragraph{Proof of \pref{lem:comparison}. }

Next, we show that $ \ddot{C}_{\pi^{*},\infty} \leq  C_{\pi^{*}, P^{\star}} = \max_{s,a}\frac{d^{\pi^{*}}_{P^{\star}}(s,a) }{\rho(s,a)}$.

 From now on, for any $i\in [1,\cdots,d]$, by defining $\Scal'_i$ s.t. $\Scal=\Scal_i \times \Scal'_i$, we prove $$\max_{s_i\in \Scal_i,a\in \Acal} \frac{d^{\pi^{*}}_{P^{\star}}(s_i,a) }{\rho(s_i,a)}\leq \max_{s\in \Scal_i,s'_i\in \Scal_i,a\in \Acal}\frac{d^{\pi^{*}}_{P^{\star}}(s_i,s'_i,a) }{\rho(s_i,s'_i,a)}=   C_{\pi^{*},\infty}. $$

First, for any $s_i\in \Scal_i,a\in \Acal$, we have 
\begin{align}\label{eq:key}
\max_{s'_i}\frac{d^{\pi^{*}}_{P^{\star}}(s_i,s'_i,a) }{\rho(s_i,s'_i,a)}=\max_{s'_i}\frac{d^{\pi^{*}}_{P^{\star}}(s_i,a)d^{\pi^{*}}_{P^{\star}}(s'_i\mid s_i,a) }{\rho(s_i,a)\rho(s'_i\mid s_i,a)}=\frac{d^{\pi^{*}}_{P^{\star}}(s_i,a) }{\rho(s_i,a)}\max_{s'_i}\frac{d^{\pi^{*}}_{P^{\star}}(s'_i\mid s_i,a) }{\rho(s'_i\mid s_i,a)}\geq \frac{d^{\pi^{*}}_{P^{\star}}(s_i,a) }{\rho(s_i,a)}. 
\end{align}
Here, we use 
\begin{align*}
    1\leq \max_{s'_i} \frac{d^{\pi^{*}}_{P^{\star}}(s'_i\mid s_i,a) }{\rho(s'_i\mid s_i,a)}, 
\end{align*}
which is proved by the contradiction argument, that is, if $ 1> \max_{s'_i} \frac{d^{\pi^{*}}_{P^{\star}}(s'_i\mid s_i,a) }{\rho(s'_i\mid s_i,a)}$, both $\rho$ and $d^{\pi^{*}}_{P^{\star}}$ cannot be probability mass functions since we would get
\begin{align*}
    1=\sum_{s'_i} d^{\pi^{*}}_{P^{\star}}(s'_i\mid s_i,a)\leq \max_{s'_i} \prns{\frac{d^{\pi^{*}}_{P^{\star}}(s'_i\mid s_i,a) }{\rho(s'_i\mid s_i,a)}}\sum_{s'_i}\rho(s'_i\mid s_i,a)< \sum_{s'_i}\rho(s'_i\mid s_i,a). 
\end{align*}
Then, by taking the maximum over $s_i\in \Scal_i,a\in \Acal$ for both sides on \pref{eq:key}, we have 
 $$\max_{s_i,a} \frac{d^{\pi^{*}}_{P^{\star}}(s_i,a) }{\rho(s_i,a)}\leq \max_{s_i,s'_i,a}\frac{d^{\pi^{*}}_{P^{\star}}(s_i,s'_i,a) }{\rho(s_1,s'_i,a)}.$$

\paragraph{Proof of \pref{lem:comparison2}. }

By denoting $s_j=s[pa_i]$, we prove for any $i$, 
\begin{align*}
    \E_{(s,a)\sim \rho}\bracks{\prns{ \frac{d^{\pi^{*}}_{P^{\star}}(s_i,a) }{\rho(s_i,a)}}^2}\leq C_{\pi^{*}, 2}. 
\end{align*}
Here, letting $s'_i$ be a value s.t. $s=(s_i,s'_i)$, we have 
\begin{align*}
    C_{\pi^{*}, 2}&= \E_{(s,a)\sim d^{\pi^{*}}_{P^{\star}}}\bracks{\frac{d^{\pi^{*}}_{P^{\star}}(s,a) }{\rho(s,a)}}=\E_{(s_i,s'_i,a)\sim d^{\pi^{*}}_{P^{\star}}}\bracks{\frac{d^{\pi^{*}}_{P^{\star}}(s_i,s'_i,a) }{\rho(s_i,s'_i,a)}}\\
    &= \E_{(s_i,a)\sim d^{\pi^{*}}_{P^{\star}}}\bracks{\E_{s'_i\sim d^{\pi^{*}}_{P^{\star}}(s_i,a) }\bracks{\frac{d^{\pi^{*}}_{P^{\star}}(s_i,s'_i,a) }{\rho(s_i,s'_i,a)}}}\\
   &= \E_{(s_i,a)\sim d^{\pi^{*}}_{P^{\star}}}\bracks{\frac{d^{\pi^{*}}_{P^{\star}}(s_i,a) }{\rho(s_i,a)}\E_{s'_i\sim d^{\pi^{*}}_{P^{\star}}(s_i,a) }\bracks{\frac{d^{\pi^{*}}_{P^{\star}}(s'_i \mid s_i,a) }{\rho(s'_i \mid s_i, a)}}}\\
   &\geq \E_{(s_i,a)\sim d^{\pi^{*}}_{P^{\star}}}\bracks{\frac{d^{\pi^{*}}_{P^{\star}}(s_i,a) }{\rho(s_i,a)}}= \E_{(s,a)\sim \rho}\bracks{\prns{ \frac{d^{\pi^{*}}_{P^{\star}}(s_i,a) }{\rho(s_i,a)}}^2}. 
\end{align*}
In the above inequality, we use 
\begin{align*}
    \E_{s'_i\sim d^{\pi^{*}}_{P^{\star}}(s_i,a) }\bracks{\frac{d^{\pi^{*}}_{P^{\star}}(s'_i \mid s_i,a) }{\rho(s'_i \mid s_i, a)}}-1 \geq 0 ,\forall s_i\in \Scal_i,\forall a\in \Acal 
\end{align*} 
as this is Chi-square divergence between two conditional distributions.

\section{Missing Proofs in \pref{sec:knrs}}

\subsection{Proofs for Finite-Dimensional KNRs (Proof of Corollary \pref{cor:knrs}) }\label{sec:proof_knrs}

We prove in a similar way as \pref{thm:version}. 

\paragraph{First Step.}
Recall 
\begin{align*}
    \xi=  \sqrt{ 2\lambda \|W^\star\|^2_2  + 8 \zeta^2 \left(d_{\Scal} \ln(5) +  \ln(1/\delta) +  \bar \Ical_{n} \right) }, \quad   \bar \Ical_{n}=\ln\left( \det(\Sigma_{n}) / \det(\lambda \Ib) \right).
\end{align*}
Thus, from \pref{lem:mle_knrs}, with probability $1-\delta$, we can show $W^{*}\in \Wcal_{\Dcal}$ since 
\begin{align*}
    \left\| \left(\widehat{W}_{\MLE}  - W^\star\right) \left(\Sigma_{n}\right)^{1/2}  \right\|_2 \leq \xi . 
\end{align*}
Hereafter, we condition on this event.

\paragraph{Second step.} For any $W \in \Wcal_{\Dcal}$, with probability $1-\delta$, we have 
\begin{align*}
    \left\| \left(W - W^\star\right) \left(\Sigma_{n}\right)^{1/2}  \right\|_2 \leq    \left\| \left(W - \widehat W \right) \left(\Sigma_{n}\right)^{1/2}  \right\|_2+\left\| \left(W^{*} - \widehat W \right) \left(\Sigma_{n}\right)^{1/2}  \right\|_2\leq  \xi. 
\end{align*}

\paragraph{Third step.}

Note $P^{\star}=P(W^{*})$. Then, 
\begin{align*}
    V^{\pi^{*}}_{P^{\star}}-V^{\hat \pi}_{P^{\star}}&\leq    V^{\pi^{*}}_{P^{\star}}-\min_{W\in \Wcal_{\Dcal}}V^{\pi^{*}}_{P(W)}+\min_{W\in \Wcal_{\Dcal}}V^{\pi^{*}}_{P(W)}- V^{\hat \pi}_{P^{\star}}\\ 
    &\leq    V^{\pi^{*}}_{P^{\star}}-\min_{W\in \Wcal_{\Dcal}}V^{\pi^{*}}_{P(W)} +\min_{W\in \Wcal_{\Dcal}}V^{\hat \pi}_{P(W)}- V^{\hat \pi}_{P^{\star}} \tag{definition of $\hat \pi$}\\ 
      &\leq  V^{\pi^{*}}_{P^{\star}}-\min_{W\in \Wcal_{\Dcal}}V^{\pi^{*}}_{P(W)} \tag{Fist step, $W^{*}\in \Wcal_{\Dcal}$}. 
\end{align*}
Then, by setting $W'=\argmin_{W\in \Mcal_{\Dcal}}V^{\pi^{*}}_{P(W)}$, we have 
\begin{align*}
       V^{\pi^{*}}_{P^{\star}}-V^{\hat \pi}_{P^{\star}} &\leq  (1-\gamma)^{-2}\EE_{(s,a)\sim d^{\pi^{*}}_{P^{\star}}}[\|P'(s,a)-P^{\star}(s,a)\|_{\mathrm{TV}}]\\
        &\leq \frac{(1-\gamma)^{-2}}{\zeta}\EE_{(s,a)\sim d^{\pi^{*}}_{P^{\star}}}[\left\| (W' - W^\star) \phi(s,a)   \right\|_2 ] \tag{\pref{lem:gaussian_tv}}\\
          &\leq   \frac{(1-\gamma)^{-2}}{\zeta} \EE_{(s,a)\sim d^{\pi^{*}}_{P^{\star}}}\bracks{\left\| (W' - W^\star) (\Sigma_{n})^{1/2}  \right\|_2 \left\| \phi(s,a) \right\|_{\Sigma_{n}^{-1}}} \tag{CS inequality }\\
          &\leq  \frac{(1-\gamma)^{-2}}{\zeta}\xi\EE_{(s,a)\sim d^{\pi^{*}}_{P^{\star}}}[ \left\| \phi(s,a) \right\|_{\Sigma_{n}^{-1}}] \tag{Second step}
\end{align*}

From \citet[Theorem 20]{ChangJonathanD2021MCSi}, with probability $1-\delta$,  we have 
\begin{align*} 
  \xi \leq c_1\sqrt{\|W^{*}\|_2+d_{\Scal}\min(\mathrm{rank}(\Sigma_{\rho})\{\mathrm{rank}(\Sigma_{\rho})+\ln(c_2/\delta)\},d)\ln (1+n) }.
\end{align*}
In addition, from \citet[Theorem 21]{ChangJonathanD2021MCSi}, with probability $1-\delta$, we also have 
\begin{align*}
    \EE_{(s,a)\sim d^{\pi^{*}}_{P^{\star}}}[\|\phi(s,a)\|_{\Sigma^{-1}_{n}}]&\leq c_1\sqrt{\frac{\bar C_{\pi^*,\phi}\mathrm{rank}[\Sigma_{\rho}]\{\mathrm{rank}[\Sigma_{\rho}] +\ln(c_2/\delta)\}}{n}}. 
\end{align*}
Finally, by combining all things, we have 
 \begin{align*}
    V^{\pi^{*}}_{P^{\star}}-V^{\hat \pi}_{P^{\star}} \leq c_1 (1-\gamma)^{-2} \min(d^{1/2},  \bar R)\sqrt{ \bar R }   \sqrt{\frac{d_{\Scal}\bar C_{\pi^*,\phi}\ln (1+n) }{n}}, \bar R=  \mathrm{rank}[\Sigma_{\rho}]\{\mathrm{rank}[\Sigma_{\rho}] +\ln(c_2/\delta)\}. 
\end{align*}

{\newedit \subsection{Proof of \pref{lem:knrs_basic} }
Using $\TV(P(W)(\cdot \mid s,a), P(W^{\star})(\cdot \mid s,a))^2=\Theta(\|(W-W^{\star})\phi(s,a)\|^2_2 ) $ \citep{devroye2018total}, we have 
\begin{align*}
    C^{\dagger}_{\pi^{*}}\leq \sup_{W}\frac{\EE_{(s,a)\sim d^{\pi^{*}}_{P^{\star}} }[\|(W-W^{\star})\phi(s,a)\|^2_2  ] }{\EE_{(s,a)\sim \rho} [\|(W-W^{\star})\phi(s,a)\|^2_2 ] }. 
\end{align*}
Here,  we have
\begin{align*}
    &\EE_{(s,a)\sim d^{\pi^{*}}_{P^{\star}} }[\|(W-W^{\star})\phi(s,a)\|^2_2  ]\\
    &=\tr( (W-W^{\star})^{\top}(W-W^{\star}) \EE_{(s,a)\sim d^{\pi^{*}}_{P^{\star}} }[\phi(s,a)\phi(s,a)^{\top}])  \\ 
    &=\sum_i a_i u^{\top}_i\EE_{(s,a)\sim d^{\pi^{*}}_{P^{\star}} }[\phi(s,a)\phi(s,a)^{\top}]) u_i. 
\end{align*}
In the above derivation, we use SVD: 
\begin{align*}
    (W-W^{\star})^{\top}(W-W^{\star}) = \sum_{i=1}^d a_i u_i u^{\top}_i. 
\end{align*}
Hence, 
\begin{align*}
    C^{\dagger}_{\pi^{*}}\leq d\bar C_{\pi^{*},\phi}. 
\end{align*}
}

{\newedit \subsection{Proofs for Infinite-Dimensional KNRs (Proof of Corollary \ref{cor:gps}) }

We prove in a similar way as \pref{thm:version}. 

\paragraph{First step.}
Recall 
\begin{align*}
   \xi=\sqrt{d_{\Scal}\{2+150 \ln^3(d_{\Scal}n/\delta)\mathcal{I}_{n}\}} ,\quad \mathcal{I}_{n}=\ln(\det(\Ib+\zeta^{-2}\Kb_{n})). 
\end{align*}
From \citet[Leemma 14]{ChangJonathanD2021MCSi}, with probability $1-\delta$, we can show $g^{*}\in \Gcal_{\Dcal}$ since
\begin{align*}
   \sum_{i=1}^{d_{\Scal}} \|\hat g_i-g^{*}_i\|^2_{k_n}\leq \xi^2. 
\end{align*}
Hereafter, we condition on this event. 

\paragraph{Second step.} For any $g\in \Gcal_{\Dcal}$, with probability $1-\delta$, we have 
\begin{align*}
     \sum_{i=1}^{d_{\Scal}}  \|g_i-g^{*}_i \|^2_{k_n}\leq         \sum_{i=1}^{d_{\Scal}}  \|g_i-\hat g_i\|^2_{k_n}+      \sum_{i=1}^{d_{\Scal}} \|g^{*}_i-\hat g_i\|^2_{k_n}\leq 2\xi. 
\end{align*}

\paragraph{Third step.}

Note $P^{\star}=P(g^{*})$. Then, 
\begin{align*}
        V^{\pi^{*}}_{P^{\star}}-V^{\hat \pi}_{P^{\star}}&\leq    V^{\pi^{*}}_{P^{\star}}-\min_{g \in \Gcal_{\Dcal}}V^{\pi^{*}}_{P(g)}+\min_{g\in \Gcal_{\Dcal}}V^{\pi^{*}}_{P(g)}- V^{\hat \pi}_{P^{\star}}\\ 
    &\leq    V^{\pi^{*}}_{P^{\star}}-\min_{g\in \Gcal_{\Dcal}}V^{\pi^{*}}_{P(g)} +\min_{g\in \Wcal_{\Dcal}}V^{\hat \pi}_{P(g)}- V^{\hat \pi}_{P^{\star}} \tag{definition of $\hat \pi$}\\ 
      &\leq  V^{\pi^{*}}_{P^{\star}}-\min_{g\in \Gcal_{\Dcal}}V^{\pi^{*}}_{P(g)} \tag{Fist step, $g^{*}\in \Gcal_{\Dcal}$}. 
\end{align*}
Then, by setting $g'=\argmin_{g\in \Gcal_{\Dcal}}V^{\pi^{*}}_{P(g)}$, we have
\begin{align*}
          V^{\pi^{*}}_{P^{\star}}-V^{\hat \pi}_{P^{\star}} &\leq  (1-\gamma)^{-2}\EE_{(s,a)\sim d^{\pi^{*}}_{P^{\star}}}[\|P'(s,a)-P^{\star}(s,a)\|_{\mathrm{TV}}]\\
        &\leq \frac{(1-\gamma)^{-2}}{\zeta}\EE_{(s,a)\sim d^{\pi^{*}}_{P^{\star}}}[\left\|g'(s,a)-g(s,a) \right\|_2 ] \tag{\pref{lem:gaussian_tv}}\\
        &\leq \frac{(1-\gamma)^{-2}}{\zeta}\EE_{(s,a)\sim d^{\pi^{*}}_{P^{\star}}}\bracks{\sqrt{k_n((s,a),(s,a))}  \prns{\sum_{i=1}^{d_{\Scal}} \left\|g'_i-g_i \right\|^2_{k_n}}^{1/2} } \tag{CS inequality}\\
        &\leq \frac{(1-\gamma)^{-2}\xi}{\zeta}\EE_{(s,a)\sim d^{\pi^{*}}_{P^{\star}}}[\sqrt{k_n((s,a),(s,a))}]. \tag{Second step}
\end{align*}

From \citet[Theorem 24]{ChangJonathanD2021MCSi}, with probability $1-\delta$, we have 
\begin{align*}
    \xi \leq c_1\sqrt{d_{\Scal}\ln^3(c_2d_{\Scal}n/\delta) \{ d^{*}+\ln(c_2/\delta) \} d^{*}\ln (1+n) }. 
\end{align*} 
In addition, from \citet[Theorem 25]{ChangJonathanD2021MCSi}, with probability $1-\delta$, we have 
\begin{align*}
   \EE_{x\sim d^{\pi^{*}}_{P^{\star}}}[\sqrt{k_{n}(x,x)}]\leq c_1\sqrt{\frac{\bar C_{\pi^*,\phi}d^{*}\{d^{*}+\ln(c_2/\delta)\}}{n}}.
\end{align*}
Combining all things together, with probability $1-\delta$, we have 
\begin{align*}
     V^{\pi^{*}}_{P^{\star}}-V^{\hat \pi}_{P^{\star}} &\leq     c_1 (1-\gamma)^{-2} \{d^{*}+\ln(c_2/\delta)\}d^{*}\sqrt{\frac{d_{\Scal}\bar C_{\pi^*,\phi}\ln^3(c_2d_{\Scal}n/\delta)\ln (1+n) }{n}}. 
\end{align*}

}

{\newedit 
\section{Missing Proofs in \pref{sec:linear_mdps}}

The proof consists of three steps. 

\paragraph{First step (pessimism).} 
We set $\lambda=1$. 
Using a result in \citet[Lemma 8.7]{agarwal2019reinforcement}, with probability $1-\delta$, 
\begin{align*}
    \left|\int \{\hat P(s' \mid s,a)-P^{\star}(s' \mid s,a)\}v(s')d(s') \right|\lesssim (1-\gamma)^{-1}  \|\phi(s,a)\|_{\Sigma^{-1}_n}\sqrt{d\ln(n |\Pi|W/\delta) }
\end{align*}
for all $(s,a)$ and $v \in \Vcal$. Hereafter, we condition on this event. Then, 
\begin{align*}
    \EE_{\Dcal}\left[\left|\int \{\hat P(s' \mid s,a)-P^{\star}(s' \mid s,a)\}v(s')\rd \iota(s') \right|^2\right] &\lesssim (1-\gamma)^{-2} \EE_{\Dcal}\left[\|\phi(s,a)\|^2_{\Sigma^{-1}_n}\right] d\ln(n |\Pi| W/\delta)\\
     &\lesssim (1-\gamma)^{-2} \frac{d^2\ln(n |\Pi|W/\delta)}{n}. 
\end{align*}
for any $v \in \Vcal$.

\paragraph{Second step.}

From the construction of the algorithm, for any $P \in \Mcal_{\Dcal}$, we have 
\begin{align*}
     \EE_{\Dcal}\left[\left|\int \{P(s' \mid s,a)-P^{\star}(s' \mid s,a)v(s')\rd \iota(s') \right|^2\right]^{1/2} \lesssim (1-\gamma)^{-1} \sqrt{\frac{d^2\ln(n |\Pi|W/\delta)}{n} }. 
\end{align*}

\paragraph{Third step: distribution shift part.}

Here, for any $P \in \Mcal_{\Dcal}$, we will prove
\begin{align}\label{eq:main_linear}
  V^{\pi^{*}}_{P^{\star}}-       V^{\pi^{*}}_{P} \leq c_1(1-\gamma)^{-2}  \sqrt{\frac{\bar C_{\pi^*,\phi}\mathrm{rank}[\Sigma_{\rho}]^2d \ln(c_2n|\Pi|W/\delta)\}}{n}}. 
\end{align}
Following the third step of the proof of  \pref{thm:version}, this immediately concludes the bound:
\begin{align*}
      V^{\pi^{*}}_{P^{\star}}-       V^{\hat \pi}_{P^{\star}} \lesssim (1-\gamma)^{-2}  \sqrt{\frac{\bar C_{\pi^*,\phi}\mathrm{rank}[\Sigma_{\rho}]^2d \ln(c_2n|\Pi|W/\delta)\}}{n}}. 
\end{align*}

From now on, we focus on the proof of \pref{eq:main_linear}. Here, we have 
\begin{align*}
  V^{\pi^{*}}_{P^{\star}}-       V^{\pi^{*}}_{P} & \leq (1-\gamma)^{-1}\left|\EE_{(s,a)\sim d^{\pi^{*}}_{P^{\star}}}\bracks{\int \{P(s' \mid s,a) - P^\star(s' \mid s,a)\}V^{\pi^{*}}_{P}(s')\rd \iota(s')  }\right| \tag{Simulation lemma, \pref{lem:simulation}}\\
& \leq (1-\gamma)^{-1}\left|\left\langle \EE_{(s,a)\sim d^{\pi^{*}}_{P^{\star}}}\bracks{\phi(s,a)}, \int \{\mu(s') - \mu^{\star}(s')\}V^{\pi^{*}}_{P}(s')\rd \iota(s')  \right \rangle \right|\\
& \leq (1-\gamma)^{-1}\EE_{(s,a)\sim d^{\pi^{*}}_{P^{\star}}}\bracks{\|\phi(s,a)\|_{\Sigma^{-1}_n}}\times  \left\|\int \{\mu(s') - \mu^{\star}(s')\}V^{\pi^{*}}_{P}(s')\rd \iota(s') \right\|_{\Sigma_n}. \tag{CS inequality}
\end{align*}
Recall from \citet[Theorem 21]{ChangJonathanD2021MCSi}, with probability $1-\delta$, we also have 
\begin{align*}
    \EE_{(s,a)\sim d^{\pi^{*}}_{P^{\star}}}[\|\phi(s,a)\|_{\Sigma^{-1}_{n}}]&\leq c_1\sqrt{\frac{\bar C_{\pi^*,\phi}\mathrm{rank}[\Sigma_{\rho}]\{\mathrm{rank}[\Sigma_{\rho}] +\ln(c_2/\delta)\}}{n}}. 
\end{align*}
Furthermore, 
\begin{align*}
   & \left\|\int \{\mu(s') - \mu^{\star}(s')\}V^{\pi^{*}}_{P}(s')\rd \iota(s') \right\|_{\Sigma_n} \\
& \leq  2\sqrt{n}\EE_{\Dcal}\left[\left|\int \{P(s' \mid s,a)-P^{\star}(s' \mid s,a)\}V^{\pi^{*}}_{P}(s')\rd \iota(s') \right|^2\right]^{1/2} + 2 \lambda \left \|\int \mu(s')V^{\pi^{*}}_{P}(s')\rd \iota(s')\right \|_2 \\
&\lesssim 2\sqrt{n}\times (1-\gamma)^{-1} \sqrt{\frac{d^2\ln(n |\Pi|W/\delta)}{n} } + 2\lambda (1-\gamma)^{-1} \|\int \mu(s')\rd\iota(s')\|_2 \\ 
&\lesssim (1-\gamma)^{-1}\sqrt{d^2\ln(n|\Pi|W/\delta)}.
\end{align*}

Hence, 
\begin{align*}
  V^{\pi^{*}}_{P^{\star}}-       V^{\pi^{*}}_{P} \leq c_1(1-\gamma)^{-2}  \sqrt{\frac{\bar C_{\pi^*,\phi}\mathrm{rank}[\Sigma_{\rho}]^2d \ln(c_2n|\Pi| W /\delta)\}}{n}}. 
\end{align*}
}

\section{Missing Proofs in \pref{sec:alg} } 

{\newedit 

Here, $P^{\star}\sim \beta(\cdot),\Dcal \sim  d(\cdot \mid P^{\star})$. Then, by denoting the posterior distribution of $P^{\star}$ given $\Dcal$ as $\beta'(\cdot\mid \cdot)$, then $P_t\sim \beta(\cdot \mid \Dcal)$. %

We start with the proof of \pref{lem:regret}. 

\subsection{Proof of \pref{lem:regret}}
\begin{align*}
  \EE\bracks{  V^{\pi(P^{\star})}_{P^\star}   -  V^{\pi_t}_{P^\star}}&=\EE\bracks{ V^{\pi(P^{\star})}_{P^\star} -L(\pi(P^{\star});\Dcal)+  L(\pi(P^{\star});\Dcal)- V^{\pi_t}_{P^\star} } \\ 
  &=\EE\bracks{ V^{\pi(P^{\star})}_{P^\star} -L(\pi(P^{\star});\Dcal)+  \EE[L(\pi(P^{\star});\Dcal)\mid \Dcal]- V^{\pi_t}_{P^\star} } \\ 
    &=\EE\bracks{ V^{\pi(P^{\star})}_{P^\star} -L(\pi(P^{\star});\Dcal)+  \EE[L(\pi(P_t);\Dcal)\mid \Dcal]- V^{\pi_t}_{P^\star} } \\ 
  &= \EE\bracks{ V^{\pi(P^{\star})}_{P^\star} -L(\pi(P^{\star});\Dcal)+  L(\pi(P_t);\Dcal)- V^{\pi_t}_{P^\star} }. 
\end{align*}
From the second line to the third line, we use $\mathrm{Pr}(P^{\star}\mid \Dcal)=\beta'(P^{\star}\mid \Dcal)$, $\mathrm{Pr}(P_t\mid \Dcal)=\beta'(P_t\mid \Dcal)$.

Besides, by denoting the event $\Zcal=\{ L(\pi;\Dcal)\leq V^{\pi}_{P^{\star}},\forall \pi\in \Pi\}$ from the assumption, 
\begin{align*}
   \EE\bracks{L(\pi(P_t);\Dcal)}&=  \EE[\EE\bracks{L(\pi(P_t);\Dcal)\mid \Zcal,P^{\star}}P(\Zcal|P^{\star})]+ \EE[2(1-\gamma)^{-1}(1-P(\Zcal|P^{\star}))]\\
   &\leq \EE\bracks{\EE\bracks{V_{P^{\star}}^{\pi(P_t)} \mid \Zcal,P^{\star}} }+ 2(1-\gamma)^{-1}\delta \\
   &\leq \EE\bracks{\EE\bracks{V_{P^{\star}}^{\pi(P^{\star})} \mid \Zcal,P^{\star}} }+2(1-\gamma)^{-1}\delta\\
   &= \EE\bracks{V_{P^{\star}}^{\pi(P^{\star})} \EE\bracks{1 \mid \Zcal,P^{\star}} }+2(1-\gamma)^{-1}\delta=\EE{\bracks{V_{P^{\star}}^{\pi(P^{\star})}} }+2(1-\gamma)^{-1}\delta. 
\end{align*}
Thus, 
\begin{align*}
      \EE\bracks{  V^{\pi(P^{\star})}_{P^\star}   -  V^{\pi_t}_{P^\star}}& \leq \EE\bracks{ V^{\pi(P^{\star})}_{P^\star} -L(\pi(P^{\star});\Dcal)+V^{\pi(P^{\star})}_{P^\star}- V_{P^{\star}}^{\pi_t} }+2(1-\gamma)^{-1}\delta. 
\end{align*}

Next, we prove  \pref{lem:npg}. 

\subsection{Proof of \pref{lem:npg}}

\begin{align*}
\EE\bracks{V_{P^{\star}}^{\pi(P^{\star})} - V^{\pi_t}_{P^\star} }=H\EE_{(s,a) \sim d_{P^{\star}}^{\pi(P^{\star})} }[A^{\pi_t}_{P^{\star}}(s,a)] \tag{Performance difference lemma}. 
\end{align*}

Here, we have 
\begin{align*}
    &\EE_{(s) \sim d_{P^{\star}}^{\pi(P^{\star})} }\bracks{\frac{\KL( \pi(P^{\star})(\cdot \mid s),\pi_{t+1}(\cdot \mid s))-\KL(\pi(P^{\star})(\cdot \mid s),\pi_t(\cdot \mid s))}{\eta}}\\
    &= \EE_{(s) \sim d_{P^{\star}}^{\pi(P^{\star})},a\sim \pi(P^{\star})(s) }\bracks{\ln\frac{\pi^{t+1}(a\mid s)}{\pi_t(a\mid s)} }\\ 
    &=\EE_{(s) \sim d_{P^{\star}}^{\pi(P^{\star})},a\sim \pi(P^{\star})(s) }\bracks{A^{\pi_t}_{P_t}(s,a)-\frac{1}{\eta}\ln\EE_{a\sim \pi^{t}}[\exp(\eta  A^{\pi_t}_{P_t}(s,a)] } \\ 
    &\geq \EE_{(s) \sim d_{P^{\star}}^{\pi(P^{\star})},a\sim \pi(P^{\star})(s) }\bracks{A^{\pi_t}_{P_t}(s,a)}-4\eta(1-\gamma)^{-2} . 
\end{align*}
From the second line to third line, we use the following 
\begin{align*}
    \ln\EE_{a\sim \pi^{t}(s)}[\exp(\eta  A^{\pi_t}_{P_t}(s,a)]& \leq  \ln\braces{\EE_{a\sim \pi^{t}(s)}[1+\eta  A^{\pi_t}_{P_t}(s,a)+\{\eta  A^{\pi_t}_{P_t}(s,a)\}^2] } \tag{$\eta\leq 2(1-\gamma),\exp(x)\leq 1+x+x^2 (x\leq 1)$}\\
    & \leq \ln\braces{\EE_{a\sim \pi^{t}(s)}[1+4\eta^2(1-\gamma)^{-2} ] } \tag{$\log(1+x)\leq x$}\\
    &\leq 4\eta^2(1-\gamma)^{-2} . 
\end{align*}
   
Then, 
\begin{align*}
    &\EE\bracks{V_{P^{\star}}^{\pi(P^{\star})} - V^{\pi_t}_{P^\star} } \\
    &=H\EE[\EE_{(s,a) \sim d_{P^{\star}}^{\pi(P^{\star})} }[\EE[A^{\pi_t}_{P_{t}}(s,a) ]]\\
      &\leq H\EE\bracks{\EE_{(s) \sim d_{P^{\star}}^{\pi(P^{\star})} }\bracks{4\eta(1-\gamma)^{-2} +\frac{\KL( \pi(P^{\star})(\cdot \mid s),\pi_{t+1}(\cdot \mid s))-\KL(\pi(P^{\star})(\cdot \mid s),\pi_t(\cdot \mid s))}{\eta}} }. 
\end{align*}

\subsection{Proof of \pref{thm:bayesian_pspo}}

Finally, 
\begin{align*}
     & \min_{t\leq T}\EE\bracks{  V^{\pi(P^{\star})}_{P^\star}   -  V^{\pi_t}_{P^\star}}\\
     &\leq \braces{\EE[V^{\pi(P^{\star})}_{P^\star} -L(\pi(P^{\star});\Dcal)  ] }+\frac{1}{T}\sum_{t=1}^T     \EE\bracks{V_{P^{\star}}^{\pi(P^{\star})} - V^{\pi_t}_{P^\star} }+2(1-\gamma)^{-1}\delta\\
     &\leq \braces{\EE[V^{\pi(P^{\star})}_{P^\star} -L(\pi(P^{\star});\Dcal)  ] }\\
     &+H\braces{4\eta(1-\gamma)^{-2}+\frac{1}{T}\EE\bracks{\EE_{(s) \sim d_{P^{\star}}^{\pi(P^{\star})} }\bracks{\frac{\KL( \pi(P^{\star}(\cdot \mid s)),\pi_{T+1}(\cdot \mid s))-\KL(\pi(P^{\star})(\cdot \mid s),\pi_1(\cdot \mid s))}{\eta}}}}\\
     &+2(1-\gamma)^{-1}\delta  \\
     &\leq \braces{\EE[V^{\pi(P^{\star})}_{P^\star} -L(\pi(P^{\star});\Dcal)  ] }+H\braces{4\eta(1-\gamma)^{-2}+\frac{\ln |\Acal|}{\eta T} }+2(1-\gamma)^{-1}\delta\\
     & \leq \braces{\EE[V^{\pi(P^{\star})}_{P^\star} -L(\pi(P^{\star});\Dcal)  ] }+4(1-\gamma)^{-2}\sqrt{\frac{\ln |\Acal|}{T}}+2(1-\gamma)^{-1}\delta. 
\end{align*}
Thus, 
\begin{align*}
 \min_{t\leq T}\EE\bracks{  V^{\pi(P^{\star})}_{P^\star}   -  V^{\pi_t}_{P^\star}} \leq \braces{\min_{L\in \Lcal_{\Dcal}}\EE[V^{\pi(P^{\star})}_{P^\star} -L(\pi(P^{\star});\Dcal)  ] }+4(1-\gamma)^{-2}\sqrt{\frac{\ln |\Acal|}{T}}+2(1-\gamma)^{-1}\delta. 
\end{align*}

\subsection{Proof of Corollary \ref{cor:discrete_bayes}}

We take $L(\pi;\Dcal)$ as $\max(\min(\min_{M\in \Mcal_{\Dcal}}V^{\pi}_{M},H),0)$ in \pref{thm:version}. Then, from the first step in the proof of \pref{thm:version}, conditional on $P^{\star}$, with probability $1-\delta$, $L(\pi;\Dcal)$ satisfies 
\begin{align*}
    L(\pi;\Dcal)\leq V^{\pi}_{P^{\star}},\forall \pi \in \Pi.  
\end{align*}
and 
\begin{align*}
    \EE_{(s,a)\sim \rho}[\TV(P'(\cdot \mid s,a), P^\star(\cdot \mid s,a) )^2 ]\lesssim \xi. 
\end{align*}
where $P'=\argmin_{M\in \Mcal_{\Dcal}}V^{\pi(P^{\star})}_{M}$. We denote the above event as $\Zcal$. We have $\mathrm{P}(\Zcal\mid P^{\star})\geq 1-\delta$. 
In addition, from the third step in the proof of \pref{thm:version},  

\begin{align*}
   & \EE[V^{\pi(P^{\star})}_{P^\star} -L(\pi(P^{\star});\Dcal)  ]\leq  \tag{The third step in the proof of \pref{thm:version}} \EE[V^{\pi(P^{\star})}_{P^{\star}}-V^{\pi(P^{\star})}_{P'}] \\
       & \leq (1-\gamma)^{-2}\EE_{P^{\star}\sim \beta}[\EE_{(s,a)\sim d^{\pi(P^{\star})}_{P^{\star}}}[\TV(P'(\cdot \mid s,a), P^\star(\cdot \mid s,a) ) ]] \tag{Simulation lemma, \pref{lem:simulation}}\\
        &\leq (1-\gamma)^{-2}\EE_{P^{\star}\sim \beta}\bracks{\EE\bracks{\EE_{(s,a)\sim d^{\pi(P^{\star})}_{P^{\star}}}[\TV(P'(\cdot \mid s,a), P^\star(\cdot \mid s,a) ) ]\mid \Zcal,P^{\star}}P(\Zcal\mid P^{\star})}\\
        &+\EE_{P^{\star}\sim \beta}[2(1-\gamma)^{-1}(1-P(\Zcal\mid P^{\star}) )]\\
       &\leq (1-\gamma)^{-2}\EE_{P^{\star}\sim \beta}\bracks{\EE\bracks{\EE_{(s,a)\sim d^{\pi(P^{\star})}_{P^{\star}}}[\TV(P'(\cdot \mid s,a), P^\star(\cdot \mid s,a) ) ]\mid \Zcal,P^{\star}}}+2(1-\gamma)^{-1}\delta \\
         &\leq (1-\gamma)^{-2}\EE_{P^{\star}\sim \beta}\bracks{\EE\bracks{\sqrt{\EE_{(s,a)\sim d^{\pi(P^{\star})}_{P^{\star}}}[\TV(P'(\cdot \mid s,a), P^\star(\cdot \mid s,a) )^2 ]}\mid \Zcal,P^{\star}}}+2(1-\gamma)^{-1}\delta \\
        &\leq (1-\gamma)^{-2}\EE_{P^{\star}\sim \beta}\bracks{\EE\bracks{\sqrt{C^{\dagger}_{\pi(P^{\star}),P^{\star}}\EE_{(s,a)\sim \rho}[\TV(P'(\cdot \mid s,a), P^\star(\cdot \mid s,a) )^2 ]}\mid \Zcal,P^{\star}}}+2(1-\gamma)^{-1}\delta \\
 &\leq (1-\gamma)^{-2}\sqrt{\xi}\EE_{P^{\star}\sim \beta}\bracks{\EE\bracks{\sqrt{C^{\dagger}_{\pi(P^{\star}),P^{\star}}}\mid \Zcal,P^{\star}}}+2(1-\gamma)^{-1}\delta \\
  &\leq (1-\gamma)^{-2}\sqrt{\xi}\EE_{P^{\star}\sim \beta}\bracks{\sqrt{C^{\dagger}_{\pi(P^{\star}),P^{\star}}}\EE\bracks{1\mid \Zcal,P^{\star}}}+2(1-\gamma)^{-1}\delta\\
  &= (1-\gamma)^{-2}\sqrt{\xi}\sqrt{\EE_{P^{\star}\sim \beta}\bracks{C^{\dagger}_{\pi(P^{\star}),P^{\star}}}}+2(1-\gamma)^{-1}\delta. 
\end{align*}
By taking $\delta=1/n$, the statement is concluded. 

\subsection{Proof of Corollary  \ref{cor:tablar_bayes} }
The proof is done as in the proof of Corollary \ref{cor:discrete_bayes}. We omit the proof.

\subsection{Proof of Corollary  \ref{cor:linear_mixture_bayes} }
The proof is done as in the proof of Corollary \ref{cor:discrete_bayes}. We omit the proof.

\subsection{Proof of Corollary \ref{cor:linear_bayes}}
We take  $L(\pi;\Dcal)$ as $\max(\min(\min_{W\in \Wcal_{\Dcal}}V^{\pi}_{P(W)},H),0)$ in \pref{thm:version}. Then, from the first and second step in the proof of Corollary \ref{cor:knrs}, conditioning on $P^{\star}$, with probability $1-\delta$, $L(\pi;\Dcal)$ satisfies 
\begin{align*}
      L(\pi;\Dcal)\leq V^{\pi}_{P^{\star}},\forall \pi \in \Pi. 
\end{align*}
and  $$ \left\| \left(W' - W^\star\right) \left(\Sigma_{n}\right)^{1/2}  \right\|_2\leq \xi, $$
where $W'=\argmin_{W\in \Mcal_{\Dcal}}V^{\pi(P^{\star})}_{P(W)}$. Besides, from \citet[Theoreem 20]{ChangJonathanD2021MCSi}, with probability $1-\delta$,  we have
\begin{align*} 
  \xi \leq c_1\sqrt{1+d_{\Scal}\min(\mathrm{rank}(\Sigma_{\rho})\{\mathrm{rank}(\Sigma_{\rho})+\ln(c_2/\delta)\},d)\ln (1+n) }.
\end{align*}
In addition, from \citet[Theoreem 21]{ChangJonathanD2021MCSi}, with probability $1-\delta$, we also have 
\begin{align*}
    \EE_{(s,a)\sim \rho}[\|\phi(s,a)\|_{\Sigma^{-1}_{n}}]&\leq c_1\sqrt{\frac{\mathrm{rank}[\Sigma_{\rho}]\{\mathrm{rank}[\Sigma_{\rho}] +\ln(c_2/\delta)\}}{n}}. 
\end{align*}
We denote the above event as $\Zcal$. Then, we have $P(\Zcal \mid P^{\star})\geq 1-\delta$.

From the third step in the proof of Corollary \ref{cor:knrs},
\begin{align*}
      &\EE[V^{\pi(P^{\star})}_{P^\star} -L(\pi(P^{\star});\Dcal)  ]\leq \EE[V^{\pi(P^{\star})}_{P^{\star}}-V^{\pi(P^{\star})}_{P(W')}]  \\ 
      &\leq \frac{(1-\gamma)^{-2}}{\zeta} \EE_{P^{\star}\sim \beta}\bracks{\EE_{(s,a)\sim d^{\pi(P^{\star})}_{P^{\star}}}[\|(W' - W^\star) \phi(s,a)  \|_2]} \tag{Simulation lemma}\\ 
       &\leq  \frac{(1-\gamma)^{-2}}{\zeta} \EE_{P^{\star}\sim \beta}\bracks{\EE\bracks{\EE_{(s,a)\sim d^{\pi(P^{\star})}_{P^{\star}}}[\|(W'  - W^\star) \phi(s,a)  \|_2] \mid P^{\star},\Zcal}}+2(1-\gamma)^{-1}\delta \\ 
         &\leq  \frac{(1-\gamma)^{-2}}{\zeta} \EE_{P^{\star}\sim \beta}\bracks{\EE\bracks{\xi \EE_{(s,a)\sim d^{\pi(P^{\star})}_{P^{\star}}}[\|\phi(s,a)\|_{\Sigma^{-1}_n}] \mid P^{\star},\Zcal}}+2(1-\gamma)^{-1}\delta \\  
           &\leq  \frac{(1-\gamma)^{-2}}{\zeta} \EE_{P^{\star}\sim \beta}\bracks{\EE\bracks{\xi \sqrt{\EE_{(s,a)\sim d^{\pi(P^{\star})}_{P^{\star}}}[\|\phi(s,a)\|^2_{\Sigma^{-1}_n}]} \mid P^{\star},\Zcal}}+2(1-\gamma)^{-1}\delta.
\end{align*}           
Then, by letting $ \bar R=  \mathrm{rank}[\Sigma_{\rho}]\{\mathrm{rank}[\Sigma_{\rho}] +\ln(c_2/\delta)\}$, 
\begin{align*}
             &\EE[ V^{\pi(P^{\star})}_{P(W')}-V^{\pi(P^{\star})}_{P^{\star}}]  \\
             &\leq  \frac{(1-\gamma)^{-2}}{\zeta} \EE_{P^{\star}\sim \beta}\bracks{\EE\bracks{\xi\sqrt{\bar C_{\pi(P^{\star} ),P^{\star}  } }\sqrt{\EE_{(s,a)\sim \rho}[\|\phi(s,a)\|^2_{\Sigma^{-1}_n}]} \mid P^{\star},\Zcal}}+2(1-\gamma)^{-1}\delta \\ 
        &\leq  c_1 (1-\gamma)^{-2} \min(d^{1/2},  \bar R)\sqrt{ \bar R }   \sqrt{\frac{d_{\Scal} \ln (1+n) }{n}}\EE_{P^{\star}\sim \beta}\bracks{\EE\bracks{\sqrt{\bar C_{\pi(P^{\star} ),P^{\star}  } }|\Zcal,P^{\star}} }+2(1-\gamma)^{-1}\delta \\ 
     &\leq  c_1 (1-\gamma)^{-2} \min(d^{1/2},  \bar R)\sqrt{ \bar R }   \sqrt{\frac{d_{\Scal} \ln (1+n) }{n}}\EE_{P^{\star}\sim \beta}\bracks{\sqrt{\bar C_{\pi(P^{\star} ),P^{\star}  } }\EE\bracks{1 |\Zcal,P^{\star}} }+2(1-\gamma)^{-1}\delta \\ 
                &\leq  c_1 (1-\gamma)^{-2} \min(d^{1/2},  \bar R)\sqrt{ \bar R }   \sqrt{\frac{d_{\Scal} \ln (1+n) }{n}}\sqrt{\EE_{P^{\star}\sim \beta}\bracks{\bar C_{\pi(P^{\star} ),P^{\star}  } } }+2(1-\gamma)^{-1}\delta . 
\end{align*}
By taking $\delta=1/n$, the statement is concluded.

\subsection{Proof of Corollary \ref{cor:low_rank_bayes}}

We take $L(\pi;\Dcal)$ as $\max(\min(\min_{M\in \Mcal_{\Dcal}}V^{\pi}_{M},H),0)$ in \pref{thm:version}. Then, from the first step in the proof of \pref{thm:version}, conditional on $P^{\star}$, with probability $1-\delta$, $L(\pi;\Dcal)$ satisfies 
\begin{align*}
    L(\pi;\Dcal)\leq V^{\pi}_{P^{\star}},\forall \pi \in \Pi.  
\end{align*}
and 
\begin{align*}
    \EE_{(s,a)\sim \rho}[\TV(P'(\cdot \mid s,a), P^\star(\cdot \mid s,a) )^2 ]\lesssim \zeta. 
\end{align*}
where $P'=\argmin_{M\in \Mcal_{\Dcal}}V^{\pi(P^{\star})}_{M}$. We denote the above event as $\Zcal$. We have $\mathrm{P}(\Zcal\mid P^{\star})\geq 1-\delta$. 
In addition, from the third step in the proof of \pref{thm:version},  

\begin{align*}
      & \EE[V^{\pi(P^{\star})}_{P^\star} -L(\pi(P^{\star});\Dcal)  ]\leq  \tag{The third step in the proof of \pref{thm:version}} \EE[V^{\pi(P^{\star})}_{P^{\star}}-V^{\pi(P^{\star})}_{P'}] \\
       & \leq (1-\gamma)^{-2}\EE_{P^{\star}\sim \beta}[\EE_{(s,a)\sim d^{\pi(P^{\star})}_{P^{\star}}}[\TV(P'(\cdot \mid s,a), P^\star(\cdot \mid s,a) ) ]] \tag{Simulation lemma, \pref{lem:simulation}}\\
        &\leq (1-\gamma)^{-2}\EE_{P^{\star}\sim \beta}\bracks{\EE\bracks{\EE_{(s,a)\sim d^{\pi(P^{\star})}_{P^{\star}}}[\TV(P'(\cdot \mid s,a), P^\star(\cdot \mid s,a) ) ]\mid \Zcal,P^{\star}}P(\Zcal\mid P^{\star})}\\
        &+\EE_{P^{\star}\sim \beta}[2(1-\gamma)^{-1}(1-P(\Zcal\mid P^{\star}) )]\\
       &\leq (1-\gamma)^{-2}\EE_{P^{\star}\sim \beta}\bracks{\EE\bracks{\EE_{(s,a)\sim d^{\pi(P^{\star})}_{P^{\star}}}[\TV(P'(\cdot \mid s,a), P^\star(\cdot \mid s,a) ) ]\mid \Zcal,P^{\star}}}+2(1-\gamma)^{-1}\delta.
\end{align*}
The final statement is immediately concluded. 
}

\section{Auxiliary Lemmas}\label{sec:auxi}

\begin{lemma}[Simulation Lemma]\label{lem:simulation}
Consider any two transitions $P$ and $\widehat{P}$, and any policy $\pi:\Scal\to \Delta(\Acal)$.  We have:

\resizebox{\textwidth}{!}{
\begin{minipage}{\textwidth}
\begin{align*}
    |V^{\pi}_{P} - V^{\pi}_{\widehat{P}}|&\leq  |(1-\gamma)^{-1}\EE_{s,a\sim d^{\pi}_{P}}[\E_{s'\sim P(s,a)}[ V^{\pi}_{\widehat{P}}(s')] -\E_{s'\sim P(s,a)}[ V^{\pi}_{\widehat{P}}(s')]    ] |\\ 
    & \leq (1-\gamma)^{-2}\EE_{s,a\sim d^{\pi}_{P}} \left[   \TV( P(\cdot | s,a),\widehat{P}(\cdot | s,a))   \right] .   
\end{align*}\end{minipage}}
\end{lemma}
\begin{proof}
Such simulation lemma is standard in model-based RL literature and the derivation can be found, for instance, in the proof of Lemma 10 from \cite{Sun2019_model}. 
\end{proof}

\begin{lemma}[MLE guarantee]\label{lem:mle}
Given a set of models $\Mcal = \{ P: \Scal\times\Acal\to \Delta(\Scal)\}$ with $P^\star \in \Mcal$, and a dataset $\Dcal = \{s_i,a_i, s'_i\}_{i=1}^n$ with $s_i,a_i \sim \rho$, and $s'_i\sim P^\star(s_i,a_i)$, let  $\widehat{P}_{\MLE}$ be 
$$ \widehat{P}_{\MLE} = \argmin_{P\in\Mcal} \sum_{i=1}^n - \ln P(s_i' | s_i,a_i). $$ With probability at least $1-\delta$, we have:
\begin{align*}
\mathbb{E}_{s,a\sim \rho} \TV( \widehat{P}_{\MLE}(\cdot |s,a) ,P^\star(\cdot | s,a))^2 \lesssim \frac{ \ln(|\Mcal| / \delta) }{n}.
\end{align*}
\end{lemma}
\begin{proof}
Refer to \citep[Section E]{Agarwal2020_flambe}
\end{proof}

\begin{lemma}[MLE guarantee for tabular models]\label{lem:tabular_mle}
\begin{align*}
    \EE_{\Dcal}\bracks{    \TV( P(\cdot | s,a),\widehat{P}_{\MLE}(\cdot | s,a))^2 }\leq \frac{|\Scal|\Acal|\{|\Scal|\ln 2+\ln (2|\Scal||\Acal|/\delta)\}}{2n} .  
\end{align*}
\end{lemma}

\begin{proof}
From \citet[Lemma 12]{ChangJonathanD2021MCSi} , with probability $1-\delta$, 
\begin{align*}
        \TV( P(\cdot | s,a),\widehat{P}_{\MLE}(\cdot | s,a))^2\leq   \frac{|\Scal|\ln 2+\ln (2|\Scal||\Acal|/\delta)}{2N(s,a) }\quad \forall (s,a)\in \Scal\times \Acal, 
\end{align*}
where $N(s,a)$ is the number of visiting times for $(s,a)$. 
Then, 
\begin{align*}
    \EE_{\Dcal}\bracks{    \TV( P(\cdot | s,a),\widehat{P}_{\MLE}(\cdot | s,a))^2 } & \leq  \EE_{\Dcal}\bracks{\frac{|\Scal|\ln 2+\ln (2|\Scal||\Acal|/\delta)}{2N(s,a)} } \\
      &\leq \sum_{(s,a)}\bracks{\frac{|\Scal|\ln 2+\ln (2|\Scal||\Acal|/\delta)}{2n} }\\
      &=  \frac{|\Scal|\Acal|\{|\Scal|\ln 2+\ln (2|\Scal||\Acal|/\delta)\}}{2n} . 
\end{align*}

\end{proof}

\begin{lemma}[MLE guarantee for KNRs]\label{lem:mle_knrs}
\begin{align*}
        \left\| \left(\widehat{W}_{\MLE}  - W^\star\right) \left(\Sigma_{n}\right)^{1/2}  \right\|_2 \leq \beta_{n}. 
\end{align*}
\end{lemma}

\begin{proof} The proof directly follows the confidence ball construction and proof from \citep{Kakade2020}.
\end{proof}

\begin{lemma}[$\ell_1$ Distance between two Gaussians] Consider two Gaussian distributions $P_1 := \Ncal(\mu_1, \zeta^2 \Ib)$ and $P_2 := \Ncal(\mu_2, \zeta^2 \Ib)$. We have:
\begin{align*}
\TV(P_1, P_2)\leq \frac{1}{\zeta} \left\| \mu_1 - \mu_2 \right\|_2.
\end{align*}\label{lem:gaussian_tv}
\end{lemma}
\begin{proof}
This lemma is proved by Pinsker's inequality and the closed-form of the KL divergence between $P_1$ and $P_2$. Refer to  \citep{Kakade2020}. 
\end{proof}

\begin{lemma}[Property of linear mixture MDPs]\label{lem:mixture}
Let $P(\theta)=\theta^{\top}\psi(s,a,s')$. Suppose $P(\theta)\in \Scal\times \Acal \to \Delta(\Scal)$. For any function $V\in \Scal \to [0,1]$, letting $\psi_{V}(s,a)=\int \psi(s,a,s')V(s')\rd(s')$, we suppose $\|\psi_{V}(s,a)\|_2\leq 1$. The following theorems hold: 
\begin{enumerate}
    \item For any $(s,a,s')$, we have $|P(\theta)(s,a,s')-P(\theta')(s,a,s')|\leq \|\theta-\theta'\|_2$.
    \item For any $(s,a)$, we have  $\mathrm{TV}(P(\theta)(s,a,\cdot),P(\theta')(s,a,\cdot))\leq \|\theta-\theta'\|_2$. Besides, for any $V:\Scal \to [0,1]$, we have 
    \begin{align*}
        |(\theta-\theta')\psi_{V}(s,a)|\leq \mathrm{TV}(P(\theta)(s,a,\cdot),P(\theta')(s,a,\cdot)). 
    \end{align*}
    \item \begin{align*}
        C^{\dagger}_{\pi^{\star},P^{\star}}&= \sup_{x}\frac{ x^{\top}\E_{(s,a)\sim d^{\pi^{*}}_{P^{\star}}}[\psi_{V_{(s,a,x)}}(s,a)\psi^{\top}_{V_{(s,a,x)}}(s,a) ] x}{x^{\top}\E_{(s,a)\sim \rho}[\psi_{V_{(s,a,x)}}(s,a)\psi^{\top}_{V_{(s,a,x)}}(s,a) ]x},\\
         V_{(s,a,x)} &=\argmax_{V:\Scal \to [0,1]}\left |x^{\top}\int \phi(s,a,s')V(s')\rd(s')\right|. 
    \end{align*}
    \item In linear MDPs (i.e., $\psi(s,a,s')=\phi(s,a) \otimes \mu(s')$), we have 
    \begin{align*}
   \sup_{V\in \{\Scal \to [0,1]\}}\sup_{x}\frac{ x^{\top}\E_{(s,a)\sim d^{\pi^{*}}_{P^{\star}}}[\psi_{V}(s,a)\psi^{\top}_{V}(s,a) ] x}{x^{\top}\E_{(s,a)\sim \rho}[\psi_{V}(s,a)\psi^{\top}_{V}(s,a) ] x}=    \sup_{x}\frac{x^{\top } \E_{d^{\pi^{*}}_{P^{\star}}}[\phi(s,a)\phi(s,a)^{\top}] x}{x^{\top} \E_{\rho}[\phi(s,a)\phi(s,a)^{\top}]x }. 
    \end{align*}
\end{enumerate}
\end{lemma}

\begin{proof}
We prove the first statement. This is proved by 
\begin{align*}
   | P(\theta)-P(\theta')|=|(\theta-\theta')\psi(s,a,s')|\leq \|\theta-\theta'\|_2\|\psi(s,a,s')\|_2\leq \|\theta-\theta'\|_2, 
\end{align*}
Here, we use $\|\psi(s,a,s')\|_2\leq 1$ which is proved by the assumption by setting $V(s)=I(s'=s)$ for any $s'$.

Next, we prove the second statement.  For fixed $\theta\in \RR^d$ and $(s,a)\in \Scal\times \Acal$, we have
\begin{align*}
\mathrm{TV}(P(\theta)(s,a,\cdot),P(\theta^{\star})(s,a,\cdot)) &=\sup_{V:\Scal \to [0,1]} |\int(\theta-\theta^{\star})^{\top}\psi(s,a,s')V(s')\rd(s')|\\
    &=\sup_{V:\Scal \to [0,1]} |(\theta-\theta^{\star})^{\top}\int \psi(s,a,s') V(s')\rd(s')|\\
  &=|(\theta-\theta^{\star})^{\top}\int \psi(s,a,s') V_{(s,a,\theta)}(s')\rd(s')|  \\ 
  & =|(\theta-\theta^{\star})^{\top}\psi_{V_{(s,a,\theta)}}(s,a)|. 
\end{align*}
In the third line, we define $V(s,a,\theta)=\argmax_{V:\Scal \to [0,1]} |(\theta-\theta^{\star})^{\top}\int \psi(s,a,s') V(s')\rd(s')|$. 

Then,  from CS inequality, 
\begin{align*}
\mathrm{TV}(P(\theta)(s,a,\cdot),P(\theta^{\star})(s,a,\cdot))\leq \|(\theta-\theta^{\star}\|_2\|\psi_{V_{(s,a,\theta)}}(s,a)|\|_2\leq  \|\theta-\theta^{\star}\|_2. 
\end{align*}
We use the assumption $\|\psi_{V_{(s,a,\theta)}}(s,a)\|_2\leq 1$. This concludes the second statement.  Besides, for any $V:\Scal \to [0,1]$, we have 
    \begin{align*}
        |(\theta-\theta')\psi_{V}(s,a)|&\leq |(\theta-\theta^{\star})^{\top}\psi_{V_{(s,a,\theta)}}(s,a)|\\  &\leq \mathrm{TV}(P(\theta)(s,a,\cdot),P(\theta')(s,a,\cdot)). 
    \end{align*}

The third statement is immediately concluded by 
\begin{align}
    \frac{\E_{(s,a)\sim d^{\pi^{*}}_{P^{\star}}}[\mathrm{TV}(P(\theta)(s,a,\cdot),P(\theta^{\star})(s,a,\cdot))^2  ]}{\E_{(s,a)\sim \rho}[\mathrm{TV}(P(\theta)(s,a,\cdot),P(\theta^{\star})(s,a,\cdot))^2  ]}
    &=    \frac{\E_{(s,a)\sim d^{\pi^{*}}_{P^{\star}}}[|(\theta-\theta^{\star})^{\top}\psi_{V_{(s,a,\theta)}}(s,a)|^2]}{\E_{(s,a)\sim \rho}[|(\theta-\theta^{\star})^{\top}\psi_{V_{(s,a,\theta)}}(s,a)|^2] }. 
\end{align}

Finally, we prove the fourth statement. Suppose  $\psi(s,a,s')=\phi(s,a) \otimes \mu(s')$ ($\otimes$ denotes kronerker product). Then, $\phi_{V}(s,a,s')=\phi(s,a)\otimes \int \mu(s')V(s')\rd(s')$. Then, by defining a vector $\mu(V)=\int \mu(s')V(s')\rd(s')$, we immediately have
\begin{align}\label{eq:linear_mixture2}
\frac{ x^{\top}\E_{(s,a)\sim d^{\pi^{*}}_{P^{\star}}}[\psi_{V}(s,a)\psi^{\top}_{V}(s,a) ] x}{x^{\top}\E_{(s,a)\sim \rho}[\psi_{V}(s,a)\psi^{\top}_{V}(s,a) ] x}= \sup_{x}\frac{ x^{\top}\E_{(s,a)\sim d^{\pi^{*}}_{P^{\star}}}[(\phi(s,a)\otimes \mu(V))(\phi(s,a)\otimes \mu(V))^{\top} ] x}{x^{\top}\E_{(s,a)\sim \rho}[(\phi(s,a)\otimes \mu(V))(\phi(s,a)\otimes \mu(V))^{\top}  ] x}.
\end{align}
Here, we have 
\begin{align*}
    \E_{\rho}[(\phi(s,a)\otimes \mu(V))(\phi(s,a)\otimes \mu(V))^{\top}]&=    \E_{\rho}[ (\phi(s,a)\otimes \mu(V))(\phi(s,a)^{\top}\otimes \mu(V)^{\top})] \\ 
    &=     \E_{\rho}[(\phi(s,a)\phi(s,a)^{\top})]\otimes (\mu(V) \mu(V)^{\top}). 
\end{align*}
We notice
\begin{align*}
    \{ \E_{\rho}[(\phi(s,a)\phi(s,a)^{\top})]\otimes (\mu(V)\mu(V)^{\top})\}^{1/2}=   \E_{\rho}[\phi(s,a)\phi(s,a)^{\top}]^{1/2}\otimes (\mu(V) \mu(V)^{\top})^{1/2}. 
\end{align*}
This is because the square root of a matrix is unique and we have $(A^{1/2}\otimes B^{1/2})(A^{1/2}\otimes B^{1/2})=AB$ for symmetric matrices $A$ and $B$. Then, by denoting $F_{\rho}=\E_{\rho}[\phi(s,a)\phi(s,a)^{\top}],F_{d^{\pi}_{P^{\star}}}=\E_{d^{\pi}_{P^{\star}}}[\phi(s,a)\phi(s,a)^{\top}]$ and denoting the pseudo inverse of $F$ as $F^+$, we can see \pref{eq:linear_mixture2} is equal to 
\begin{align*}
    &\{F_{\rho}^{1/2}\otimes (\mu(V) \mu(V)^{\top})^{1/2}\}^{+}   
    \{ F_{d^{\pi}_{P^{\star}}}\otimes (\mu(V) \mu(V)^{\top})\}
    \{F_{\rho}^{1/2}\otimes (\mu(V) \mu(V)^{\top})^{1/2}\}^{+}\\
    &=\{F_{\rho}^{-1/2}\otimes (\mu(V) \mu(V)^{\top})^{-1/2}\}   
    \{ F_{d^{\pi}_{P^{\star}}}\otimes (\mu(V) \mu(V)^{\top})\}
    \{F_{\rho}^{-1/2}\otimes (\mu(V) \mu(V)^{\top})^{-1/2}\}\\
    &=\{F_{\rho}^{-1/2}F_{d^{\pi}_{P^{\star}}} F_{\rho}^{-1/2}\} \otimes \{ (\mu(V) \mu(V)^{\top})^{-1/2} (\mu(V) \mu(V)^{\top}) (\mu(V) \mu(V)^{\top})^{-1/2}\}\\
 &=\{F_{\rho}^{-1/2}F_{d^{\pi}_{P^{\star}}} F_{\rho}^{-1/2}\}\otimes I_k\, (k=\rank(\mu(V) \mu(V)^{\top})). 
\end{align*}
Here, $I_k$ is a diagonal matrix s.t. $k\in \NN^{+}$ values in the diagonal entries are $1$ and the rest of values are $0$. Then, the maximum singular value of $\{F_{\rho}^{-1/2}F_{d^{\pi}_{P^{\star}}} F_{\rho}^{-1/2}\}\otimes I_k$ is equal to the one of $\{F_{\rho}^{-1/2}F_{d^{\pi}_{P^{\star}}} F_{\rho}^{-1/2}\}$. This is equal to 
\begin{align*}
    \sup_x \frac{x^{\top} F_{d^{\pi}_{P^{\star}}} x }{x^{\top} F_{\rho} x}
\end{align*}
Hence, the fourth statement is concluded. 

\end{proof}

\begin{lemma}[Distribution shift lemma]\label{lem:distribution_shift}
Suppose $A_1,A_2,A_3$ are semipositive definite matrices:  
\begin{align*}
    \Tr(A_1 A_2 ) \leq \sigma_{\max}(A^{-1/2}_3 A_1 A^{-1/2}_3)  \Tr(A_3 A_2 ). 
\end{align*}
Note 
\begin{align*}
    \sigma_{\max}(A^{-1/2}_3 A_1 A^{-1/2}_3) =\sup_{x \in \RR^d}\frac{ x^{\top} A_1  x }{x^{\top} A_3  x }. 
\end{align*}
\end{lemma}
\begin{proof}
\begin{align*}
    &\Tr(A_1 A_2)=\Tr(A^{1/2}_1 A_2A^{1/2}_1)=\Tr(A^{1/2}_1 A^{-1/2}_3  A^{1/2}_3 A_2  A^{1/2}_3 A^{-1/2}_3 A^{1/2}_1 )\\
    &=\Tr(A^{-1/2}_3 A_1 A^{-1/2}_3  A^{1/2}_3 A_2  A^{1/2}_3  ). 
\end{align*}
In addition, for any semipositive definite matrices  $A,B$ we have 
\begin{align*}
    \Tr(AB)=  \Tr(U\Lambda U^{\top} B)= \Tr(\Lambda U^{\top} BU)\leq \sigma_{\max}(\Lambda ) \Tr(U^{\top} BU)= \sigma_{\max}(A)\Tr(B), 
\end{align*}
where $U\Lambda U^{\top}$ is the SVD decomoposition of $A$. 
This concludes that 
\begin{align*}
    \Tr(A_1 A_2)\leq \sigma_{\max}(A^{-1/2}_3 A_1 A^{-1/2}_3)  \Tr(A_3 A_2 ). 
\end{align*}

\end{proof}

The following lemma is useful to obtain the generalized result of \pref{thm:version}. The proof is given in \citet[Theorem 3.27]{WainwrightMartinJ2019HS:A}. We first define
\begin{align*}
    Z  &=\sup_{f\in \Fcal}|\{\E_{\Dcal}-\E_{\rho}\}[f]\\
\Sigma^2 &= \sup_{f\in \Fcal} \E_{\Dcal}[\{f(s,a)-\E_{\rho}[f(s,a)] \}^2],\,\sigma^2 = \sup_{f\in \Fcal} \mathrm{var}[f(s,a)]. 
\end{align*}

\begin{lemma}[Functional Bernstein's inequality]\label{lem:Talagrand}
Suppose $\|f\|_{\infty}\leq B$. 
With probability $1-\delta$, 
\begin{align*}
    |Z-\E[Z]|\leq  \Sigma^2\sqrt{\frac{\log(c/\delta)}{n} } +\frac{B\log(c/\delta)}{n}. 
\end{align*}
As an immediate corollary, 
\begin{align*}
    |Z-\E[Z]|\leq  \{\sigma^2 + B \E[Z] \} \sqrt{\frac{\log(c/\delta)}{n} } +\frac{B\log(c/\delta)}{n}. 
\end{align*}

\end{lemma}

\end{document}